\documentclass{article}
\usepackage{fullpage}

\usepackage[utf8]{inputenc} 
\usepackage[T1]{fontenc}    
\usepackage{hyperref}       
\usepackage{url}            
\usepackage{booktabs}       
\usepackage{amsfonts}       
\usepackage{nicefrac}       
\usepackage{microtype}      

\usepackage{amsmath}
\usepackage{amsthm,amsfonts,amssymb}
\usepackage{xspace}
\usepackage{bbm}
\usepackage{graphicx}
\usepackage{bbm}
\usepackage{hyperref}   
\usepackage{etoolbox}
\usepackage{enumitem}
\usepackage{stmaryrd}

\usepackage[noend]{algorithmic}
\usepackage[ruled,vlined]{algorithm2e}

\SetCommentSty{mycommfont}

\usepackage{comment}
\usepackage{natbib}
\usepackage{float}

\usepackage{accents}

\usepackage{graphicx} 
\usepackage{booktabs} 

\usepackage{caption}

\usepackage{tikz}
\usetikzlibrary{positioning}
\usetikzlibrary{calc}
\usepackage{enumitem}
\usepackage{subcaption}
\usepackage{mathtools}

\usepackage{mathrsfs}

\makeatletter

\newtheorem*{theorem*}{Theorem (Informal)}
\newtheorem*{definition*}{Definition (Informal)}
\newtheorem{theorem}{Theorem}[section]
\newtheorem{lemma}{Lemma}[section]
\newtheorem{corollary}{Corollary}[section]
\newtheorem{definition}{Definition}[section]
\newtheorem{example}{Example}[section]
\newtheorem{remark}{Remark}[section]

\newcommand{\R}{\mathbb{R}}
\newcommand{\E}{\mathop{\mathbb{E}}}
\newcommand{\argmin}{\mathop{\mathrm{argmin}}}
\newcommand{\argmax}{\mathop{\mathrm{argmax}}}
\newcommand{\cX}{\mathcal{X}}
\newcommand{\cS}{\mathcal{S}}
\newcommand{\cC}{\mathcal{C}}
\newcommand{\cY}{\mathcal{Y}}
\newcommand{\cE}{\mathcal{E}}
\newcommand{\cH}{\mathcal{H}}
\newcommand{\cP}{\mathcal{P}}
\newcommand{\cI}{\mathcal{I}}

\newcommand{\cG}{\mathcal{G}}

\newcommand{\cU}{\mathcal{U}}

\newcommand{\cA}{\mathcal{A}}

\newcommand{\cL}{\mathcal{L}}

\newcommand{\cQ}{\mathcal{Q}}

\newcommand{\cB}{\mathcal{B}}
\newcommand{\cD}{\mathcal{D}}

\newcommand{\cZ}{\mathcal{Z}}

\usepackage{thmtools} 
\usepackage{thm-restate}

\usepackage{amsmath}
\usepackage{authblk}

\allowdisplaybreaks

\title{High-Dimensional Prediction for Sequential Decision Making}

\author[1]{Georgy Noarov\thanks{Supported in part by the AWS AI Gift for Research in Trustworthy AI. 
This work was done in part while the author was a visiting student in the Data-Driven Decision Processes Semester at the Simons Institute for the Theory of Computing, Berkeley.}}
\author[1]{Ramya Ramalingam}
\author[1]{Aaron Roth\thanks{Supported in part by the Simons Collaboration on the Theory of Algorithmic Fairness, and NSF grants FAI-2147212 and CCF-2217062.}}
\author[1]{Stephan Xie}
\affil[1]{University of Pennsylvania Department of Computer and Information Sciences}

\begin{document}
\maketitle

\begin{abstract}
We study the problem of making predictions of an adversarially chosen high-dimensional state that are \emph{unbiased} subject to an arbitrary  collection of conditioning events, with the goal of tailoring these events to downstream decision makers. We give efficient algorithms for solving this problem, as well as a number of applications that stem from choosing an appropriate set of conditioning events. 

For example, we can efficiently produce predictions targeted at any polynomial number of decision makers, such that if they best respond to our predictions, each of them will have diminishing swap regret at the optimal rate. We then generalize this to the online combinatorial optimization problem, where the decision makers have large action spaces corresponding to structured subsets of a set of base actions: We give the first algorithms that can guarantee (to any polynomial number of decision makers) no regret to the best fixed action, not just overall, but on any polynomial number of \emph{subsequences} that can depend on the actions chosen as well as any external context. We show how playing in an extensive-form game can be cast into this framework, and use these results to give efficient algorithms for obtaining no \emph{subsequence regret} in extensive-form games --- which gives a new family of regret guarantees that captures and generalizes previously studied notions such as regret to informed causal deviations, and is generally incomparable to other known families of efficiently obtainable guarantees.

We then turn to uncertainty quantification in machine learning, and consider the problem of producing \emph{prediction sets} for online adversarial multiclass and multilabel classification. We show how to produce class scores that have \emph{transparent coverage guarantees}: they can be used to produce prediction sets covering the true labels at the same rate as they would \emph{had our scores been the true conditional class probabilities}. We then show that these transparent coverage guarantees imply strong online adversarial \emph{conditional validity} guarantees (including \emph{set-size conditional} coverage and \emph{multigroup-fair} coverage) for (potentially \emph{multiple}) downstream prediction set algorithms relying on our class scores. Moreover, we show how to guarantee that our class scores have improved $L_2$ loss (or cross-entropy loss, or more generally any Bregman loss) compared to any collection of benchmark models. This can be viewed as a high-dimensional, real-valued version of \emph{omniprediction}. Compared to conformal prediction techniques, this both gives increased flexibility and eliminates the need to choose a non-conformity score. 
\end{abstract}
\thispagestyle{empty} \setcounter{page}{0}
\clearpage

 \tableofcontents
 \thispagestyle{empty} \setcounter{page}{0}
 \clearpage

\section{Introduction}
When is it a good idea for a decision maker to react to a predicted outcome as if the prediction is correct? Understanding the answer to this question has at least two important kinds of applications:
\begin{enumerate}
    \item It gives us a natural \emph{algorithm design} principle for sequential decision making under uncertainty, which we call \emph{predict-then-act}. We  first \emph{predict} the payoffs for each of our available actions, taking care to produce predictions that are a ``good idea'' for us to follow, and then choose our action to optimize our payoff as if the prediction was correct.
    \item It gives us a means to coordinate action amongst a wide variety of potentially unsophisticated agents, each with different objective functions: we  produce predictions that are a ``good idea'' for each of them to follow, and then allow them to best-respond to our predictions as if they were correct. In this scenario we don't need the agents to have the sophistication to run a complicated decision making algorithm themselves --- but we might hope that they obtain the same kinds of utility guarantees that they would if they were instead running a more sophisticated algorithm.
\end{enumerate}

Calibration is a natural candidate answer to our question. Informally speaking, predictions are calibrated when they are unbiased even conditional on the value of the prediction: In other words, our predictions $p$ for an unknown outcome $y$ should satisfy (for all possible values $v$ our prediction might take): $\E[y-v | p=v]=0$. Amongst all policies mapping predictions to actions, the ``best response'' policy that acts as if the predictions are correct is payoff optimal if the predictions are calibrated. Moreover, it has been known since \cite{foster1998asymptotic} that it is possible to make sequential predictions that are guaranteed to be calibrated in hindsight, even if outcomes are chosen by an adaptive adversary.

But calibration has at least two serious shortcomings:
\begin{enumerate}
    \item The statistical and computational complexity of producing  forecasts for $d$-dimensional outcomes that are guaranteed to be calibrated grows exponentially with $d$. This is because the number of (discretized) values $v$ that we might predict in $d$ dimensions grows exponentially with $d$. For calibration, not only is our prediction space exponentially large, but we ask that our predictions be unbiased subject to exponentially many conditioning events. In fact, even in 1 dimension, it is known that achieving adversarial calibration at a rate of $O(\sqrt{T})$ is impossible \citep{qiao2021stronger} --- even though it is possible to obtain swap regret at this rate \citep{blum2007external}. 
    \item Calibrated forecasts also need not be informative: even a constant predictor can be calibrated (so long as it predicts the mean outcome), and, in general, since calibration is a \emph{marginal} guarantee, it can fail to hold after conditioning on various decision-relevant features $x$ that are available at decision time.
\end{enumerate}

When can we overcome these shortcomings?

\subsection{Overview of Our Results}
We study the problem of sequentially making predictions $p_t \in \mathbb{R}^d$ of $d$-dimensional adversarially selected vectors $y_t \in \mathbb{R}^d$, potentially as a function of available contexts $x_t$. We ask for predictions that are \emph{unbiased} subject to some set of conditioning events $\cE$, where each event $E(x_t,p_t) \in \cE$ can be a function of both the context $x_t$ and our own prediction $p_t$. In other words, we want that after $T$ rounds of interaction: 
$$\left\lVert\E_{t \in [T]}\left[(p_t - y_t) | E(x_t,p_t)\right]\right\rVert_\infty \leq \alpha(E) \quad \text{for all $E \in \cE$}.$$ Here we say that such predictions have $\cE$-bias bounded by $\alpha$. The standard notion of calibration corresponds to the set of events defined as $E_v(p_t) = \mathbbm{1}[p_t = v]$ for each $v$ in some $\epsilon$-net of the prediction space, when predictions are limited to being made from the same net. The size of an $\epsilon$-net in $d$-dimensions is exponential in $d$, which is one of the primary difficulties with high dimensional calibration. 

Our first result is an algorithm for making $\cE$-unbiased predictions for any collection of events $\cE$, with bias $\alpha(E) = O\left(\frac{\ln(d|\cE| T)}{\sqrt{T_E}}\right)$ and with per-round running time scaling polynomially with $d$, $t$, and $|\cE|$ for each round $t \in [T]$. Here $T_E = \sum_{t=1}^T E(x_t,p_t)$ is the number of rounds for which event $E$ was active. Thus the algorithm is computationally efficient whenever the set of events $\cE$ is polynomially sized, and up to low order terms, achieves in an adversarial setting a bias rate that would be statistically optimal even in a stochastic setting. When the events are ``disjoint'' (i.e. at most one is active for any given prediction), we show how to improve the per-round running time to depend only polylogarithmically on $t$ at each round $t \in [T]$.

We establish several basic connections between our ability to make unbiased predictions subject to events defined by the best-response correspondence of a downstream decision maker, and the regret of that downstream decision maker, whenever the utility function of the decision maker is linear in the state vector $y_t$ that we are predicting. This captures decision makers who have \emph{arbitrary} valuations over $d$ discrete states when we predict a probability distribution over those states, as well as a number of other important cases. We then apply this result to obtain new results in online combinatorial optimization, learning in extensive form games, and uncertainty quantification using prediction sets.
\subsubsection{Applications}
\paragraph{Warm-up: Groupwise Swap Regret for Many Decision Makers}
To build intuition, we start by deriving algorithms for obtaining no swap regret in the experts problem with $d$ actions. Informally, an agent has no swap regret if they have obtained utility that is as large as they would have had they played the best action in hindsight---not just overall, but also on every subsequence on which they selected any particular one of their actions. For a decision maker with utility function $u$, to guarantee that they will have no swap regret when they best respond to our predictions, it suffices that our predictions are unbiased with respect to the $d$ disjoint events defined by their best response correspondence --- i.e. the events in which they choose to play each of their $d$ actions. Similar observations have been made before \cite{perchet2011internal,zhao2021calibrating,haghtalab2023calibrated}. But now, paired with our prediction algorithm, we can efficiently make predictions that guarantee diminishing swap regret to \emph{every} downstream decision maker with utility function in a polynomially sized set $U$, with bounds that are optimal up to a term growing logarithmically with $|U|$. We can also enlarge the set of events as a function of external context $x_t$ to simultaneously give diminishing swap regret to each agent not just overall, but for arbitrary subsequences of actions that might e.g. correspond to demographic or other decision relevant groupings of the prediction space. This improves upon \cite{blum2020advancing}, who gave algorithms for obtaining diminishing \emph{external} groupwise regret for a single decision maker.

\paragraph{Subsequence Regret in Online Combinatorial Optimization and Extensive Form Games}

We then consider the online combinatorial optimization problem, in which a decision maker has a combinatorially large action space corresponding to subsets of $d$ base actions, and has a utility function that is linear in the payoff of each of the base actions. A canonical special case of this setting is the online shortest paths problem, in which the base actions correspond to edges in a network, and an agent's action set corresponds to the set of $s\rightarrow t$ paths in the network (different agents might have different source and destination pairs). The action space of each player can be as large as $2^d$. Given any polynomial collection of subsequence indicator functions (which can depend both on external context as well as the decisions made by the decision makers --- e.g. the subsequence of days on which the chosen path includes toll roads and it is raining), we show how to efficiently make predictions such that the downstream decision maker has no regret to any of their $O(2^d)$ actions, not just overall, but also as restricted to any of the subsequences. Subsequence regret guarantees like this were previously known for settings with small action spaces \citep{blum2007external} (i.e., via algorithms with running time that is polynomial in the number of actions) --- we give the first such result for a large action space setting. Our result naturally extends to making predictions that give this guarantee not just to a single decision maker, but to every decision maker with utility function in any polynomially sized set $U$.

We then observe that this result applies to extensive form games as a special case. Informally speaking, an extensive form game is a sequential interaction played on a \emph{game tree}. A player controls a subset of the nodes in the tree, and must decide on an action to take at each node; opponents or chance players decide on actions at other nodes, and play proceeds down the tree until it reaches a leaf node which corresponds to a payoff to each agent. Internal nodes of the game tree can be grouped together into ``information sets'' that are indistinguishable to the agent (which constrains the agent to choose the same action at every node within an information set). Because an agent must decide on an action to take at each node of the tree, the strategy space of the game is exponentially large in the size of the game tree. Nevertheless, the expected payoff that a strategy yields for a player can be expressed as the inner product of a vector representing the set of leaves that the agent's strategy makes reachable, and a vector corresponding to the payoff-weighted vector of probabilities that the agent's opponents' make each leaf reachable. This means it can be expressed as an online combinatorial optimization problem of dimension equal to the number of leaves in the game tree (exponentially smaller than the number of actions). Our methods therefore give algorithms for obtaining subsequence regret in extensive form games for arbitrary polynomial collections of subsequences. 

Our method does \emph{not} require that the information sets in the extensive-form game satisfy the ``perfect recall'' assumption as many prior methods do, and is an efficient reduction to the ``best response'' problem in extensive form games. That is, whenever it is possible to efficiently compute an extensive form strategy that best responds to fixed opponent strategies, our algorithms are efficient. In some cases, the running time of our algorithms can be improved to depend on the number of information sets rather than the number of game tree leaves. We show that subsequence regret generalizes the existing notion of regret to informed causal deviations \cite{gordon2008no,dudik2012sampling} --- which means it can be used to achieve convergence to notions of extensive form correlated equilibrium; subsequence regret is incomparable to other known families of regret guarantees in extensive form games \cite{morrill2021efficient,farina2022simple, farina2023polynomial}.

\paragraph{Score-Free Prediction Sets with Anytime Transparent Coverage}
A popular way of quantifying the uncertainty of machine learning predictions in multiclass classification problems is to produce prediction sets rather than point predictions. A prediction set is a set of labels that is intended to contain the true label with some target probability, say 95\%. Given an example $x$, if we knew the true conditional probability $p(y|x)$ of each label $y$, we could produce the smallest possible prediction set subject to the coverage guarantee by sorting the labels in decreasing order of their likelihood, and including them in the prediction set until their cumulative probability exceeded 95\%. More generally, we could optimize any other objective function to produce a prediction set, and read off the coverage rate of our set by summing the probabilities of the labels included within our prediction set. Unfortunately, the scores produced by machine learning models are \emph{not} true conditional probabilities, and so this approach generally does not work. The approach taken by the conformal prediction literature (see e.g.  \cite{shafer2008tutorial,angelopoulos2021gentle}) is to use a ``non-conformity score'' to reduce the high dimensional space of prediction sets to a 1-dimensional nested set of prediction sets, and solve a 1-D quantile estimation problem. Much of the art in making conformal prediction work well is choosing the ``right'' 1-dimensional scoring function. 

We give a score-free method of producing prediction sets in arbitrary sequential prediction problems. Given a method for mapping individual label probabilities to prediction sets (such as e.g.\ finding the smallest set subject to a coverage constraint), we produce predicted probabilities for each label that are unbiased conditional on the event that the method includes the label within the prediction set. For any polynomial collection of such methods, our predicted probabilities can be used by each of these methods, and they will be guaranteed to satisfy the same coverage guarantees they would have had the probabilities been correct. In other words, the prediction sets have ``transparent'' coverage guarantees, in that we can estimate the coverage of each method by simply summing up the predicted scores for each label that is contained within the method's prediction sets. For example, this lets us give class probabilities that can be simultaneously used to produce prediction sets for many different coverage probabilities, and optimizing many different objectives (e.g. weighted prediction set sizes with different weights). By defining our events in terms of relevant context, we can also extend all of our guarantees to offer \emph{groupwise} or \emph{multivalid} coverage as in \cite{bastanipractical,jung2023batch}. Through another appropriate instantiation of our event collection, we also, for the first time, obtain online adversarial \emph{set-size-conditional} (\cite{angelopoulos2020uncertainty}) coverage guarantees.

Moreover, we can produce class scores with these transparent coverage guarantees that are simultaneously as accurate as any other prediction method at our disposal: we show how to produce scores satisfying a high-dimensional notion of \emph{calibeating} \cite{foster2022calibeating}. What this means is given any polynomial collection of predictors $f(x)$ which map features to class probabilities, we produce predicted class probabilities that have smaller Brier score --- or cross entropy --- or more generally lower loss according to any \emph{Bregman score} than any predictor in the class, while simultaneously being useful for producing prediction sets. This can be viewed as an online, high-dimensional, and real-valued extension of \emph{omniprediction} \cite{GopalanKRSW22,GJRR23}. In addition to giving a more flexible collection of guarantees than conformal prediction methods, this eliminates the need to choose a non-conformity score.

\paragraph{More Applications}
We expect that our framework will find many other applications. We here mention a few. First, there are some basic applications that take advantage of the ability of our framework to make predictions within an arbitrary convex feasible region. For example, we can make probability forecasts in $d$ outcome settings that satisfy marginal (or top label) calibration \emph{subject to the constraint that the predicted probabilities sum to 1}. Similarly we can produce predicted CDFs for a real valued outcome satisfying marginal quantile calibration \citep{gupta2022online} at each quantile level, subject to the constraint that our predicted quantile values are monotone. Although these applications are extremely simple, we are not aware of ways to easily obtain these guarantees with prior work. 

We also briefly mention an application of our techniques that is explored in concurrent work of \cite{CRS23}, who apply our algorithms in a repeated principle agent setting defined by \cite{camara2020mechanisms}. Briefly, \cite{camara2020mechanisms} gave a mechanism that replaced the standard ``common prior'' assumptions that underlie principal agent models with calibrated forecasts of an underlying state, and is applicable in adversarial settings. \cite{camara2020mechanisms} use the traditional notion of calibration, and as a result inherit exponential computational and statistical dependencies on the cardinality of the state space. \cite{CRS23} show how to apply our techniques to recover the same results (under weaker assumptions) with an exponentially improved dependence on the cardinality of the state space.

\subsection{Related Work}
\paragraph{Calibration, multicalibration and downstream optimization} The study of sequential calibration goes back to \cite{dawid1985calibration} who viewed it as a way to define the foundations of probability, and algorithms for producing calibrated forecasts in an adversarial setting were first given by \cite{foster1998asymptotic}. \cite{foster1999regret} were the first to connect sequential calibration to sequential decision making, showing that a decision maker who best responds to (fully) calibrated forecasts obtains diminishing internal regret (and that when all agents in a game do so, empirical play converges to correlated equilibrium). \cite{kakade2008deterministic} and \cite{foster2018smooth} make a similar connection between ``smooth calibration'' (which in contrast to classical calibration can be obtained with deterministic algorithms) and Nash equilibrium.

In the recent computer science literature, there has been interest in constructive calibration guarantees (obtained by efficient algorithms and obtaining good rates) that hold conditional on context in various ways, called \emph{multi-calibration} \citep{hebert2018multicalibration}. Multicalibration has been studied both in the batch setting \citep{hebert2018multicalibration,kim2019multiaccuracy,GHK23,haghtalab2023unifying} and in the online sequential setting \citep{foster2006calibration,foster2011complexity,gupta2022online,GJRR23}. For the most part (with a few notable exceptions \cite{gopalan2022low,zhao2021calibrating}) multicalibration has been studied in the 1-dimensional setting in which the outcome being predicted is boolean. This has been extended to predicting real valued outcomes, with notions of calibration tailored to variances \cite{jung2021moment}, quantiles \cite{bastanipractical,jung2023batch}, and other distributional properties \cite{NR23}. See \cite{rothuncertain} for an introductory exposition of this literature.

There is a line of work that aims to use (multi)calibration as a tool for a one-dimensional form of downstream decision making, called omniprediction. The goal of omniprediction, introduced in \cite{GopalanKRSW22}, is to make probabilistic predictions of a binary outcome as a function of contextual information that are useful for optimizing a variety of downstream loss functions. \cite{GopalanKRSW22} show that a predictor that is multicalibrated with respect to a benchmark class of functions $\cH$ and a binary label can be used to optimize any convex, Lipschitz loss function of an action and a binary label (see \cite{gopalan2023loss} for a related set of results). These results are in the batch setting. In the online setting, \cite{kleinberg2023u} defined ``U-calibration'', which can be viewed as a non-contextual version of omniprediction, in which the goal is to make predictions that guarantee an arbitrary downstream decision maker no external-regret. Since there is no context in this setting, the benchmark class to which regret is measured is the set of constant functions. In comparison to \cite{kleinberg2023u}, our goal is to give both stronger guarantees than external regret, and to be able to do so even when the state space is very large. What we pay for these stronger guarantees is a logarithmic dependence on the number of downstream utility functions our guarantees hold for (\cite{kleinberg2023u} give algorithms that guarantee no external regret for \emph{any} downstream utility function). Using a connection between multicalibration and swap-regret established by \cite{GHK23} and \cite{gopalan2023characterizing}, \cite{GJRR23} give \emph{oracle efficient} algorithms for online multicalibration, with applications to online omniprediction --- i.e. algorithms that are efficient reductions to the problem of online learning over the benchmark class of functions $\cH$. Like other work in omniprediction, this result is limited to the 1-dimensional binary setting. 

The most closely related work is \cite{zhao2021calibrating}, who define and study ``decision calibration'' in the batch setting in the context of predicting a probability distribution over $k$ discrete outcomes. Decision calibration is a slightly weaker requirement than what we study, also defined in terms of the best-response correspondence of a decision maker's utility function. Decision calibration asks, informally, that a decision maker be able to correctly estimate the expected reward of their best response policy; we ask for a slightly stronger condition that requires them to also be able to estimate the utility of deviations as a function of their play. This kind of unbiased estimation (based on the best response correspondence of a decision maker) has also been previously observed to be related to swap regret in \cite{perchet2011internal} and \cite{haghtalab2023calibrated}. The algorithmic portion of our work can be viewed as extending \cite{zhao2021calibrating} from the batch to the online adversarial setting; Our applications hinge crucially on both the online aspect of our algorithm and on the more general setting we consider, beyond predicting distributions on $k$ outcomes.  

An expansive recent literature has focused on the similarly named \emph{predict-then-optimize} problem~\cite{elmachtoub2022smart, el2019generalization, liu2021risk}. This line of work investigates a setup in which predictions made from  data are to be used in a linear optimization problem downstream in the pipeline. This is  similar in motivation to our `predict-then-act' framework, but with two important differences: (1) the predict-then-optimize framework aims to optimize for a single downstream problem, whereas our framework aims to simultaneously provide guarantees to an arbitrary finite collection of downstream decision makers; and (2) the surrogate loss approach studied in this literature is naturally embedded in a batch/distributional setting, where the goal is to exactly optimize for the Bayes optimal downstream decision policy, up to generalization/risk bounds; meanwhile, our framework naturally lives in the  online adversarial setting, and aims for different notions of optimality defined in terms of regret bounds, as well as omniprediction-type `best-in-class' optimality. Both frameworks can be used to solve downstream combinatorial optimization problems~\cite{mandi2020smart, demirovic2019investigation}; but our framework appears to have a broader set of applications --- as a consequence of its strong calibration properties, we are able to apply our framework to derive strong uncertainty quantification guarantees, which do not appear to naturally fit within the predict-then-optimize framework. There also exist other approaches for learning in batch decision making pipelines, that are different from the predict-then-optimize method; see e.g.~\cite{donti2017task, khalil2017learning, wilder2019melding, vanderschueren2022predict}.

\paragraph{No-regret guarantees in online learning} There are also long lines of work in online learning related to our applications; here we survey the most relevant. No-regret learning, which requires that a decision maker obtain cumulative loss at most that of their best single action in hindsight against an adversarial sequence of losses, has been studied at least since \cite{hannan1957approximation} --- see \cite{hazan2016introduction} for a modern treatment of this literature. We highlight \cite{kalai2005efficient} (which we will make use of) who give efficient no regret algorithms in online linear and combinatorial optimization problems, which are large-action-space settings in which the cost of each action has linear structure.  Internal regret, which corresponds to regret on the subsequences defined by the play of each action, was first defined by \cite{foster1999regret}, who also showed it could be obtained by best responding to calibrated forecasts. \cite{lehrer2003wide} defined a notion of ``wide-range regret'' which is equivalent to subsequence regret: that a player should have no regret not just overall on the whole sequence, but also on subsequences that can be defined both as a function of time (``time selection functions'') and as a function of the actions of the learner. \cite{blum2007external} gave algorithms for obtaining this kind of subsequence regret (including, notably internal (or ``swap'') regret as a special case). The algorithm of \cite{blum2007external} is efficient when the action space is polynomially sized: it requires computing eigenvectors of a square matrix of dimension equal to the number of actions in the game.  Motivated by fairness concerns, \cite{blum2020advancing} give an algorithm for obtaining diminishing ``groupwise'' regret, which is equivalent to regret with respect to a collection of time selection functions; very recently (and concurrently with this paper), \cite{groupwise} show how to modify the algorithm of \cite{blum2020advancing} to make it ``oracle efficient'' --- to reduce it to the problem of obtaining external regret with overhead polynomial in the number of time selection functions. These results do not accommodate subsequences that can depend on the actions of the learner, which are crucial for our applications. We give the first efficient algorithms for getting subsequence regret, for an arbitrary polynomial number of subsequences, in online combinatorial optimization settings, by operating over the (polynomially sized) linear representation space for the costs rather than the (exponentially sized) action space.

The problem of efficiently obtaining no-regret in extensive form games was first studied by \cite{zinkevich2007regret}, who gave algorithms for obtaining no (external) regret, which is sufficient for convergence to Nash equilibrium in zero-sum games in self-play. \cite{gordon2008no,dudik2012sampling} give algorithms for obtaining no-regret to \emph{causal deviations}, which is sufficient for convergence to a notion of correlated equilibrium in extensive form games \citep{von2008extensive}. More recently, there has been renewed theoretical interest in regret-minimization in extensive form games: both \cite{morrill2021efficient} and \cite{farina2023polynomial} have defined incomparable classes of no-regret guarantees. Our notion of subsequence regret in extensive form games generalizes regret to causal deviations---which can be obtained by asking for no-regret on those subsequences in which the player makes each particular internal node in the game tree reachable---and is incomparable to the classes defined in both \cite{morrill2021efficient} and \cite{farina2023polynomial}.

\paragraph{Uncertainty quantification} The goal of conformal prediction is to endow black box predictors with the ability to produce ``prediction sets'' --- predictions corresponding to sets of labels --- that have the property that they contain the true label with a desired coverage probability. Conformal prediction uses a one-dimensional ``non-conformity score'' to reduce the problem to a 1-dimensional quantile estimation problem---see \cite{shafer2008tutorial} and \cite{angelopoulos2021gentle} for approachable introductions to the topic. A primary point of departure for our work is that we dispense with the non-conformity score, and deal directly with the high-dimensional prediction set problem.

Our method, while different from conformal prediction, provides online adversarial coverage guarantees, and thus joins a very recent collection of works on adversarial conformal inference methods~\cite{gibbs2021adaptive, gupta2022online,zaffran2022adaptive, bastanipractical, gibbs2022conformal}.

While vanilla conformal prediction approaches give marginal coverage guarantees (i.e., those that hold  on average over the entire data set), significant amounts of recent work in conformal prediction has focused on obtaining \emph{conditional} guarantees of various sorts, which hold even over smaller, relevant portions of the data. \cite{vovk2012conditional} was perhaps the first to study conditional coverage problems. While it was shown by~\cite{foygel2021limits} that obtaining full conditional coverage in a distribution free regression setting is impossible, various special types of conditional coverage have been studied in the batch conformal setting, e.g.\ \cite{sadinle2019least}, \cite{romano2020malice}, \cite{izbicki2019flexible}, \cite{feldman2021improving}, \cite{cauchois2021knowing}, \cite{bian2023training}.  Recently, multicalibration techniques (aimed at calibrating to either the variance or the quantiles of the non-conformity score) have been used to give  group-conditional and threshold conditional guarantees in both the batch and adversarial settings \citep{jung2021moment,gupta2022online,bastanipractical,jung2023batch,deng2023happymap,gibbs2023conformal}---all of this still applied in the presence of a chosen non-conformity score, which makes the problem 1-dimensional.

An appealing feature of our method is that it predicts probability scores that ``look like the true conditional probabilities $p(y|x)$ for the purposes of producing prediction sets''. The fact that deep neural networks often produce miscalibrated class scores was observed  by~\cite{guo2017calibration} and generated a large literature that is too expansive to fully survey here. 
\cite{sadinle2019least}  showed via a reduction to the Neyman-Pearson lemma that if scores exactly coincided with the true conditional probabilities, they could be used to produce  \emph{set-size optimal} prediction sets. Other relations between true conditional probabilities and optimality of various prediction set efficiency criteria were studied by~\cite{vovk2016criteria}. Rather than aiming for optimality in an absolute sense, we demonstrate that our multiclass probability predictions satisfy a strong best-in-class property, which is new to the uncertainty quantification literature, but has been studied in several recent works on calibration. \cite{foster2022calibeating} introduces the problem of ``calibeating'': making (one-dimensional) calibrated forecasts in an adversarial setting that have lower Brier score than any fixed and given benchmark model. \cite{lee2022online} give improved bounds for simultaneously calibeating many models. We show how to produce multiclass models that give transparent coverage guarantees while simultaneously satisfying a high-dimensional version of calibeating; obtaining lower Brier score (or indeed any Bregman score) than any one of a collection of given benchmark models. This can also be viewed as a high-dimensional and real-valued generalization of omniprediction~\cite{GopalanKRSW22}.

\section{Preliminaries}
\label{sec:prelim}

This section introduces the central object of our study, the sequential (online adversarial) pipeline:
\[\text{Data $\to$ Predictions $\to$ Decisions}.\] In Section~\ref{sec:prelim-predictions}, we define the prediction task as well as the desired event-conditional unbiasedness guarantee for the predictions. In Section~\ref{sec:prelim-decisions}, we formally introduce the utility-based model for downstream decision makers, as well as several types of \emph{regret}, i.e., metrics with which to measure the decision maker's success. 

\subsection{Predictions}
\label{sec:prelim-predictions}

We are faced with a sequential high dimensional prediction setting, defined by an arbitrary \emph{context space} $\cX$ and a convex compact prediction space $\cC \subseteq \mathbb{R}^d$ for some finite dimension $d$. Without loss of generality (up to scaling our bounds by a multiplicative constant), we assume that $2 \max_{y \in \cC}||y||_\infty \leq 1$. 

In rounds $t \in \{1,2,\ldots\}$, a \emph{learner}, or \emph{predictor}, interacts with an \emph{adversary} as follows:
\begin{enumerate}
    \item The learner (may) observe a context $x_t \in \cX$;
    \item The learner produces a distribution over predictions $\psi_t \in \Delta \cC$, from which a \emph{prediction} $p_t \in \cC$ is sampled;
    \item The adversary produces an outcome $y_t \in \cC$.
\end{enumerate}
Let $\Pi = \{(x,p,y) \in \cX\times \cC\times \cC\}$ denote the set of possible realized triples at each round. An interaction over $T$ rounds produces a transcript $\pi_T \in \Pi^T$. We write $\pi_T^{< t}$ as the prefix of the first $t-1$ triples in $\pi_T$, for any $t\leq T$. We write $\Pi^* = \bigcup_{T =1}^\infty \Pi^T$ for the space of all transcripts. A learner is a collection of randomized algorithms (one for each $t$) mapping a transcript of length $t-1$ and a context to a distribution over predictions at round $t$: $\mathrm{Learner}_t:\Pi^{t-1}\times \cX \rightarrow \Delta \cC$. An adversary is a collection of randomized algorithms, each mapping a transcript of length $t-1$ to a context and a distribution over realizations: $\mathrm{Adv}_t:\Pi^{t-1} \rightarrow \cX \times \Delta \cC$. This models an adaptive adversary who can make decisions as an arbitrary function of the past history, but must be independent of the learner's randomness at round $t$. A learner paired with an adversary implicitly define a distribution over transcripts.

An intermediate goal (generally in service of downstream decision making, as defined in Section \ref{sec:prelim-decisions}) will be that in hindsight (i.e., in expectation over the empirical distribution of the transcript $\pi^T$ that ends up being realized after $T$ rounds) our predictions are \emph{unbiased}, not just overall, but also conditional on various \emph{events}. We now formally define the notion of events and our unbiasedness objective.

\begin{definition}[Events]
An \emph{event} $E$ is an arbitrary mapping from transcripts, contexts and predictions to $[0,1]$: $E:\Pi^*\times \cX\times \cC \rightarrow [0,1]$. 

If the range of $E$ is $\{0,1\}$, then we say that $E$ is a \emph{binary} event. We say that a collection of binary events $\cE$ is \emph{disjoint} if for every $(\pi,x,p) \in \Pi^*\times \cX\times \cC$: $\sum_{E \in \cE}E(\pi,x,p) \leq 1$. 

We will elide arguments to $E$ that are not used; for example, if an event is independent of the transcript, we will write $E(x_t,p_t)$ rather than $E(\pi_{t-1},x_t,p_t)$, and if the event is also independent of context, then we will simply write $E(p_t)$.
\end{definition}

\begin{definition}[Event Frequency]
    Fixing a transcript $\pi_T = \{(x_t,p_t,y_t)\}_{t=1}^T$, the \emph{frequency} of event $E$ with respect to $\pi_T$ at time $t$ is given by:
    \begin{equation*}
        n_t(E, \pi_T) = \sum_{\tau = 1}^t E(\pi_T^{< \tau},x_\tau, p_\tau)
    \end{equation*}
    Observe that for any collection of disjoint events $\cE$ and any round $t$, one has $\sum_{E \in \cE}n_t(\pi_T^{< \tau},E,\pi_T) \leq t$.
\end{definition}

\begin{definition}[Event-conditional unbiasedness] \label{def:unbiased}
    Fix a collection $\cE$ of events and a function $\alpha:\mathbb{R} \rightarrow \mathbb{R}$. A transcript $\pi_T = \{(x_t,p_t,y_t)\}_{t=1}^T$ is $\alpha$-unbiased with respect to $\cE$ if for every $E \in \cE$ and every coordinate $i \in [d]$: 
    \[
    \left|\sum_{t=1}^T (p_{t,i}-y_{t,i}) \cdot E(x_t,p_t)\right| \leq \alpha(n_t(E,\pi_T)).
    \]
\end{definition}

\begin{remark}
    Two remarks about Definition~\ref{def:unbiased} are in order. First, it asks for our prediction vectors to be unbiased (on average over time) separately across their coordinates. However, we allow the conditioning events $E$ to depend on \emph{the entire vector} $p_t$. This is where our approach derives its power, compared to asking for multi-calibration guarantees separately in every coordinate (such guarantees can be obtained by running e.g.\ the algorithms of \cite{gupta2022online} independently for each of the $d$ coordinates).

    Second, we let $\alpha$ be a \emph{function} of $n_t(E,\pi_T)$. This allows us to give finer-grained guarantees that will scale with $\alpha(n_t(E,\pi_T)) \approx \sqrt{n_t(E,\pi_T)}$, as opposed to the coarser style bounds of $\alpha \approx \sqrt{T}$ used in prior work\footnote{For the problem of online ``quantile multicalibration'', \cite{bastanipractical} gave bounds scaling with a quantity analogous to $\sqrt{n_t(E,\pi_T)}$, but at a cost of an exponentially sub-optimal dependence on other problem parameters.} \cite{gupta2022online,lee2022online}.
\end{remark}

\subsection{Decisions and Regret}
\label{sec:prelim-decisions}
We are interested in making predictions that are useful for downstream decision makers. We study decision makers who can choose amongst a set of actions $\cA = \{1,\ldots,K\}$\footnote{Without loss of generality, we assume all downstream have the same action set. If not, let $K$ be the cardinality of the largest action set, and introduce dummy actions for any agent with fewer actions.} Agents will obtain utility that is a function of both the action they take, and the outcome $y \in \cC \subseteq \mathbb{R}^d$ (which, in a game theoretic setting, may itself depend on the actions taken by the agents). We will assume that the utility functions are linear and Lipschitz-continuous in $y$. 

\begin{definition}[Decision maker's utility]
    \label{def: utilities}
    A utility function $u:\cA \times \cC\rightarrow [0,1]$ maps an action $a \in \cA$ and an outcome $y \in \cC$ to a real number $u(a,y)$. We assume that for every action $a \in \cA$:
    \begin{enumerate}
    \item $u(a, \cdot)$ is \emph{linear} in its second argument:  for all $\alpha_1,\alpha_2 \in \mathbb{R},$ $y_1,y_2 \in \cC$, $$u(a,\alpha_1 y_1 + \alpha_2 y_2) = \alpha_1 u(a,y_1) + \alpha_2 u(a, y_2)$$  
    \item $u(a,\cdot)$ is $L$-Lipschitz in its second argument, in the $L_\infty$ norm: for all $y_1,y_2 \in \cC$,
    $$\left|u(a,  y_1) - u(a,y_2)\right| \leq L||y_1-y_2||_\infty$$
    \end{enumerate}
    We write $\cU_{L}$ to denote a collection of $L$-Lipschitz utility functions of this form. 
\end{definition}
\begin{remark}
For simplicity we assume that the utility function is \emph{linear} in $y$, but we can equally well handle the case in which the utility function is \emph{affine} in $y$: we simply augment the prediction space $\cC$ with an extra $(d+1)$-st coordinate that takes constant value $1$. This preserves the convexity of $\cC$, and now allows for arbitrary constant offsets in the utility of each action $a$. This allows us to capture many settings of interest. 
    As an example, if there are $d$ discrete outcomes that are payoff relevant to the decision maker in arbitrary ways, then we can define $\cC$ as the simplex of probability distributions over outcomes, and let $u(a,p)$ be the expected utility for the agent who plays action $a$ when the outcome is sampled from $p$; this is linear in $p$, and captures the setting studied in prior work \cite{zhao2021calibrating,kleinberg2023u}; but our results extend to the more general case of arbitrary convex compact $\cC$, and so we are not limited to talking about distributions over discrete outcomes. Most of our applications will take advantage of this generality.
\end{remark}

A utility function naturally induces a best response function, mapping outcomes $y \in \cC$ to actions $a \in \cA$ that are utility maximizing given the outcomes:
\begin{definition}[Best response]
    Fix a utility function $u:\cA\times\cC\rightarrow [0,1]$. The corresponding \emph{Best Response} function $\delta_u:\cC\rightarrow \cA$ is defined as:
    $$\delta_u(y) = \argmax_{a \in \cA} u(a,y)$$
    We assume that all ties are broken lexicographically. If $\delta_u(y) = a$, we say that $a$ is a best response for utility function $u$ given $y$. 
    We write $E_{u,a}(y) = \mathbbm{1}[\delta_u(y)=a]$ to denote the binary event that $a$ is a best response to $y$ for utility function $u$. Observe that for any utility function $u$, the set of events $\{E_{u,a}\}_{a \in \cA}$ is disjoint.
\end{definition}

\begin{lemma}
\label{lem:convexdelta}
    For any utility function $u$ that is linear in its second argument, the corresponding best response function $\delta_u$ has convex levelsets: for any $\alpha_1,\alpha_2 \in \mathbb{R}^{\geq 0}$ with $\alpha_1 + \alpha_2 \leq 1$, and any $y_1,y_2 \in \cC$, if $\delta_u(y_1) = \delta_u(y_2) = a$, then $\delta_u(\alpha_1 y_1 + \alpha_2 y_2) = a$.
\end{lemma}
\begin{proof}
    This follows from the linearity of $u$ in its second argument and the definition of $\delta_u$. For any alternative action $a' \in \cA$, We can compute:
    \begin{eqnarray*}
        u(a,(\alpha_1 y_1 + \alpha_2 y_2)) &=& \alpha_1 u(a,y_1) + \alpha_2 u(a,y_2) \\
        &\geq& \alpha_1 u(a',y_1) + \alpha_2 u(a',y_2) \\
        &=& u(a', (\alpha_1 y_1 + \alpha_2 y_2))
    \end{eqnarray*}
    Here the first and last equalities follow from linearity of $u$ in its second argument, and the inequality follows from the definition of the best response function $\delta_u$. 
\end{proof}

As we make predictions $p_1,\ldots,p_t$ a decision maker with utility function $u$ may use these predictions to take a sequence of actions $a_1,\ldots,a_t$. We call a decision maker \emph{straightforward} if they take the predictions at face value, assuming that $y_t = p_t$:
\begin{definition}[Straightforward Decision Maker]
    \label{def:straightforward-dm}
    A straightforward decision maker with utility function $u$ treats predictions as correct and at each day $t$ chooses $a_t = \delta_u(p_t)$.
\end{definition}

Because the  predictions need not be correct, in hindsight (i.e. with knowledge of the realizations $y_1,\ldots,y_t$), a straightforward decision maker may regret not having taken some other sequence of actions. We study several kinds of regret in this paper. $\Phi$-regret, as defined by \cite{phiregret}, is defined with respect to a collection of mappings $\phi:\cA\rightarrow \cA$ from played actions to alternative actions called \emph{strategy modification rules}. 
\begin{definition}[$\Phi$-regret]
    Fix a transcript $\pi_T$. The \emph{regret} that a straightforward decision maker with utility function $u$ has with respect to a strategy modification rule $\phi:\cA\rightarrow \cA$ is:
    $$r(\pi_T,u,\phi) = \frac{1}{T}\sum_{t=1}^T u(\phi(a_t),y_t) - u(a_t,y_t),$$
    where $a_t = \delta_u(p_t)$ for each $t$.

    Let $\Phi$ be a collection of strategy modification rules $\phi$. We say that the decision maker has $\Phi$-regret $\alpha$ if $r(\pi_T,u,\phi) \leq \alpha$ for all $\phi \in \Phi$. 
\end{definition}

\begin{definition}[External regret and Swap regret \cite{foster1999regret,blum2007external}]
    For each action $a' \in \cA$, let $\phi_{a'}$ be the constant function defined as $\phi(a) = a'$ for all $a \in \cA$. 
    \emph{External} regret corresponds to $\Phi$-regret for $\Phi_{\textrm{Ext}} = \{\phi_{a'} : a' \in \cA\}$, the set of all constant strategy modification rules. 

    \emph{Swap} regret corresponds to $\Phi$-regret for $\Phi_{\textrm{Swap}}$ equal to the set of \emph{all} strategy modification rules. 
\end{definition}

Subsequence regret is defined by a collection of  events, and requires that the decision maker have no \emph{external} regret on any of the subsequences defined by the events. 

\begin{definition}[Subsequence regret \cite{lehrer2003wide,blum2007external,lee2022online}]
    Fix a transcript $\pi_T$. The \emph{regret} that a straightforward decision maker with utility function $u$ has with respect to an event $E$ and strategy modification rule $\phi:\cA\rightarrow \cA$ is:
    $$r(\pi_T,u,E,\phi) = \sum_{t=1}^T E(\pi_{t-1},x_t,p_t) \left(u(\phi(a_t),y_t) - u(a_t,y_t) \right)$$
    where $a_t = \delta_u(p_t)$ for each $t$. 

    Fix a collection of events $\cE$. We say that the decision maker has $(\cE, \alpha)$-regret if for every $E \in \cE$ and for every $\phi \in \Phi_{\textrm{Ext}}$, $r(\pi_T,u,E,\phi) \leq \alpha(n_T(E, \pi_T))$. 
\end{definition}
Note that although that events are defined as a function of the \emph{predictions} $p_t$, since a straightforward decision maker takes actions as a function of the prediction, we can equally well define events as a function of the actions $a_t$ taken by the decision maker. Observe that for a decision maker with utility function $u$, swap regret corresponds to subsequence regret for the special case of $\cE = \{E_{u,a} : a \in \cA\}$, the set of subsequences on which the straightforward decision maker would take each action.

Finally, we consider \emph{type} regret, which was studied (without a name) by \cite{zhao2021calibrating}. Type regret, informally, is the regret that a straightforward decision maker has to straightforwardly optimizing some \emph{other} utility function $u'$ (or \emph{type}) instead of her own utility function $u$. 

\begin{definition}[Type regret~\cite{zhao2021calibrating}]
Fix a transcript $\pi_T$. The \emph{regret} that a straightforward decision maker with utility function $u$ has with respect to an an alternative utility function $u'$ is:
$$r(\pi_T,u,u') = \frac{1}{T}\sum_{t=1}^T u(\delta_{u'}(p_t),y_t) - u(\delta_u(p_t),y_t)$$
Fix a collection of utility functions $\cU$. We say that the decision maker has $\cU$-regret $\alpha$ if for every $u' \in \cU$, $r(\pi_T,u,u') \leq \alpha$.
\end{definition}
 Type regret is not a special case of $\Phi$ regret, since the comparator class cannot be described by any strategy modification rule. 

\section{Making Unbiased Predictions}

\label{sec:general_algorithm}

In this section we give a general algorithm for obtaining $\alpha$-unbiased predictions with respect to a collection of events $\cE$, with running time scaling polynomially with $d$ and $|\cE|$, and error $\alpha$ scaling only logarithmically with $d$ and $|\cE|$. The algorithm and its analysis consist of two parts.

The first part is a standard reduction from the problem of online multiobjective optimization to minimax optimization, using the framework of \cite{lee2022online}. Similar reductions appear in \cite{gupta2022online,haghtalab2023calibrated}. Rather than giving an analysis from first principles as \cite{lee2022online} do, we follow \cite{haghtalab2023calibrated} in directly reducing to an experts problem (although rather than reducing to a sleeping experts problem as \cite{haghtalab2023calibrated} do, we are able to get improved bounds by using the \texttt{MsMwC} algorithm of \cite{chen2021impossible} for the standard experts problem). This reduces the problem of obtaining online unbiased predictions to the problem of solving for the minimax equilibrium strategy of a particular zero sum game that has exponentially many (in $d$) strategies for each player, which is an obstacle to solving the problem straightforwardly. This step can be viewed as applying a contextual/time-varying variant of Blackwell's approachability theorem \cite{blackwell1956analog}, in which the set to be ``approached'' may change at each round.

The second part of our algorithm solves this minimax problem --- with a polylogarithmic running time dependence in the approximation parameter (via reduction to the Ellipsoid algorithm) in the case of $\cE$ consisting of binary-valued and disjoint events, or with polynomial running time dependence in the approximation parameter using Follow-the-Perturbed-Leader in the more general case. 


To preface the detailed derivation of the algorithm below, we give its pseudocode (presented as a single-round call, such that in the protocol of Section~\ref{sec:prelim} the learner needs to call \texttt{UnbiasedPrediction($\cE, t, \pi_{t-1}, x_t$) at every round $t$}), as well as its performance guarantees. 

A few comments are in order. First, here and in the future we sometimes denote by $\cE_t$ (or $\cE^t$) the projection of the event collection $\cE$ onto round $t$, i.e., the collection of mappings $\cE_t := \{E(\pi_t, \cdot, \cdot): \cX \times \cC \to [0, 1]\}$. This is both for notational convenience as well as to emphasize that our framework allows for the collection of events $\cE$ to be revealed gradually over time, so long as the events $E(\pi_t, \cdot, \cdot)$ are available to the learner by the start of each round $t$. 

Second, in the pseudocode below $g_{1:{t-1}}$ denotes the collection of \emph{event gains} $\{g_1, \ldots, g_{t-1}\}$, which we construct to be used by the experts algorithm (\texttt{MsMwC}) in order to come up with \emph{event weights} $q_t$ in round $t$.\footnote{We do not explicitly pass $g_{1:t-1}$, nor the internal state of \texttt{MsMwC} at the end of the previous round, to \texttt{UnbiasedPrediction}, but it is understood that when \texttt{UnbiasedPrediction} is called at round $t$, this information is readily available for the next call to \texttt{MsMwC} (i.e., \texttt{MsMwC} does not need to recompute its previous trajectory from scratch in each round).}

Third, in the future we may leave out the $\pi_t$ and $x_t$ inputs to \texttt{UnbiasedPrediction} for brevity (and also if the setting is non-contextual and there is no $x_t$ to pass), and may simply write calls to it as \texttt{UnbiasedPrediction($\cE, t$)}.

\begin{algorithm}[H]
\begin{algorithmic}
\STATE Calculate \emph{event gains} from the preceding round $t-1$ (see Section~\ref{sec:minimax-reduction}): 
\[
g_{t-1} \gets (g^{t-1}_{i, \sigma, E})_{i, \sigma, E}, \quad \text{where $g^{t-1}_{(i,\sigma,E)} = \sigma \cdot E(x_{t-1},p_{t-1}) \cdot (p_{t-1,i} - y_{t-1,i})$ for $i \in [d], \sigma = \pm 1, E \in \cE_{t-1}$}
\]
\STATE Compute \emph{event weights} using Multiscale Multiplicative Weights with Correction (see Section~\ref{sec:minimax-reduction}): 
\[
q_{t} = (q_{t,(i,\sigma,E)})_{i, \sigma, E} \gets \texttt{MsMwC}(g_{1:t-1})
\]
\STATE Solve the following minimax problem \emph{up to $\epsilon = 1/t$ error in minimax value}, using either the Ellipsoid method (if the events in $\cE_t$ are binary and disjoint) or FTPL (otherwise) (see Section~\ref{sec:solving}):
\[
\psi_{t} \gets \argmin_{\psi'_{t} \in \Delta \cC} \max_{y \in \cC} \E_{p_{t} \sim \psi'_{t}} \left[\sum_{i=1}^d\sum_{\sigma \in \{-1,1\}}\sum_{E \in \cE_t} q_{t,(i,\sigma,E)}\cdot\sigma \cdot E(x_{t},p_{t})\cdot (p_{t,i} - y_{i}) \right]
\]
\RETURN distribution over predictions $\psi_{t}$

\end{algorithmic}
\caption{\texttt{UnbiasedPrediction}($\cE, t, \pi_{t-1}, x_{t}$)}
\label{alg:unbiased-prediction}
\end{algorithm}

\begin{theorem}[Guarantees for \texttt{UnbiasedPrediction}]
    \label{thm:main-guarantee}
    Given a convex compact prediction space $\cC \subseteq \R^d$ and a collection $\cE$ of events of size $|\cE|$, our algorithm \texttt{UnbiasedPrediction} outputs, on any $T$-round transcript $\pi_T$, a sequence of distributions over predictions $\psi_1, \psi_2, \ldots, \psi_T \in \Delta \cC$ satisfying:
    \[
    \E_{p_t \sim \psi_t \, \forall t} \left[ \left|\sum_{t=1}^T E(x_t,p_t)\cdot (p_{t,i} - y_{t,i})\right| \right]
    \leq
    O\left(\ln(2d|\cE|T)+\sqrt{\ln(2d|\cE|T) \sum_{t=1}^T \E_{p_t \sim \psi_t}[(E(x_t,p_t))^2]}\right).
    \]
    The per-round time complexity is polynomial in $d$, $|\cE|$, and the time it takes to evaluate each $E \in \cE$. When the events are binary and disjoint, the running time is polylogarithmic in $t$ at any round $t \in [T]$; in the general case the running time is polynomial in $t$ at any round $t \in [T]$.
\end{theorem}
\begin{proof}
    The theorem follows by combining the non-constructive minimax bound (i.e., the bound that assumes the minimax problem in each round $t \in [T]$ is solved up to an $\epsilon_t$ error in minimax value) of Theorem~\ref{thm:nonconstructive-expectation} with the runtime guarantees for solving the minimax game with binary disjoint events (Theorem~\ref{thm:ellipsoid}) and with general, not necessarily binary or disjoint, events (Theorem~\ref{thm:FTPL-No-Sampling}).
\end{proof}

\begin{remark}[(Nearly-)Anytime Guarantee] \label{rem:anytime}
    As formulated, our \texttt{UnbiasedPrediction} procedure does not rely on knowing the time horizon $T$, except through its oracle invocations of \texttt{MsMwC}. It can be checked that \texttt{MsMwC} only needs an upper bound on $T$ --- call it $T_\mathrm{max}$ --- and thus by setting the internal components of \texttt{MsMwC} to depend on $T_\mathrm{max}$, \footnote{Specifically, the internal learning rate $\eta$ of \texttt{MsMwC} should be set proportional to $\sqrt{\log T_\mathrm{max}}$, and its internal optimization should be performed over all weight vectors whose individual components are no less than a multiple of $1/T_\mathrm{max}$.} we automatically ensure that the bias of \texttt{UnbiasedPrediction} will be: 
    \[O\left(\ln(2d|\cE|T_\mathrm{max})+\sqrt{\ln(2d|\cE|T_\mathrm{max}) \sum_{t=1}^T \E_{p_t \sim \psi_t}[(E(x_t,p_t))^2]}\right) \quad \text{after every round $T$ such that $T \leq T_\mathrm{max}$}.\]
    The property of our online bound holding at all intermediate rounds, until the final time horizon, can be referred to as \emph{anytime}. 
    
    It should be noted that algorithms in online learning are often only considered fully anytime if they have \emph{no} dependence at all on $T_\mathrm{max}$. Although  \texttt{UnbiasedPrediction} does have a dependence on $T_\mathrm{max}$, (1) the regret bound depends on $T_\mathrm{max}$ only logarithmically and so is insensitive to even large over-estimates; and (2) the per-round runtime of \texttt{UnbiasedPrediction} (and of its subroutine \texttt{MsMwC}) does not depend on $T_\mathrm{max}$; rather, it scales with $t$ at every round $t \in [T]$.
\end{remark}

\subsection{Reduction to a Minimax Problem}
\label{sec:minimax-reduction}
We follow \cite{gupta2022online,lee2022online,haghtalab2023calibrated} in reducing a multicalibration-like problem to a multi-objective learning problem, in which for each  event $E \in \cE$ and each coordinate of our predictions $i \in [d]$, we have two objective values, one corresponding to the positive bias of our predictions, and one corresponding to the negative bias of our predictions. Our goal is to make predictions so that the maximum value across all of these $2d\cdot |\cE|$ objectives is small, leading to small bias with respect to events $\cE$. 

We encode these $2d|\cE|$ objectives as an experts problem with $n = 2d|\cE|$ experts.  We index the experts by $(i, \sigma,E)$ for $i \in [d]$ representing a coordinate, $\sigma \in \{-1,1\}$ representing a sign, and $E \in \cE$ representing an event. At round $t$, if the learner predicts $p_t \in \cC$ and the adversary produces outcome $y_t \in \cC$, then the gains for each expert are defined to be:
\[
g^t_{(i,\sigma,E)} = \sigma \cdot E(x_t,p_t)\cdot (p_{t,i} - y_{t,i})
\]

Recalling our w.l.o.g.\ assumption that $2\max_{y \in \cC}||y||_\infty \leq 1$, we observe that $g^t_{(i,\sigma,E)} \in [-1,1]$, and that the cumulative gain of each expert $(i,\sigma,E)$ up through round $t$ is:
$\sigma \cdot \sum_{\tau=1}^t E(x_\tau,p_\tau)\cdot (p_{\tau,i} - y_{\tau,i})$, 
which is exactly the $\sigma$-signed bias of the predictions through round $t$ in coordinate $i$ conditional on event $E$.  

Recall that our goal is to make each of these $\sigma$-signed event-conditional bias terms small over time, i.e., each such term should over time diminish at a rate $\alpha(n_T(E))$, where $E$ is the number of rounds on which the corresponding event $E$ was active. To assist us in this, we invoke the MsMwC experts no-regret algorithm of~\cite{chen2021impossible}, which was several years ago used to resolve the ``impossible tuning'' problem in online learning with experts.

\paragraph{Tool: MsMwC, A Small Loss Algorithm for the Experts Problem}
\label{sec:experts}
Multi-Scale Multiplicative Weights With Correction (MsMwC) \citep{chen2021impossible} is an algorithm that can be applied to the experts learning problem in a way that simultaneously gets regret bounds to every expert that scale with the cumulative (squared) gain of that expert. Recall that in the experts learning setting, the learning algorithm must sequentially choose amongst $n$ \emph{experts} in rounds $t \in \{1,2,\ldots\}$, whose \emph{gains} (or losses) are determined by the adversary. In each round $t$:
\begin{enumerate}
    \item The algorithm chooses a distribution $q_t \in \Delta [n]$ over the experts, and then:
    \item A vector of gains for the experts  $g_t \in [-1,1]^n$ is chosen by an adversary, and the algorithm experiences gain $\hat g_t = q_t \cdot g_t$. 
\end{enumerate}

The \emph{regret} that the algorithm has to expert $i$ after $t$ rounds is defined as $R_{t,i} = \sum_{\tau=1}^t \left(g_{t,i}-\hat g_t\right)$, and the goal in this setting is to make the average regret to each expert diminish over time. Classic algorithms for this problem, such as the Multiplicative Weights Update, attain $O(\sqrt{T})$ regret bounds relative to each expert; such regret bounds are not \emph{adaptive}, in the sense that they don't depend on the profile of gains of each expert and thus cannot exploit any possible observed trends in each particular expert's gains. In contrast,
the algorithm of \cite{chen2021impossible} guarantees that we can simultaneously bound the regret to each expert $i$, $R_{t,i}$, by a quantity that scales with the cumulative (squared) gain of expert $i$: 

\begin{theorem}[Special Case of Theorem 1 from ~\cite{chen2021impossible}]
\label{thm:impossible}
    There exists an algorithm (\text{MsMwC}) with per-round running time polynomial in $n$,  that simultaneously for every expert $i$, guarantees regret:
    $$R_{T,i} \leq O\left(\ln(nT)+\sqrt{\ln(nT) \sum_{t=1}^T g_{t,i}^2}\right)$$
\end{theorem}

\paragraph{Applying MsMwC to event-conditional gains:} We may now apply the regret-bound for MsMwC in Theorem~\ref{thm:impossible}) to find that:
\begin{corollary}[MsMwC applied to event gains]
\label{cor:minimax-reduction}
When using the gains constructed in this section, MsMwC plays a sequence of distributions $q_t \in \Delta [2d|\cE|]$ that satisfies, for every triple $(i^*,\sigma^*,E^*) \in [d]\times\{-1,+1\}\times\cE$,
\begin{align*}
    &\sum_{t=1}^T \sum_{i \in [d], \sigma = \pm 1, E \in \cE} 
\!\!\!\!\!\!\! q_{t,(i,\sigma,E)}\cdot\sigma \cdot E(x_t,p_t)\cdot (p_{t,i} - y_{t,i}) \\
&\geq \sigma^* \cdot \sum_{t=1}^T E^*(x_t,p_t)\cdot (p_{t,i^*} - y_{t,i^*}) - O\left(\ln(2d|\cE|T)+\sqrt{\ln(2d|\cE|T) \sum_{t\in[T]} (E^*(x_t,p_t))^2}\right).
\end{align*}
\end{corollary}

Thus, to upper bound the bias in our predictions across all coordinates and all events, it suffices to  make predictions $p_t$ that upper bound the gain of MsMwC. Towards this end, we define a zero sum game between the learner and the adversary with objective function equal to the per-round expected gain of MsMwC:
\begin{equation}
\label{eq:utility}
u_t(p,y) = \sum_{i=1}^d\sum_{\sigma \in \{-1,1\}}\sum_{E \in \cE} q_{t,(i,\sigma,E)}\cdot\sigma \cdot E(x_t,p)\cdot (p_{i} - y_{i})
\end{equation}

Non-constructively, we can argue that the learner has a prediction strategy which guarantees that the gain of MsMwC (and hence by construction, the bias of the cumulative predictions, conditional on any event $E \in \cE)$ is small. We will make use of Sion's minimax theorem.

\begin{theorem}[\cite{sion1958minimax}]
Let $X\subset \mathbb R^n$ and $Y\subset \mathbb R^m$ be convex and compact sets, and $u:X \times Y \to \mathbb R$ be a objective function, such that $u(x,\cdot)$ is upper semi-continuous and quasi-concave for all $x\in X$ and $u(\cdot, y)$ is lower semi-continuous and quasi-convex for all $y\in Y$.  Then:
\begin{equation*}
    \min_{x\in X} \max_{y\in Y} u(x,y) = \max_{y\in Y} \min_{x\in X} u(x,y).
\end{equation*}
\end{theorem}

The learner will play the role of the minimization player in this zero sum game, and the adversary will play the role of the maximization player. Their pure strategy spaces will both be (subsets of) $\cC$, the prediction space. 
Observe that $u_t(p,y)$ is linear in the adversary's action $y$ (and hence concave), and so we may take the adversary's action space to be $A_2 = \cC$. On the other hand, $u_t(p,y)$ is \emph{not} necessarily convex in $p$, because of the functions $E(x_t,p)$ (which may be arbitrary), and so to satisfy the conditions of Sion's minimax theorem, we must discretize the learner's action space and allow them to randomize. 

\begin{definition}
    Let $\cC_\epsilon$ be any finite $\epsilon$-net of $\cC$ in the $\ell_\infty$ norm --- i.e. any finite set $\cC_\epsilon \subset \cC$ such that for all $p \in \cC$, there is a $\hat p \in \cC_\epsilon$ such that $||\hat p - p||_\infty \leq \epsilon$. 
\end{definition}

We take the learner's action space in this game to be $A_1 = \Delta \cC_\epsilon$, the set of distributions over an $\epsilon$-net of the prediction space $\cC$. This is a convex compact action space, and the utility function is linear over $A_1$ (by linearity of expectation). 

We can compute the minimax value of the game by considering the world in which the adversary moves first, picking a realization $y \in \cC$, and then the learner best responds with some $p \in \cC_\epsilon$: Since $\cC_\epsilon$ is an $\epsilon$-net of $\cC$, the learner can in particular choose $p$ satisfying $|p_{i}-y_{i}| \leq \epsilon$ in every coordinate $i$. Together with Corollary \ref{cor:minimax-reduction}, this gives us the following theorem: 

\begin{theorem}
\label{thm:nonconstructive-expectation}
If at every round $t \in [T]$ the learner samples a prediction from $\psi_t$, a $(1/t)$-approximate minimax strategy for the zero-sum game with utility function $u_t$ defined in Equation~\ref{eq:utility}, then for every coordinate $i$ and every event $E \in \cE$, the expected bias in coordinate $i$ conditional on event $E$ is bounded by:
\begin{align*}
    \E_{p_t \sim \psi_t \, \forall t} \left[ \left|\sum_{t=1}^T E(x_t,p_t)\cdot (p_{t,i} - y_{t,i})\right| \right]
    & \leq O\left(\ln(2d|\cE|T) + \sqrt{\ln(2d|\cE|T) \cdot\sum_{t=1}^T \E_{p_t \sim \psi_t} [(E(x_t,p_t))^2]}\right).
\end{align*}
\end{theorem}

\begin{proof}
Let $\epsilon_t = \frac{1}{t}$.
By definition of an $\epsilon_t$-net, for every $y \in \cC$ we are guaranteed that there is a $p \in \cC_{\epsilon_t}$ with $||p - y||_\infty \leq \epsilon_t$, and thus:    
\[
\max_{y \in \cC} \min_{p \in \cC_{\epsilon_t}} u_t(p,y) 
\leq \max_{E \in \cE, i \in [d]} \epsilon_t \cdot \left|q_{t, (i, \sigma, E)} \cdot E(x_t, p)\right| \leq \epsilon_t.
\]
Thus, Sion's minimax theorem implies that:
\[
\min_{\psi \in \Delta \cC_{\epsilon_t}}\max_{y \in \cC} \E_{p \sim \psi}[u_t(p,y)] 
\leq \epsilon_t.
\]
Let $\psi_t$ be any $\epsilon_t$-approximate minimax distribution of this game, such that:
$
\max_{y \in \cC} \E_{p_t \sim \psi_t}[u_t(p_t,y)] \leq \epsilon_t.
$
If the learner samples a prediction from $\psi_t$ at each round $t$, then by definition of the utility function $u_t$, the expected loss of MsMwC will be bounded by $\epsilon_t$ at each round $t$. By Corollary \ref{cor:minimax-reduction}, after taking expectations with respect to the randomness in the learner's play according to the distributions $\psi_t$ across rounds $t \in [T]$, we have that the expected bias in each coordinate $i$ conditional on each event $E \in \cE$ is bounded as:
\begin{align*}
    \E_{p_t \sim \psi_t \, \forall t} \left[ \left|\sum_{t=1}^T E(x_t,p_t)\cdot (p_{t,i} - y_{t,i})\right| \right]
    &\leq \sum_{t=1}^T \epsilon_t + O\left(\ln(2d|\cE|T) + \!\!\!\!\E_{p_t \sim \psi_t \, \forall t} \sqrt{\ln(2d|\cE|T) \cdot \sum_{t \in [T]} (E(x_t,p_t))^2} \right) \\
    &\leq O(\ln T) + O\left(\ln(2d|\cE|T) + \!\!\!\!\E_{p_t \sim \psi_t \, \forall t} \sqrt{\ln(2d|\cE|T) \cdot \sum_{t\in [T]} (E(x_t,p_t))^2} \right) \\
    &\leq O\left(\ln(2d|\cE|T) + \sqrt{\ln(2d|\cE|T) \cdot \sum_{t\in[T]} \E_{p_t \sim \psi_t} [(E(x_t,p_t))^2]}\right),
\end{align*}
where the second line follows by bounding the harmonic series, and the final line is by Jensen's inequality with respect to the (concave) square root function.
\end{proof}

Theorem \ref{thm:nonconstructive-expectation} gives a (non-constructive) algorithm for producing predictions that have bounded bias conditional on all events in $\cE$. The difficulty is that the algorithm requires playing an approximate minimax strategy for a game in which the learner has an exponentially large strategy space, and the adversary has a continuously large strategy space. In Section \ref{sec:solving}, we show how to make this algorithm constructive by giving polynomial time algorithms for solving this minimax problem.

\subsection{Solving the Minimax Problem}
\label{sec:solving}

\subsubsection{For Disjoint Binary Convex Events}
\label{sec:disjoint-minimax}
In this section we show how to find an $\epsilon$-approximate solution to the minimax problem defined in Section \ref{sec:minimax-reduction}, with running time that is polynomial in $d$, $|\cE|$, and $\log(1/\epsilon)$ in the case in which the events $E \in \cE$ are binary valued and disjoint: for all $x, p$: 
$\sum_{E \in \cE}E(x,p) \leq 1$. We will also assume that for every history $\pi$ and context $x$, the predictions $p$ that satisfy $E(\pi,x,p) = 1$ form a convex set for which we have a polynomial time separation oracle. In Section \ref{sec:general-minimax} we will show how to solve the same minimax problem in the general case, without any of these assumptions (assuming only that we can \emph{evaluate} the quantity $E(\pi,x,p)$ in polynomial time) --- with the caveat that our general solution will have a running time dependence that is polynomial in $1/\epsilon$ rather than $\log 1/\epsilon$. Proofs from this Section can be found in Appendices \ref{app:LP} and \ref{app:proofs-disjoint}.

Following the reduction to a zero-sum game in Section \ref{sec:minimax-reduction}, our goal is to solve for the learner's equilibrium strategy $\psi_t \in \Delta(\cC)$  in the game with utility function 
$$u_t(p,y) = \sum_{i=1}^d\sum_{\sigma \in \{-1,1\}}\sum_{E \in \cE} q_{t,(i,\sigma,E)}\cdot\sigma \cdot E(\pi_{t-1},x_t,p)\cdot (p_{i} - y_{i}) $$
corresponding to the per-round gain of MsMwC. In other words, we need to approximately solve the minimax problem defined as: 
\begin{equation}
\label{eq:minmax}
\psi_t^* = \argmin_{\psi \in \Delta(\cC)} \max_{y\in \cC} \E_{p\sim \psi}[u_t(p, y)]
\end{equation}
By relaxing the minimization player's domain from $\cC$ to $\Delta(\cC)$, the set of distributions over predictions, we have made the objective linear (and hence convex/concave), but we have continuously many optimization variables --- both primal variables (for the minimization player) and dual variables (for the maximization player). Our strategy for solving this problem in polynomial time will be to argue that it has a solution in which only $|\cE|$ many primal variables take non-zero values, that we can efficiently identify those variables, and that we can implement a separation oracle for the dual ``constraints'' in polynomial time. This will allow us to construct a reduced but equivalent linear program that we can efficiently solve with the Ellipsoid algorithm. 

We first observe that in the utility function $u_t(p,y)$, the learner's predictions $p$ ``interact'' with the outcomes $y$ only through the activation of the events $E(\pi_{t-1},x_t,p)$. This implies that \emph{conditional} on the values of the events $E(\pi_{t-1},x_t,p)$, there is a unique $p_t$ that minimizes $u_t(p,y)$ \emph{simultaneously} for all $y$. In general the collection of events $E(\pi_{t-1},x_t,p)$ could take on many different combinations of values --- but our assumption in this section that the events are disjoint and binary means that there are in fact only $|\cE|$ different values for us to consider. The predictions we need to consider are defined by the following efficiently solvable convex programs:

\begin{definition}
\label{def: eventoptimal}
    For $E \in \cE$, let $p_t^{*,E}$ be a solution to the following convex program (selecting arbitrarily if there are multiple optimal solutions): 
         \begin{equation*}
        \begin{array}{ll@{}ll}
        \mathrm{minimize}_{p \in \cC}  & \displaystyle \sum_{i=1}^d\sum_{\sigma \in \{-1,1\}} q_{t,(i,\sigma,E)}\cdot\sigma \cdot p_{i} &\\
        \mathrm{subject \, to}& \\
                         & E(\pi_{t-1},x_t,p) = 1
        \end{array}
    \end{equation*}

    Let $\cP_t = \{p_t^{*,E}\}_{E \in \cE}$ be a collection of $|\cE|$ vectors in $\cC$ constituting solutions to the above programs. 
\end{definition}
\begin{remark}
    As we have assumed in this section, the set of $p$ such that $E(\pi_{t-1},x_t,p)=1$ is a convex region endowed with a separation oracle, and so these are indeed convex programs that we can efficiently solve with the Ellipsoid algorithm. This is often the case: for example, if we have a decision maker with a utility function $u$ over $K$ actions, the disjoint binary events $E_{u,a}$ (for each action $a \in [K]$) are defined by $K$ linear inequalities, and so form a convex polytope with a small number of explicitly defined constraints. This will turn out to be the collection of events relevant for obtaining diminishing swap regret for downstream decision makers. 
\end{remark}

We next verify that the prediction values defined in Definition \ref{def: eventoptimal} are  best responses for the minimization player against all possible realizations $y_t$ that the maximization player might choose, \emph{conditional} on a positive value of a particular event:

\begin{restatable}{lemma}{lemeventoptimal}
\label{lem:eventoptimal}
    Simultaneously for all $y \in \cC$, we have:
        $$p_t^{*,E} \in \argmin_{p :E(\pi_{t-1},x_t,p) = 1} u_t(p,y). $$
\end{restatable}

A consequence of this is that solutions to the following reduced minimax problem (which now has only $|\cE|$ variables for the minimization player --- the weights defining a distribution over the $|\cE|$ points $p_t^{*,E}$ ) are also solutions to our original minimax problem \ref{eq:minmax}:
\begin{equation}
\label{eq:minmax-reduced}
\psi_t^* = \argmin_{\psi \in \Delta(\cP_t)} \max_{y \in \cC} \E_{p\sim \psi}[u_t(p, y)]
\end{equation}

\begin{restatable}{lemma}{lemtwoimpliesone}
\label{lem:2implies1}
    Fix any optimal solution $\psi_t^*$ to minimax problem \ref{eq:minmax-reduced}. Then $\psi_t^*$ is also an optimal solution to minimax problem \ref{eq:minmax}.
\end{restatable}

Thus, to find a solution to minimax problem \ref{eq:minmax}, it suffices to find a solution to minimax problem \ref{eq:minmax-reduced}. Minimax problem \ref{eq:minmax-reduced} can be expressed as a linear program with $|\cE|+1$ variables but with continuously many constraints, one for each $y \in \cC$:
\begin{equation}
\label{eq:minmax-reduced-lp}
            \begin{array}{ll@{}ll}
            \text{minimize}_{\psi \in \Delta(\cP_t)}  & \gamma &\\
            \text{subject to}& \\
                             & \E_{p \sim  \psi}[u_t(p,y)] \le \gamma & \ \ \ \  \forall y \in \cC  \\
            \end{array}
\end{equation}
We can find an $\epsilon$-approximate solution to a polynomial-variable linear program using the Ellipsoid algorithm in time polynomial in the number of variables and $\log(1/\epsilon)$  so long as we have an efficient \emph{separation oracle} --- i.e., an algorithm to find an $\epsilon$-violated constraint whenever one exists, given a candidate solution. In this case, implementing a separation oracle corresponds to computing a \emph{best response} for the adversary (the maximization player) in our game---and since the utility function in our game given in Equation \ref{eq:utility} is \emph{linear} in the Adversary's chosen action $y$, implementing a separation oracle corresponds to solving a linear maximization problem over the convex feasible region $\cC$---a problem that we can solve efficiently assuming we have a separation oracle for $\cC$. There are a number of technical details involved in making this rigorous, which can be found in Appendix \ref{app:LP}. Here we state the final algorithm and guarantee.

\begin{algorithm}[H]
\begin{algorithmic}
\FOR{$E \in \cE$}
\STATE Solve the convex program from Definition \ref{def: eventoptimal} to obtain $p^{*,E}$.
\ENDFOR
\STATE Let $\cP_t = \{p^{*,E}\}_{E \in \cE}$.
\STATE Solve linear program \ref{eq:minmax-reduced-lp} over $\cP_t$ using the weak Ellipsoid algorithm to obtain solution $\psi_t'$.
\IF{$\psi_t'\notin \Delta(\cP)$}
\STATE Let $\psi_t^*$ be the Euclidean projection of $\psi_t$ onto $\Delta(\cP)$ returned by the simplex projection algorithm.
\ELSE
\STATE Let $\psi_t^* = \psi_t'$.
\ENDIF
\RETURN $\psi_t^*$
\end{algorithmic}
\caption{\texttt{Get-Approx-Equilibrium-LP($t, \epsilon, \cE$)}}
\label{alg:minmax-LP}
\end{algorithm}

\begin{restatable}{theorem}{ellipsoid}
\label{thm:ellipsoid}
    Given a polynomial-time separation oracle for $\cC$, for any $\epsilon > 0$, there exists an algorithm (Algorithm \ref{alg:minmax-LP}) that returns an $\epsilon$-approximately optimal solution $\psi_t^*$ to minimax problem \ref{eq:minmax} and runs in time polynomial in $d$, $|\cE|,\log(\frac{1}{\epsilon})$.
\end{restatable}

\subsubsection{The General Case}
\label{sec:general-minimax}

For a general collection of events --- which are not necessarily disjoint, may be real valued, and may not correspond to convex subsets of the prediction space --- the technique we gave in Section \ref{sec:disjoint-minimax} for approximately solving the game with a polylogarithmic dependence on the approximation parameter no longer applies. Nevertheless, by simulating play of the zero-sum game that we wish to solve using appropriately chosen learning dynamics, we can still compute an $\epsilon$-approximate minimax strategy in time polynomial in $1/\epsilon$.

At a high level, our strategy will be to simulate play of the zero-sum game with objective function defined in Equation \ref{eq:utility}. The maximization player (the adversary) will play according to the Follow-the-Perturbed-Leader algorithm introduced below. This allows us to efficiently minimize regret over the adversary's high dimensional action space because we can express the adversary's utility function as a linear function of their own action, which casts their learning problem as an instance of online linear optimization. Most straightforwardly, we would like the minimization player (the Learner) to \emph{best respond} to the Adversary's actions at each round: but the Learner's utility function is complex, because of interactions with the events $E$, which can be arbitrarily defined. Thus it is not clear how to efficiently compute a best response for the learner. Nevertheless, the structure of the utility function makes it such that they can guarantee themselves the value of the game at each round simply by copying the Adversary's strategy, which is computationally easy. Although this is not necessarily a best response, it suffices for our purposes: standard arguments will imply that the time-averaged play for the Learner converges to an approximate minimax strategy of the game.

\paragraph{Tool: FTPL, An Oracle Efficient Algorithm for Online Linear Optimization}
In the present setting of general (not necessarily disjoint) events, we will make use of \emph{Follow the Perturbed Leader (FTPL)}, an oracle efficient algorithm for online linear optimization \cite{hannan1957approximation,kalai2005efficient}. In the online linear optimization problem, the learner has a decision space $\cD \subseteq \mathbb{R}^n$ and the adversary has a state space $\cS \subseteq \mathbb{R}^n$. In rounds $t \in \{1,\ldots,\}$:
\begin{enumerate}
\item The algorithm chooses a distribution over decisions $D_t \in \Delta \cD$;
\item The adversary chooses a state $s_t \in \cS$;
\item The algorithm experiences gain $\hat g_t = \E_{d_t \sim D_t}[\langle d_t, g_t\rangle] = \langle \E_{d_t \sim D_t} [d_t], g_t \rangle$.
\end{enumerate}

After $t$ rounds, the regret of the algorithm to decision $d \in \cD$ is: $R_{t,d} = \sum_{\tau=1}^t\left( \langle d, s_\tau \rangle - \hat g_\tau \right)$.

In contrast to the experts problem defined in Section \ref{sec:experts}, in which the number of choices for the algorithm at each round is $n$, here the algorithm can choose any element of $\cD$ at each round, which may be exponentially large in $n$ (or continuously large). An algorithm (``Follow the Perturbed Leader'') of \cite{kalai2005efficient} gives an \emph{oracle efficient} algorithm for obtaining diminishing regret guarantees for this problem, which means that it runs in time polynomial in $n$ and the time needed to solve linear optimization problems over $\cD$ --- i.e. to solve problems of the form $d_t = \argmax_{d \in \cD} \langle d, s \rangle$ for any vector $s \in \mathbb{R}^n$. The algorithm is simple and maintains its distribution $D_t$ at each round implicitly by giving an efficient sampling algorithm: At round $t$, $D_t$ is the distribution corresponding to the sampling algorithm:
\begin{enumerate}
    \item Let $s = \sum_{\tau=1}^{t-1} s_\tau + z$, where $z \in \mathbb{R}^n$ is such that each coordinate $z_i$ is sampled $z_i \sim \textrm{Unif}[0,1/\delta]$ from the uniform distribution over the interval $[0,1/\delta]$ for some parameter $\delta$.
    \item Let $d_t = \argmax_{d \in \cD} \langle d, s \rangle$.
\end{enumerate}

\begin{theorem}[\cite{kalai2005efficient}]
\label{thm:kalai-vempala}
    Let $\Delta = \sup_{d_1,d_2 \in \cD}||d_1-d_2||_1$ be the diameter of the decision space, let $K = \sup_{d \in \cD, s \in \cS}|\langle d, s \rangle|$ be the maximum per-round gain of the algorithm, and let $A = \sup_{s \in \cS}||s||_1$ be the maximum norm of any state. Then for any $\delta \leq 1$, Follow the Perturbed Leader has regret at round $T$  to each decision $d$ that is at most: 
    $$R_{T,d} \leq \delta KAT + \frac{\Delta}{\delta}.$$
    Choosing $\delta = \sqrt{\frac{\Delta}{KAT}}$ gives:
    $$\max_{d \in \cD} R_{T,d} \leq 2\sqrt{\Delta K A T}.$$
\end{theorem}


\paragraph{Simplifying the Adversary's objective:} When thinking of the Adversary's optimization problem, it is helpful to elide the parts of the objective function (from Equation \ref{eq:utility}) that are independent of the adversary's actions. Towards this end we define the following objective:

\begin{equation*}
    u_t'(p, y) = \sum_{i = 1}^d -y_{i} \cdot s_{t,i}(p) = \langle -y, s_t(p) \rangle,
\end{equation*}
where for any $p \in \cC$, we define:
\begin{equation}
\label{eq:FTLP-state}
    s_{t, i}(p) = \sum_{\sigma \in \{-1,1\}}\sum_{E \in \cE} q_{t,(i,\sigma,E)}\cdot\sigma \cdot {E(x_t,p)}
\end{equation}
for each $i \in [d]$. Note the relationship between the two objective functions:
\begin{equation}
\label{eq:objective-relationship}
    u_t(p, y) = \langle (p - y), s_t(p) \rangle = \langle p, s_t(p) \rangle + u_t'(p, y).
\end{equation}

Since the difference between $u_t(p,y)$ and $u'_t(p,y)$ (the term $\langle p, s_t(p) \rangle$) is independent of the Adversary's chosen action $y$, the Adversary's regret as measured with respect to $u'_t$ is identical to regret as measured with respect to $u_t$; thus from this point onwards, we imagine the Adversary to be optimizing for $u'_t$.

Observe that fixing $p$ (and hence $s_t(p)$), $u'_t(p,\cdot)$ is a linear gain function of exactly the form that Follow-the-Perturbed-Leader is designed to optimize. To implement Follow-the-Perturbed-Leader, we will need to assume the existence of an online linear optimization oracle over $\cC$ --- i.e. an optimization algorithm  $M: \R^d \to \cC$ which solves:
\begin{equation*}
    M(s) \in \argmax_{y \in \cC} ~ \langle -y, s \rangle.
\end{equation*}

The complete algorithm is stated in Algorithm \ref{alg:FTPL-No-Sampling}. 

\begin{algorithm}[ht]
\begin{algorithmic}
\STATE Initialize $s_t(p_0)$ to 0. \\
\STATE Set $T' = \frac{2dC^2}{\epsilon'^2}$, $\delta = \sqrt{\frac{2d}{T'}}$
\STATE Fix distribution $\cZ = \text{Unif}[0,\frac{1}{\delta}]^d$.
\FOR{${\tau}=1, \dots, T'$}
\STATE Compute $y_{\tau} = \E_{z \in \cZ}\left[ M\left(\sum_{k = 0}^{{\tau}-1} s_{t}(p_k) + z \right)\right]$ \tcp{Adversary chooses the FTPL distribution} 
\STATE Set $p_{\tau} = y_{\tau}$ \tcp{Learner copies Adversary's expected action}
\STATE Compute $s_{t}(p_{\tau})$ as defined in Equation \ref{eq:FTLP-state}.
\ENDFOR
\STATE Define $\overline{p}$ as the uniform distribution over the sequence $(p_1, p_2, \cdots , p_{T'})$.
\RETURN $\overline{p}$
\end{algorithmic}
\caption{\texttt{Get-Approx-Equilibrium($t,\epsilon'$)}}
\label{alg:FTPL-No-Sampling}
\end{algorithm}
\begin{theorem}
\label{thm:FTPL-No-Sampling}
For any $t \in [T]$ and $\epsilon' > 0$, Algorithm \ref{alg:FTPL-No-Sampling} returns a distribution over actions $\overline{p}$ which is an $\epsilon'$-approximate minimax equilibrium strategy for the zero-sum game with objective $u_t$. 
\end{theorem}

To prove Theorem \ref{thm:FTPL-No-Sampling}, we first establish an intermediate fact --- that within the learning dynamics simulated in Algorithm \ref{alg:FTPL-No-Sampling}, the Adversary has low regret. The proof can be found in Appendix \ref{app:general-minimax}, and just instantiates the guarantees of Follow the Perturbed Leader in our setting. 
\begin{restatable}{lemma}{lemFTPLConstants}
\label{lem:FTPL-Constants}
Playing FTPL for $T'$ rounds with perturbation parameter $\delta = \sqrt{\frac{2d}{T'}}$, for a learner with decision space $\cC$ and adversary with state space $\cS = \{s_t(p)~| ~p \in \cC\}$ yields a regret bound: 
\begin{equation*}
    R_{T', y} \leq  \sqrt{2dC^2 T'}
\end{equation*}
for all $y \in \cC$, where $d$ is the dimension of all vectors in $\cC$ and $C = 2 \max_{y \in \cC} \lVert y \rVert_{\infty}$. 
\end{restatable}

\begin{proof}[Proof of Theorem \ref{thm:FTPL-No-Sampling}]
Define $y^*$ as the adversary's best response in hindsight to the realized sequence of actions taken by the learner during the run of Algorithm \ref{alg:FTPL-No-Sampling}:
\begin{equation*}
    y^* = \argmax_{y \in \cC} \sum_{\tau = 1}^{T'} u_t'(p_{\tau}, y^*) = \argmax_{y \in \cC} \sum_{\tau = 1}^{T'} u_t(p_{\tau}, y),
\end{equation*}
where in the second equality we use the fact that the Adversary's best-response function is identical under $u$ and $u'$. 
Let $D_\tau$ denote the adversary's FTPL distribution over actions at round $\tau$ --- i.e. the distribution defined by  $M\left(\sum_{k = 0}^{{\tau}-1} s_{t}(p_k) + Z\right)$, where $Z \sim \cZ$. Looking first at the expected gain of the adversary over the realized transcript, we see that:
\begin{equation*}
    \sum_{\tau=1}^{T'} \E_{y \in D_{\tau}} [u_t'(p_{\tau}, y)] \geq \sum_{\tau = 1}^{T'} u_t'(p_{\tau}, y^*) - R_{T', y^*},
\end{equation*}
where $R_{T', y^*}$ is the adversary's regret from playing FTPL. The cumulative utility that the adversary would have achieved with their best fixed action in hindsight (as measured with respect to the original utility function $u_t$) can now be bounded as:
\begin{align*}
    \max_{y \in \cC} \sum_{\tau = 1}^{T'} u_t(p_{\tau}, y) = \sum_{\tau = 1}^{T'} u_t(p_{\tau}, y^*) &= \sum_{\tau = 1}^{T'} u_t'(p_{\tau}, y^*) + \sum_{\tau = 1}^{T'}\langle p_{\tau}, s_t(p_{\tau}) \rangle \\
    &\leq \left(R_{T', y^*}  + \sum_{\tau=1}^{T'} \E_{y \sim D_{\tau}}\left[u_t'(p_\tau, y)\right] \right) + \sum_{\tau = 1}^{T'}\langle p_{\tau}, s_t(p_{\tau}) \rangle. 
\end{align*}
The first line follows from Equation \ref{eq:objective-relationship}. To simplify the expected sum of utilities $\sum_{\tau=1}^{T'} \E_{y \sim D_{\tau}}\left[u_t'(p_\tau, y)\right]$, we use the fact that at each round $\tau$ in Algorithm \ref{alg:FTPL-No-Sampling}, the learner's response $p_{\tau}$ is defined to be exactly the expectation of the adversary's distribution $D_{\tau}$. So, the expected utility under the \textit{original} objective function $u_{\tau}$ for the realized sequence of plays can be computed for each $\tau$:
\begin{align*}
    \E_{y \sim D_{\tau}} \left[u_t(p_{\tau}, y)\right] = \E_{y \sim D_{\tau}} \left[ \langle (p_{\tau} - y), s_t(p_{\tau})\rangle\right] &= \left\langle \left(p_{\tau}-\E_{y \sim D_{\tau}}[y] \right), s_t(p_{\tau}) \right\rangle = 0
\end{align*}
Using Equation $\ref{eq:objective-relationship}$ once more, 
\begin{align*}
    &0 = \sum_{\tau=1}^{T'} \E_{y \sim D_{\tau}} \left[u_t(p_{\tau}, y)\right] = \sum_{\tau=1}^{T'} \E_{y \sim D_{\tau}} \left[u_t'(p_{\tau}, y)\right] + \sum_{\tau=1}^{T'} \langle p_{\tau}, s_t(p_{\tau}) \rangle \\ \implies& \sum_{\tau=1}^{T'} \E_{y \sim D_{\tau}} \left[u_t'(p_{\tau}, y)\right] = - \sum_{\tau=1}^{T'} \langle p_{\tau}, s_t(p_{\tau}) \rangle
\end{align*} 
Substituting this into the bound on the utility of the adversary's best response in hindsight:
\begin{align*}
    \max_{y \in \cC} \sum_{\tau = 1}^{T'} u_t(p_{\tau}, y) &\leq \left(R_{T', y^*} -\sum_{\tau = 1}^{T'}\langle p_{\tau}, s_t(p_{\tau}) \rangle \right) + \sum_{\tau = 1}^{T'}\langle p_{\tau}, s_t(p_{\tau}) \rangle  \\
    &= R_{T', y^*} \\
    &\leq  \sqrt{2dC^2 T'},
\end{align*}
with the last inequality following directly from Lemma \ref{lem:FTPL-Constants}. Notice that for any fixed strategy $y \in \cC$:
\begin{equation*}
    \frac{1}{T'} \sum_{i=1}^{T'}  u({p_{\tau}, y}) = \E_{p \sim \overline{p}} [u_t(p, y)],
\end{equation*}
where $\overline{p}$ is the learner's mixed strategy returned from Algorithm \ref{alg:FTPL-No-Sampling}, the uniform distribution over all realized vectors $p_{\tau}$ in the transcript. Thus,
\begin{equation*}
    \max_{y \in \cC} \E_{p \sim \overline{p}} [u_t(p, y)] \leq \frac{ \sqrt{2dC^2 T'}}{T'},
\end{equation*}
and so $\overline{p}$ is a $\sqrt{\frac{2dC^2}{T'}}$-approximate equilibrium strategy for the learner. By our choice of $T' = \frac{2dC^2}{\epsilon'^2}$, this simplifies to $\epsilon'$, as desired. 
\end{proof}

Assuming we can compute the term $\E_{D_\tau}[y_\tau]$ exactly for each round $\tau$, Algorithm \ref{alg:FTPL-No-Sampling} offers a computationally efficient method for finding an $\epsilon'$-approximate minimax strategy for the learner. However, the ability to do so depends on the geometric structure of $\cC$, and though there are natural sets $\cC$ (such as the unit hypercube and the simplex) for which obtaining a closed-form expression for the expectation is straightforward, in general this may not be the case. In these cases, we may sample from the distribution $D_\tau$ to obtain an approximate value for the expectation at round $\tau$.
We give the sampling-based Algorithm~\ref{alg:FTPL-Sampling}, along with Theorem~\ref{thm:FTPL-Sampling} that establishes its guarantees, in Appendix~\ref{app:general-minimax}. The proof of Theorem~\ref{thm:FTPL-Sampling} will be similar to the proof of Theorem \ref{thm:FTPL-No-Sampling} but with additional use of concentration inequalities in order to bound the sampling errors.

\section{Connecting Predictions and Decision Making}
\label{sec:generic-regret}
We next make connections between our ability to make unbiased predictions and the quality of decisions that are made downstream as a function of our predictions in a general setting. The form of the argument will proceed in the same way that it will in our main applications, and so is instructive. 

Specifically, we will show that a straightforward decision maker who simply best responds to our predictions can be guaranteed both \emph{no swap regret} and \emph{no type regret} (see Section~\ref{sec:prelim-decisions} for definitions) --- if we just make our predictions \emph{unbiased} conditional on the events defined by the decision maker's best-response correspondence.

\paragraph{The Predict-then-Act Paradigm} These results, whose formal statements are given below, suggest a natural design paradigm for sequential decision algorithms, which we call \emph{predict-then-act}. The idea is simple: first we make a prediction $p_t$ for an unknown payoff-relevant parameter $y_t$, and then we choose an action as if our prediction were correct --- i.e. we best respond to $p_t$. This is nothing more than implementing a straightforward decision maker for our predictions. We can parameterize the predict-then-act algorithm with various events $\cE$, such that our predictions will be unbiased with respect to events in $\cE$. Whenever $\cE$ is a collection of polynomially many events that can each be evaluated in polynomial time, the predict-then-act algorithm can be implemented in polynomial time per step. While Predict-Then-Act is quite simple, its flexibility in a variety of settings lies in the design of the event set $\cE$ and prediction space $\cC$.  By choosing the events $\cE$ to be appropriately tailored to the task at hand, we can arrange that Predict-Then-Act has guarantees of various sorts.

\begin{algorithm}[H]
\begin{algorithmic}
\FOR{$t$ in $1\ldots T$}
    \STATE Compute $\psi_t \gets$\texttt{UnbiasedPrediction($\cE, t$)}
    \STATE Predict $p_t \sim \psi_t$ \\
    \FOR{$u_i \in \cU$}
        \STATE Decision maker $i$ selects action $a_t = \delta_{u_i}(p_t) = \argmax_{a \in \cA} u_i(a,p_t)$
    \ENDFOR
    \STATE Observe outcome $y_t \in \cC$
\ENDFOR
\end{algorithmic}
\caption{\texttt{Predict-Then-Act($T,\cU,\cE, \cC, \cA$)}}
\label{alg:predict-then-act}
\end{algorithm}

\subsection{No Swap Regret}

We start by showing that predictions that are unbiased conditional on the outcome of the best response functions of downstream decision makers cause the actions of those downstream decision makers to have low swap regret. Similar observations have been previously made by \cite{perchet2011internal} and \cite{haghtalab2023calibrated} in the context of a single decision maker in the experts learning problem. In contrast, \cite{foster1999regret} showed a similar connection (also for a single decision maker in the experts learning problem) for \emph{full distributional calibration}, which conditions on \emph{every} event in a (discretization) of the prediction space. This is an exponentially large number of conditioning events in the dimension of the prediction space. The advantage of the theorem below is that it requires a number of conditioning events that is only \emph{linear} in the number of utility functions $u$ of interest and downstream actions $a$, and --- in view of our unbiasedness guarantees obtained in Section~\ref{sec:general_algorithm} gives a regret bound that scales only logarithmically in the dimension $d$ of the decision space using an algorithm with running time that is only polynomial in $d$.
\begin{theorem}
\label{thm:swap}
    Fix a collection of $L$-Lipschitz utility functions $\cU_L$, an action set $\cA = \{1,\ldots,K\}$, and any transcript $\pi_T$. Let $\cE = \{E_{u,a} : u \in \cU_L, a \in \cA\}$ be the set of binary events corresponding to straightforward decision makers with utility functions $u \in \cU_L$ taking each action $a \in \cA$. Then if $\pi_T$ is $\alpha$-unbiased with respect to $\cE$, for every $u \in \cU_L$, the straightforward decision maker with utility function $u$ has swap regret at most: 
    $$\max_{\phi:\cA\rightarrow \cA}r(\pi_T,u,\phi) \leq \frac{2L\sum_{a \in \cA}\alpha(n_T(E_{u,a},\pi_T))}{T}$$
    If $\alpha$ is concave, then this bound is at most:
     $$\max_{\phi:\cA\rightarrow \cA}r(\pi_T,u,\phi) \leq \frac{2LK\alpha(T/K)}{T}$$
\end{theorem}

The logic of the proof is simple. For every action $a$ and alternative action $b$, we consider the subsequence of rounds on which a straightforward decision maker would have chosen to play $a$ in response to our predictions. Our predictions are unbiased conditional on the event that the decision maker chooses to play $a$ --- i.e. this subsequence --- and so (on average over the subsequence), the decision maker's estimates for the payoff they received for playing $a$, as well as the payoff that \emph{they would have received} for playing $b$ are correct. Moreover, the reason they chose to play $a$ on each round in this subsequence (as a straightforward decision maker) was because pointwise, on each round, we estimated that $a$ would obtain higher payoff than $b$. Thus, it must be that on average over the subsequence, $a$ really did obtain higher payoff than $b$. Since this is true for every pair of actions, the decision maker must have had no swap regret. The proof formalizes this logic:

\begin{proof}[Proof of Theorem \ref{thm:swap}]
    Fix any $\phi:\cA\rightarrow \cA$ and any $u \in \cU_L$. We need to upper bound $r(\pi_T,u,\phi)$. Using the linearity of $u(a,\cdot)$ in its second argument for all $a \in \cA$, we can write:
    \begin{eqnarray*}
r(\pi_T,u,\phi) &=& \frac{1}{T}\sum_{t=1}^T u(\phi(\delta_u(p_t)),y_t)-u(\delta_u(p_t),y_t) \\
&=& \frac{1}{T} \sum_{a \in \cA} \sum_{t : \delta_u(p_t) = a} u(\phi(a),y_t) - u(a,y_t) \\
&=& \sum_{a \in \cA}\frac{1}{T}\sum_{t=1}^T E_{u,a}(p_t)\left(u(\phi(a),y_t) - u(a,y_t)\right) \\
&=&\sum_{a \in \cA} \left(u\left(\phi(a),\frac{1}{T}\sum_{t=1}^T E_{u,a}(p_t) y_t\right) - u\left(a,\frac{1}{T}\sum_{t=1}^T E_{u,a}(p_t) y_t\right)\right) \\
&\leq& \sum_{a \in \cA} \left(u\left(\phi(a),\frac{1}{T}\sum_{t=1}^T E_{u,a}(p_t) p_t\right) - u\left(a,\frac{1}{T}\sum_{t=1}^T E_{u,a}(p_t) p_t\right) + \frac{2L\alpha(n_T(E_{u,a},\pi_T)) }{T} \right) \\
&\leq& \sum_{a \in \cA} \left( \frac{2L\alpha(n_T(E_{u,a},\pi_T))}{T}\right) \\
&=& \frac{2L \sum_{a \in \cA}\alpha(n_T(E_{u,a},\pi_T))  }{T}
    \end{eqnarray*}
Here the first inequality follows from the $\alpha$-unbiasedness condition and the $L$-Lipschitzness of $u$: indeed, for every $a'$ (and in particular for $a' \in \{a, \phi(a)\}$) we have 
\[
\left|u\left(a',\frac{1}{T}\sum_{t=1}^T E_{u,a}(p_t) p_t\right) - u\left(a',\frac{1}{T}\sum_{t=1}^T E_{u,a}(p_t) y_t\right)\right| \leq L || \frac{1}{T} \sum_t (p_t - y_t) E_{u, a}(p_t)||_\infty \]
\[
= \frac{L}{T} \max_{i \in [d]} |\sum_t (p_{t, i} - y_{t, i}) E_{u, a}(p_t)| \leq \frac{\alpha(n_T(E_{u,a},\pi_T)) L}{T}.
\] 
The 2nd inequality follows from the fact that by definition, whenever $\delta_u(p_t) = a$ (and hence whenever $E_{u,a}(p_t) = 1$), $u(a,p_t) \geq u(a',p_t)$ for all $a' \in \cA$, and the fact that by Lemma \ref{lem:convexdelta}, the levelsets of $\delta_u$ are convex, and hence $u(a,\frac{1}{T}\sum_{t=1}^T E_{u,a}(p_t)p_t) \geq u(a',\frac{1}{T}\sum_{t=1}^T E_{u,a}(p_t)p_t) $ for all $a'$.

Recall that for any utility function $u$, the events $\{E_{u,a}\}_{a \in \cA}$ are disjoint, and so for any $u$ and any $\pi_T$, $\sum_{a \in \cA}n_T(E_{u,a},\pi_T) \leq T$. Therefore, whenever $\alpha$ is a concave function (as it is for the algorithm we give in this paper, and as it is for essentially any reasonable bound), the term $\sum_{a \in \cA}\alpha(n_T(E_{u,a},\pi_T))$ evaluates to at most $K\alpha(T/K)$.
\end{proof}

\subsection{No Type Regret}

Predictions that are unbiased conditional on the outcome of the best response functions of downstream decision makers with utility functions in $\cU$ also suffice to guarantee that downstream decision makers have no type regret to any utility function in $\cU$. Guarantees of this sort were first given by \cite{zhao2021calibrating}, in a batch setting.

\begin{restatable}{theorem}{typeregret}
\label{thm:type}
    Fix a collection of $L$-Lipschitz utility functions $\cU_L$, an action set $\cA = \{1,\ldots,K\}$, and any transcript $\pi_T$. Let $\cE = \{E_{u,a} : u \in \cU, a \in \cA\}$ be the set of binary events corresponding to straightforward decision makers with utility functions $u \in \cU_L$ taking each action $a \in \cA$. Then if $\pi_T$ is $\alpha$-unbiased with respect to $\cE$, for every $u \in \cU_L$, the straightforward decision maker with utility function $u$ has type regret with respect to $\cU_L$ at most:
    $$\max_{u' \in \cU_L}r(\pi_T,u,u') \leq \frac{L \sum_{a \in \cA}\left( \alpha(n_T(E_{u',a},\pi_T))+\alpha(n_T(E_{u,a},\pi_T))\right)}{T} $$
    For any concave $\alpha$, this is at most: 
        $$\max_{u' \in \cU_L}r(\pi_T,u,u') \leq \frac{2LK\alpha(T/K) }{T} $$

\end{restatable}

The proof of Theorem \ref{thm:type} can be found in Appendix \ref{app:genericregret}.

\section{Application: Groupwise Swap Regret, and Subsequence Regret for Online Combinatorial Optimization and Extensive-Form Games}
\label{sec: applications}
We now show how to use the machinery that we have developed to give algorithms in a variety of sequential settings. In all cases, the scenario we analyze is that in rounds $t$, a predictor makes predictions $p_t$, after which one or more decision makers choose actions that are best responses to $p_t$ (i.e. they function as \emph{straightforward} decision makers). When we are designing an algorithm for a single decision maker, we always use the \emph{predict-then-act} paradigm (Algorithm \ref{alg:predict-then-act}). When we are designing a coordination mechanism, we imagine that our predictions are made to a variety of decision makers, who each independently act. We show how guaranteeing that our predictions are unbiased subject to appropriately chosen events gives desirable guarantees for the decision makers of various sorts.

\subsection{Warm-up: (Groupwise) Swap Regret for the Experts Problem}
\label{sec:warmup}
As the simplest example of how to use the machinery we have developed, we show how to derive an algorithm that recovers optimal swap regret bounds for the experts problem (defined in Section \ref{sec:experts}). Recall that in the experts problem, there are $n$ ``experts'' $i \in [n]$, and at each round, the algorithm must choose an expert $i_t$ (or distribution over experts $q_t$). A vector of gains $g_t \in [-1,1]^n$ is realized, and the algorithm obtains reward $g_{t,i}$ if she chose expert $i_t = i$ (and expected reward $\langle q_t, g_t \rangle$ if she chose a distribution over experts $q_t$).  We can derive an algorithm for this problem using our framework by at every round $t$:
\begin{enumerate}
    \item Making a prediction $p_t$ for the vector of gains $g_t$ that will be realized at the end of the round, and
    \item Selecting the expert $i_t$ that has the largest predicted gain: $i_t = \argmax_{i \in [n]} p_{t,i}$.
\end{enumerate}
To cast this in our framework, we set the prediction space $\cC = [-1,1]^n$ to be the set of all feasible gain vectors, the action space $\cA = \{1,\ldots,n\}$ to be the set of experts, and define a  utility function $u:\cA\times \cC\rightarrow \mathbb{R}$ as $u(i,p) = p_i$. This utility function is indeed linear and $L$-Lipschitz in its second argument for $L = 1$. We instantiate our prediction algorithm to produce unbiased predictions with the set of $n$ disjoint binary events $\cE = \{E_{u,a}\}_{a \in \cA}$ (in this case, for each $i \in [n]$, the event $E_{u,i}$ is defined such that $E_{u,i}(p_t) = 1$ if and only if $i = \argmax_{i' \in [n]} p_{t,i}$). We can now read off swap regret bounds from the machinery we have built up:

\begin{theorem}
    The  Predict-Then-Act Algorithm  parameterized with $\cC = [-1,1]^n$ and $\cE = \{E_{u,a}\}_{a \in \cA}$ obtains at each round $t \leq T$ expected swap regret at most:
    $$\E_{\pi_t}\left[\max_{\phi \in \Phi_{\textrm{Swap}}}r(\pi_t,u,\phi)\right] \leq O\left(\sqrt{\frac{n\ln(nT)}{t}} \right)$$
We can compute the approximate minimax equilibrium using Algorithm \ref{alg:minmax-LP}, with per-round running time polynomial in $n$ and $\log t$ for each round $t \in [T]$.
\end{theorem}
\begin{proof}
From Theorem \ref{thm:main-guarantee}, the predictions of the canonical algorithm with event set $\cE$ have expected bias bounded by:
    $$\alpha(T, n_t(E,\pi_t)) = O\left(\ln(nT)+\sqrt{\ln(nT) \cdot  n_t(E,\pi_t)}\right)$$
    where here we use that $\sum_{t'=1}^t E(x_{t'},p_{t'})^2 \leq n_t(E,\pi_t) \leq t$. 
Plugging this bound into Theorem \ref{thm:swap}, we get that the canonical algorithm has swap regret bounded by 
$$\E_{\pi_t}\left[\max_{\phi \in \Phi_{\textrm{Swap}}}r(\pi_t,u,\phi)\right] \leq O\left(\sqrt{\frac{n\ln(nT)}{t}} \right)$$

The running time guarantee follows from Theorem \ref{thm:ellipsoid}, noting that the events $\cE$ are disjoint and binary.
\end{proof}
This recovers (up to low order terms) the optimal swap regret bounds for the experts problem \cite{blum2007external,ito2020tight}, and does so in a nearly ``anytime'' manner (i.e. with bounds depending on $t$ at every time step $t \leq T$). 

Of course, more direct algorithms for obtaining these swap regret bounds are already known---but a strength of our approach is that it can be easily and transparently adapted to various extensions of the setting. Most directly, we can simultaneously ask for unbiasedness with respect to the best response correspondence of multiple utility functions, which allows a single set of predictions to enable multiple decision makers to simultaneously obtain low swap regret.  It also allows for many generalizations for a single decision maker. We give an example below in Section \ref{sec:groupwise} --- obtaining no \emph{group-wise} swap regret in a setting with contexts, There are a variety of other extensions that are possible as well that we leave as simple exercises (e.g. giving swap regret in the sleeping experts setting, offering group-wise sleeping experts bounds with \emph{adaptive} regret guarantees in the style of \cite{hazan2009efficient}, or giving the kinds of \emph{transductive} regret bounds of \cite{mohri2017online} that compete with policies that choose actions as a function of the realized transcript (groupwise and in the sleeping experts setting if desired)).

\subsubsection{Group-Wise Swap Regret}
\label{sec:groupwise}
Suppose that in each round we need to choose an expert $i_t$ on behalf of an individual with observable features $x_t \in \cX$ that we learn at the beginning of the round (i.e. before we need to make the decision). Individuals may be members of different \emph{groups} $G \subseteq \cX$, and we may have a collection of (potentially intersecting) groups $\cG$ on which we care about the performance of our algorithm. The groups could represent (e.g.) demographic groups defined by features like race, sex, income, etc, in a context in which one is interested in fairness (as in the literature on subgroup fairness, e.g.  \cite{kearns2018preventing,hebert2018multicalibration,kearns2019empirical}), or could represent subsets of the population that are believed to be clinically relevant (e.g. defined by medical history, genotype, etc) in a personalized medicine setting, or anything else.

\cite{blum2020advancing} and \cite{lee2022online} gave algorithms for obtaining group-wise \emph{external} regret; here we show how to obtain the  stronger guarantee of group-wise swap regret, which we can define as follows:
\begin{definition}[Group-wise $\Phi$ and Swap-Regret]
    Fix a transcript $\pi_T$ and a group $G \subseteq \cX$. Let $T_G(\pi_T) = |\{t \leq T : x_t \in G\}|$ denote the number of rounds in which $x_t \in G$. The regret that a straightforward decision maker with utility function $u$ has with respect to a strategy modification rule $\phi:\cA\rightarrow \cA$ on group $G$ is:
    $$r(\pi_T,u,\phi,G) = \frac{1}{T_G(\pi_T)}\sum_{t : x_t \in G} u(\phi(a_t),y_t) - u(a_t,y_t),$$
where $a_t = \delta_u(p_t)$ for each $t$. Let $\Phi$ be a collection of strategy modification rules and let $\cG$ be a collection of groups. We say that a decision maker has $(\Phi,\cG)$-groupwise regret $\alpha$ for a function $\alpha:\mathbb{R}\rightarrow \mathbb{R}$ if for every $G \in \cG$ and for every $\phi \in \Phi$, $r(\pi_T,u,\phi,G) \leq \alpha(T_G(\pi_T))$. We say that the decision maker has $\cG$-groupwise swap regret $\alpha$ if they have $(\Phi_{\textrm{Swap}},\cG)$-groupwise regret $\alpha$ for the set $\Phi_{\textrm{Swap}}$ of all strategy modification rules.
\end{definition}
To solve this variant of the experts problem, we continue to have prediction space $\cC = [-1,1]^n$, $\cA = \{1,\ldots,n\}$, and $1$-Lipschitz utility function $u:\cA\times \cC\rightarrow \mathbb{R}$ defined as $u(i,p) = p_i$. The only change is that we now instantiate our prediction algorithm to produce unbiased predictions with respect to the $|\cG|\cdot n$ events $\cE = \{E_{u,a,G}\}_{a \in \cA, G \in \cG}$ defined such that for each action $i \in \cA$ and $G \in \cG$, $E_{u,i,G}(x_t,p_t) = 1$ if and only if $x_t \in G$ and $i = \argmax_{i'\in [n]}p_{t,i}$. We can now apply the following straightforward modification of Theorem \ref{thm:swap}, adapted to groupwise swap regret:
\begin{restatable}{theorem}{groupswap}
\label{thm:group-swap}
    Fix a collection of $L$-Lipschitz utility functions $\cU_L$, an action set $\cA = \{1,\ldots,K\}$, a collection of groups $\cG \in 2^{\cX}$, and any transcript $\pi_T$. Let $\cE = \{E_{u,a,G} : u \in \cU_L, a \in \cA, G \in \cG\}$ be the set of binary events defined as $E_{u,a,G}(x_t,p_t) = 1$ if and only if $x_t \in G$ and $\delta_u(p_t) = a$. Then if $\pi_T$ is $\alpha$-unbiased with respect to $\cE$, for every $u \in \cU_L$, the straightforward decision maker with utility function $u$ has $\cG$-groupwise swap regret at most: 
    $$\max_{\phi:\cA\rightarrow \cA}r(\pi_T,u,\phi,G) \leq \frac{2L\sum_{a \in \cA}\alpha(n_T(E_{u,a,G},\pi_T))}{T_G(\pi_T)}$$
    for every group $G \in \cG$.
    If $\alpha$ is concave, then this bound is at most:
     $$\max_{\phi:\cA\rightarrow \cA}r(\pi_T,u,\phi,G) \leq \frac{2LK\alpha(T_G(\pi_T)/K)}{T_G(\pi_T)}$$
\end{restatable}
The proof requires only additional notation compared to the proof of Theorem \ref{thm:swap}, and can be found in the Appendix. We can now read off group-wise swap-regret bounds for the machinery we have built up to get the following bound on anytime groupwise swap regret:
\begin{theorem} Fix any collection of groups $\cG \subset 2^\cX$. 
    The canonical Predict-Then-Act Algorithm parameterized with $\cC = [-1,1]^n$ and $\cE = \{E_{u,a,G} :, a \in \cA, G \in \cG\}$ obtains for every $t \leq T$ group-wise swap regret at most:
$$\E_{\pi_t}\left[\max_{\phi \in \Phi_{\textrm{Swap}}}r(\pi_t,u,\phi,G) - \Theta\left(\sqrt{\frac{n\ln(n|\cG|T)}{t_G(\pi_t)}} \right)\right] \leq 0$$   We can compute the minimax equilibrium using Algorithm \ref{alg:FTPL-No-Sampling}, with per-round running time polynomial in $n$ and $t$ on every round $t \in [T]$.
\end{theorem}
\begin{proof}
From Theorem \ref{thm:main-guarantee}, the predictions of the canonical algorithm with event set $\cE$ have expected bias bounded at round $t$ by:
    $$\alpha(T, n_t(E_{u,a,G},\pi_t)) = O\left(\ln(n|\cG|T)+\sqrt{\ln(n|\cG|T) \cdot n_t(E_{u,a,G},\pi_t)}\right)$$
Plugging this bound in to Theorem \ref{thm:group-swap} we get that the canonical algorithm has $\cG$-groupwise swap regret bounded by 
$$\E_{\pi_t}\left[\max_{\phi \in \Phi_{\textrm{Swap}}}r(\pi_t,u,\phi,G) - \Theta\left(\sqrt{\frac{n\ln(n|\cG|T)}{t_G(\pi_t)}} \right)\right] \leq 0$$   The events that we condition on are no longer disjoint, and so we cannot use Algorithm \ref{alg:minmax-LP} to make our predictions with running time polylogarithmic in $1/\epsilon$, but we can still use Algorithm \ref{alg:FTPL-No-Sampling} to obtain running time polynomial in $1/\epsilon$.
\end{proof}

\subsection{Subsequence Regret in Online Combinatorial Optimization and Extensive Form Games}
\label{sec: subsequence}
In Section \ref{sec:generic-regret} we gave two qualitatively different regret guarantees that follow from making predictions that are unbiased with respect to $\cE = \{E_{u,a} : u \in \cU, a \in \cA\}$, the set of binary events corresponding to straightforward decision makers with utility functions $u \in \cU$ taking each action $a \in \cA$. This is a natural collection of events to condition on 
 in settings in which the action set $\cA$ is modestly sized --- but when $\cA$ is exponentially large, then in general it is not possible to offer non-trivial guarantees over this collection of events without exponentially large data requirements. 
 
 In this section, we show that in the online combinatorial optimization problem, despite the fact that $\cA$ is exponentially large (in the number of base actions), it is nevertheless possible to efficiently obtain \emph{subsequence} regret with respect to the subsequences defined by any polynomially large collection of events $\cE$. Recall that subsequence regret corresponds to no \emph{external} regret on each subsequence --- i.e. no regret to each of the \emph{exponentially} many actions in $\cA$, simultaneously on each subsequence $E \in \cE$.

The online combinatorial optimization problem is an instance of the online linear optimization problem (defined in Section~\ref{sec:general-minimax} in the context of deriving our main algorithm via FTPL) defined by a set of base actions $B = \{1,\ldots,n\}$ and a collection of \emph{feasible subsets} of the base actions $\cD \subseteq 2^B$. In rounds $t \in \{1,\ldots,T\}$:
\begin{enumerate}
    \item The algorithm chooses a distribution over feasible action subsets $D_t \in \Delta \cD$
    \item The adversary chooses a vector of gains $g_t \in [-1,1]^n$ over the base actions. 
    \item The algorithm experiences gain $\hat g_t = \E_{S_t \sim D_t}[\sum_{i \in S_t} g_{t,i}]$.
\end{enumerate}

For example, if the base actions $B$ correspond to the roads in a road network, the feasible subsets $\cD$ correspond to collections of roads that form $s-t$ paths in the underlying network, and the gains correspond to the (negative) road congestions for each edge in the network, then we have the online shortest paths problem \cite{takimoto2003path,kalai2005efficient}. More generally, $\cD$ could represent \emph{any} combinatorial structure, such as spanning trees, Hamiltonian paths, or anything else. The Follow-the-Perturbed-Leader algorithm \cite{kalai2005efficient} reduces the problem of obtaining efficient \emph{external} regret bounds for combinatorial optimization problems to the offline problem of solving linear optimization problems over $\cD$. Here we show how to efficiently get the  stronger guarantee of no subsequence regret (defined in Section \ref{sec:prelim-decisions}) for any polynomial collection of events $\cE$ by reduction to the offline problem of optimizing over $\cD$.

To cast online combinatorial optimization in our framework, we set our prediction space to be $\cC = [-1,1]^{n}$, our action space to be  $\cA = \cD$, and the decision maker's utility function to be $u(S_t, p_t) = \sum_{i \in S_t} p_{t,i}$, which is linear in $p_{t}$ as required. 

Recall that the events $E \in \cE$ can depend on the chosen set $S_t$ that the learner plays at each round. If the payoffs of different \emph{compound actions} $S_t$ were unrelated to one another, then we would have to resort to conditioning on events $E_{u,S}$ as in Section \ref{sec:warmup} --- i.e. the events that correspond to a downstream decision maker with utility function $u$ playing action $S$. The difficulty is that in the online combinatorial optimization problem, there are exponentially many actions $S$, and hence exponentially many such events. Here we take advantage of the linear structure. The idea is that we can condition on events defined by \emph{base actions} $b \in B$. In particular, we condition on the events that the downstream decision maker chooses an action $S$ that \emph{contains} base action $b$. Although there are as many as $2^n$ compound actions $S$, there are only $n$ base actions, and hence $n$ such conditioning events. The linear structure of the payoff makes these events sufficient to guarantee the learner no \emph{external} regret. If we now want the learner to additionally have a no \emph{subsequence} regret guarantee with respect to a collection of arbitrary events $\cE$, we instead condition on the intersection of the events $E \in \cE$ with the events that the downstream learner plays each of their base actions $b$, which results in only $n\cdot|\cE|$ conditioning events. This is enough to give us efficient algorithms promising no subsequence regret on any polynomially sized collection of subsequences:
\begin{restatable}{theorem}{subsequence}
\label{thm: subsequence}
    Fix a base action set $B = \{1,2,\ldots n\}$, prediction space $\cC = [-1,1]^n$, an action set $\cD \subseteq 2^B$, a collection of events $\cE$, and utility function $u(S_t, p_t) = \sum_{i \in S_t} p_{t,i}$. Let $\cI = \{I_{b,E}:b \in B, E \in \cE\} $ be the set of binary events defined as $I_{b,E}(p_t)=1$ if and only if $b \in \delta_u(p_t)$ and $E(\pi_{t-1},x_t,p_t) = 1$. Then if $\pi_T$ is $\alpha$-unbiased with respect to $\cI \cup \cE$, the straightforward decision maker with utility function $u$ has $\cE$-subsequence regret at most:
    \[ \max_{E \in \cE, \phi \in \Phi_{Ext}} r(\pi_T,u,\phi, E) \le  \sum_{b \in B} \frac{\alpha(n_T(E, \pi_T)) + \alpha(n_T(I_{b, E}, \pi_T))}{T}\]
\end{restatable}

The proof of Theorem \ref{thm: subsequence} is in Appendix \ref{app: subsequence}. We can read off concrete anytime subsequence regret bounds using the machinery we have built up:
\begin{restatable}{corollary}{subsequencecorr}
    Fix any collection of events $\cE$. The canonical Predict-Then-Act Algorithm parameterized with $\cC = [-1,1]^n, \cA = \cD, \cE' = \cI \cup \cE$, where $\cI = \{I_{b,E}:b \in B, E \in \cE\}$, obtains expected subsequence regret at each round $t \leq T$ at most:
    \[ \E_{\pi_t}\left[\max_{E \in \cE, \phi \in \Phi_{\textrm{Ext}}}r(\pi_t,u,\phi,E) \right] \leq O \left(\frac{n\sqrt{\ln(n|\cE|T)}}{\sqrt{t}} \right) \]We can implement the Predict-Then-Act algorithm using Algorithm \ref{alg:FTPL-No-Sampling}, with runtime in every round $t \in [T]$ polynomial in $n$ and $t$.
\end{restatable}

\begin{remark}
    We have described the result as if there is a single decision maker. However we can equally well handle the case in which there are many different decision makers $j$ who experience affine gain $g_{t,i}^j(y^t)$ for base action $i$ (as a function of some common state $y$), and who experience total gain $\hat g_t^j(S_t) = \sum_{i \in S}g_{t,i}^j(y^t)$ for compound action $S_t$. For example, in the context of an online shortest paths problem, $y_t$ could be a vector of congestions on each road segment, but different downstream agents could have different disutilities for delay, could have preferences for or against toll roads, scenic routes, etc, and could have different action sets corresponding to different source-destination pairs in the road network. If we have $m$ such downstream agents, we can produce forecasts that are unbiased according to the events defined in Theorem \ref{thm: subsequence} for \emph{each} of their utility functions, which will have the effect of increasing the running time of the algorithm linearly with $m$, and increasing the regret guarantee logarithmically in $m$. The result will be forecasts that yield regret guarantees simultaneously for all of the $m$ downstream agents. 
\end{remark}

\subsubsection{Regret in Extensive-Form Games}
Extensive-form games are a generalization of normal form games in which players may take actions in sequence, have multiple interactions with one another, and reveal some (but possibly not all) information about their actions to one another over the course of their interaction. Many real-world interactions are fruitfully modeled as extensive form games --- for example, superhuman poker players were developed by modelling poker as a large extensive form game \cite{bowling2015heads,brown2019superhuman}.

An extensive form game can be  represented by a game tree, where at each node of the tree, a particular agent plays an action, until a terminal node is reached and utilities for all players are revealed. Players are able to condition on past play in the game, as specified by the \emph{information set}, or a set of nodes for which the revealed past play is the same, capturing the notion of imperfect information.

Recent work in extensive-form game solving has explored the connection between various no-regret dynamics and equilibrium concepts, in attempts to find the largest set of strategy modification functions, or deviations, such that polynomial time convergence to equilibria is still possible. \cite{farina2022simple, farina2023polynomial} find that minimizing regret defined by a set of trigger deviations and linear swap deviations on sequence-form strategies converge to extensive-form correlated equilibrium (EFCE) and linear correlated equilibrium (LCE) respectively. On the other hand, \cite{morrill2021hindsight, morrill2021efficient} categorize families of nonlinear (with respect to sequence-form strategies) behavioral deviations which converge to equilibria concepts that are subsets or supersets of EFCE.

In this section, we present regret minimization in the extensive-form game setting as a special case of online combinatorial optimization.
\paragraph{Extensive-form Games}
An extensive-form game of $n$ players $\{1,2,\ldots, n\} \cup \{c\}$ is represented by a tree, which includes a set of nodes $\cH$ and a set of directed edges, or actions $\cA$. There is a chance player $c$, representing Nature, who at each of their nodes plays actions with fixed probability determined when the game begins. There are two types of nodes: a \emph{terminal node} $z \in \cH$ is a leaf of the tree, while all other non-terminal nodes are referred to as  \emph{internal nodes}. Let $\cZ$ be the set of all terminal nodes in the game. An \emph{information set} $I \in \cI$ is a set of internal nodes such that the information available to the player at these nodes is indistinguishable, i.e. the known history of past play is the same. The \textit{partition function} $H: \cH \to \cI$ defines this partition of $\cH$ into information sets. At each information set $I\in \cI$, a single player $i \in [n]$ selects an action from the available actions at that information set, denoted $\cA(I)$ (a basic consistency condition requires that the actions available at each node of an information set all be the same, and equal to $\cA(I)$).  Let $\cI_i$ be the set of all information sets at which player $i$ chooses an action. Once a terminal node $z \in \cZ$ of the tree is reached, the game ends and the vector of gains $\{g_1(z),g_2(z),\ldots g_n(z)\}$ for all players is revealed.

An \textit{assignment function} $\rho: \cI \to [n] \cup \{c\}$ takes as input an information set $I \in \cI$ and returns the player who takes an action at that information set. We overload $\rho$ to also return the acting player for single nodes in $\cH$. Note that for any two nodes $h_1$ and $h_2$ in the same information set, $\rho(h_1) = \rho(h_2)$. The \textit{successor function} $c: \cH \times \cA \to \cH$ maps any node-action pair $(h,a)$ to the node resulting from taking action $a$ at $h$. For any two information sets $I,I' \in \cI_i$, we write $I \prec I'$ if there exists a path from any node $h \in I$ to any node $h' \in I'$.


\paragraph{Strategy Representations and Reachability}
A standard extensive-form strategy representation is the \emph{behavioral strategy}, where at each one of the player's information sets, a probability is assigned to each of the actions that the player may take. 
\begin{definition}
    A \emph{behavioral strategy} $\sigma_i: \cI \to \Delta(\cA(\cI))$ for player $i$ assigns a valid probability distribution over actions at each information set $I \in \cI_i$. Let $\sigma_{i}(I,a)$ be the probability action $a$ is played in the distribution over information set $I$. A deterministic behavioral strategy, denoted $s_i$, assigns probability 1 to some action at each information set $I\in \cI$. Let $S_i$ be the set of all deterministic behavioral strategies of player $i$, and $\Sigma_i$ be the set of all (randomized) behavioral strategies of player $i$.
\end{definition}
One disadvantage of the behavioral strategy representation is that it results in nonconvexity in the computation of important quantities, such as the probability of reaching any terminal node. In our framework, we use a terminal node-based representation of deterministic strategies, representing each strategy by the subset of terminal nodes it makes reachable, which is a sufficient statistic to compute the payoff of that strategy. This representation relies on the concept of \textit{reachability}. Reachability can be reasoned about as a binary-valued property (whether or not there is a chance a node is reached) or as a probability (the exact probability of reaching a node). We discuss the former in the context of a learner's strategy $\sigma_i$, and the latter with respect to a collection of opponents' strategies $\sigma_{-i}$: 

\begin{definition}
    A node $h \in \cH$ (or information set $I \in \cI$) is made \emph{reachable by player $i$ under behavioral strategy $\sigma_i$} if, for a path $P_h=\{(I_1,a_1), (I_2, a_2), \ldots, (I_K, a_K)\}$ from the root to $h$ (or $I$), for all $k \in [K]$ such that $\rho(I_k)=i$, $\sigma_i(I_k, a_k)>0$.
\end{definition}
\begin{definition}
    The \emph{reach probability} $r_h$ (or $r_I$) of a node $h \in \cH$ (or information set $I \in \cI$) under opponents' strategies $\sigma_{-i}$ is
    \[ r_h = \prod_{k \in [K], \rho(I_k) \neq i} \sigma_{\rho(I_k)}(I_k,a_k)\]
    given a path $P_h = \{(I_1,a_1), (I_2, a_2), \ldots, (I_K, a_K)\}$ from the root to $h$ (or $I$).
\end{definition}
With these concepts, we can define our notion of a ``leaf-form strategy'', which represents any behavioral strategy $\sigma_i$ by an indicator vector specifying which subset of leaf nodes are made reachable by $\sigma_i$. For the purposes of this work, it is sufficient to restrict our attention to the collection of pure leaf-node strategies induced by deterministic behavioral strategies.

\begin{definition}
    A \emph{leaf-form strategy representation} of a deterministic behavioral strategy $s_i$ is a binary vector $\pi \in \{0,1\}^{|\cZ|}$ indexed by all terminal nodes $z \in \cZ$, with $\pi_{z}=1$ if $z$ is reachable under $s_i$. Let $\Pi_i$ be the space of all leaf-form strategies of player $i$ induced by the set of all deterministic behavioral strategies $S_i$.
\end{definition}

We can similarly define a representation of opponents' strategies in terms of the probabilities of reaching each of the terminal nodes, weighted by the learner's payoff at each of them:
\begin{definition}
    Fix a collection of opponents' strategies $\sigma_{-i}$. The \emph{payoff-weighted reachability vector} $v$ induced by $\sigma_{-i}$ is a vector $v \in \R^{|\cZ|}$ indexed by terminal nodes $z \in \cZ$ with $v_z = r_z \cdot g_i(z)$, where $g_i(z)$ is the payoff for player $i$ at $z$. 
\end{definition}
For a strategy profile $\sigma$ where player $i$ is playing a deterministic strategy, if a node is made reachable under $\sigma_i$, then the node's reach probability under $\sigma_{-i}$ is the exact probability of reaching it. Therefore, for any such strategy profile, we can compute the expected payoff for the learner (playing a deterministic strategy) as the inner product of a corresponding leaf-form strategy and payoff-weighted reachability vector:

\begin{lemma}
    Fix a learner's deterministic behavioral strategy $s_i$ and opponent's strategies $\sigma_{-i}$. Let $\pi$ be the leaf-form strategy representing $s_i$. Let $v$ be the \emph{payoff-weighted reachability vector} induced by $\sigma_{-i}$. Then, the expected utility for the learner is $\langle \pi, v \rangle$.
\end{lemma}

\paragraph{Oracle-Efficient Optimization over Leaf-form Strategies} The learning algorithms we present will be efficient reductions to the problem of solving linear optimization problems over the space of leaf-form strategies. In the remaining exposition, we will assume that we have an oracle for solving this best response problem. Whenever an extensive form game satisfies the properties of \textit{perfect recall} and \textit{path recall} for the learner, a best-response oracle can be efficiently implemented using backwards induction --- we give details in Appendix \ref{subsec:best-response}.

\begin{definition}
    A \emph{best response oracle} for a game $G$ is an algorithm that, when given as input game $G$, player $i$, and vector of  values $v \in \R^{|\cZ|}$, returns $\pi^* = \argmax_{\pi \in \Pi_i} \langle \pi, v \rangle$.
\end{definition}

\paragraph{Extensive-form Games as Online Combinatorial Optimization}
We can now cast learning in extensive-form games as an instance of online combinatorial optimization by viewing each player as a learner operating with the leaf-form representation of their strategies, maximizing their expected payoff against an adversary, who chooses a payoff-weighted reachability vector at each round. In the setup of online combinatorial optimization, for rounds $t \in \{1,2,\ldots, T\}$:
\begin{enumerate}
    \item The learner (representing player $i$) picks a leaf-form strategy $\pi_{t} \in \Pi_i$.
    \item The adversary picks a payoff-weighted reachability vector $v_t$ fixing any $\sigma_{-i} \in \Sigma_{-i}$.
    \item The learner experiences a gain of $g_{t}=\langle \pi_t, v_t \rangle$.
\end{enumerate}
The algorithm for achieving subsequence regret for the learner in this framework proceeds by at each round $t$ first making a prediction $p_t \in \R^{|\cZ|}$ of the opponents' payoff-weighted reachability vector for that round. The straightforward decision maker then picks a leaf-form strategy $\pi_{t}$ that maximizes their expected payoff with respect to $p_t$, using the best response oracle, after which the true payoff-weighted reachability vector $v_{t}$ is revealed. We can obtain $\cE$-subsequence regret bounds of the form described in Theorem \ref{thm: subsequence}, where the action space is $\cA = \Pi_i$, the set of all leaf-form strategies, using base actions $B$, the set of all one-hot vectors of length $|\cZ|$.  


As written, the running time of our algorithm is linear in the number of terminal nodes in the game tree. This comes from the fact that we have represented the payoff of strategies in the extensive form game as linear functions over the payoff obtainable at each terminal node. More generally, however, if payoffs can be described as linear functions over a lower dimensional space, then the same algorithm can be run with running time dependence on this lower dimension. In Appendix \ref{subsec:efg-efficiency} we describe conditions under which the running time can be taken to depend only on the number of information sets of a player, rather than the number of terminal nodes in the game tree, which can sometimes be a large improvement.

\paragraph{Informed Causal Deviations} Choosing different sets of events $\cE$ will define different classes of strategy deviations. Our guarantees will allow us to state that we can obtain no-regret over the set of all deviations that map the periods for which an event $E$ has occurred ($E(x_t,p_t)=1$) to any another fixed strategy. In the following, we show how to pick a small collection of events $\cE$ to recover the well-studied notion of regret to \emph{informed causal deviations}, which lead to convergence to extensive-form correlated equilibrium \cite{von2008extensive}.

We show that regret to informed causal deviations \cite{gordon2008no, dudik2012sampling} is bounded as a special case of regret to  subsequence deviations defined by a polynomial number of events. Informed causal deviations of a strategy allow a player to consider, in hindsight, the best strategy they could have taken given that they reach a particular information set (among their own information sets) and take a particular action at that information set, known as the trigger sequence. A formal definition is given below:

\begin{definition}
    Fix an information set $I' \in \cI_i$, subsequent action $a' \in \cA(I')$, and strategy $s'_i$. An \emph{informed causal deviation} $\phi$ of strategy $s_i \in S_i$ returns a strategy such that at each information set $I \in \cI$,
    \[
    \phi(s_i) = \begin{cases}
        s'_i(I), & I \succeq I', s_i(I')=a' \\
        s_i(I), & \text{otherwise}\\
        \end{cases}
  \]
    Let $\Phi_{causal} = \{\phi_{I,a,s} : I \in \cI_i, a\in \cA(I), s' \in S_i \}$ be the set of all informed causal deviations. 
\end{definition}

We set up our algorithm to make predictions that are unbiased with respect to the $|\cI| \cdot |\max_{I \in \cI}\cA(I)|$ events defined as $\cE= \{E_{I,a}:  I \in \cI_i, a\in \cA(I)\}$. The event $E_{I,a}(p_t)=1$ if and only if the strategy played at time $t$ makes the information set $I$ reachable and plays $a$ at information set $I$. 
We now directly apply Theorem \ref{thm: subsequence} to obtain the following theorem:

\begin{restatable}{theorem}{efg-subsequence}
\label{thm: efg-subsequence}
    Fix an extensive-form game $G$. Let the base action set $B$ be the set of all one-hot vectors $b_z$ with dimension $|\cZ|$. Fix the action space of the learner as the leaf-form strategy space $\Pi_i \subseteq 2^B$. Fix events defined as $\cE= \{E_{I,a}: I \in \cI_i, a\in \cA(I)\}$, and fix any transcript $\pi_T$. Let $Q = \{q_{b_z,E_{I,a}}:b_z \in B, E_{I,a} \in \cE\}$ be the set of binary events defined as $q_{b_z,E_{I,a}}(p_t)=1$ if and only if $b_z \in \delta_u(p_t)$ and $E_{I,a}(p_t) = 1$. Then if $\pi_T$ is $\alpha$-unbiased with respect to $Q \cup \cE$, the straightforward decision maker has $\cE$-subsequence regret at most:
    \[ \max_{E \in \cE,\phi \in \Phi_{Ext}} r(\pi_T,u,\phi, E) \le \sum_{b_z \in B}\frac{\left(\alpha(n_T(E_{I,a}, \pi_T))+\alpha(n_T(q_{b_z,E_{I,a}}, \pi_T))\right)}{T}.\]
\end{restatable}

\begin{corollary}
    Fix an extensive-form game $G$ and collection of events $\cE= \{E_{I,a}: I \in \cI_i, a\in \cA(I)\}$. The canonical Predict-Then-Act Algorithm parameterized with $\cC = \R^{|\cZ|}, \cA = \Pi_i, \cE' = Q \cup \cE$, where $Q = \{q_{b_z,E_{I,a}}:b_z \in B, E_{I,a} \in \cE\}$, obtains expected subsequence regret at each round $t \leq T$ at most:
    \[ \E_{\pi_t}\left[\max_{E \in \cE, \phi \in \Phi_{\textrm{Ext}}}r(\pi_t,u,\phi,E) \right] \leq O \left(\frac{|\cZ|\ln(|\cZ||\cI||\max_{I \in \cI}\cA(I)|T)}{\sqrt{t}} \right). \]
\end{corollary}

From this guarantee, we can show that having no subsequence regret implies having no causal regret. 
\begin{restatable}{theorem}{causal}
\label{thm: causal}
    Fix an extensive-form game $G$ and transcript $\pi_T$. If the algorithm has $\cE$-subsequence regret with event set $\cE=\{E_{I,a}: I \in \cI_i, a\in \cA(I)\}$ bounded by:
    \[ \max_{E\in \cE, \phi \in \Phi_\textrm{Ext}} r(\pi_T, u, \phi, E) \le \alpha, \]
    then we have that causal regret, or $\Phi$-regret with respect to the set $\Phi_{\textrm{causal}}$, is bounded by:
    \[ \max_{\phi \in \Phi_{\textrm{causal}}} r(\pi_T, u, \phi) \le \alpha.\]
\end{restatable}
The proof can be found in Appendix \ref{app: efg}. An immediate consequence of this result is if all players minimize subsequence regret with respect to the events defined in Theorem \ref{thm: efg-subsequence}, the empirical frequency of play will converge to a solution subset of the set of extensive-form correlated equilibria (EFCE). 

In general, the kinds of subsequence deviations we define are incomparable to the behavioral deviation landscape defined by \cite{morrill2021efficient} or linear swap deviations as studied by  \cite{farina2023polynomial}. It appears that none of these three families of deviations subsumes any of the others.  We note that the machinery we have developed for subsequence deviations is easily able to take into account external context $x$ and past history, and so can be used to give regret conditional on features that may be relevant but are not represented within the action space of the game.

\section{Application: Score-Free Prediction Sets with Anytime Transparent Coverage}
\label{sec:conformal}

\paragraph{Set-valued multiclass prediction} Consider a multi-class prediction problem, in which examples are of the form $(x, y) \in \cX \times \cY$, where $\cX$ is some feature space and $\cY$ is some finite (but possibly large) label space with $|\cY| = k$. 
A standard task in the area of distribution-free uncertainty quantification is to train a model $S: \cX \to 2^\cY$ that, given features $x \in \cX$, outputs a \emph{prediction set} $S(x) \in 2^\cY$. The idea is that, in contrast to regular classification, here the trained model $S$ need not map every feature vector $x \in \cX$ to a single label in $\cY$: instead, it is at liberty to map any $x \in \cX$ to a set consisting of multiple labels at once, thus indicating uncertainty about the prediction. 

A central objective in this setting is to ensure that the prediction sets produced by the model \emph{cover} (i.e.\ include) the true label $y$ of a feature vector $x$ with probability at least $1-\alpha \in [0, 1]$ over the data. More formally, the \emph{coverage probability} $\Pr_{(x, y)}[y \in S(x)]$ must be \emph{no less than} (or, sometimes, approximately \emph{equal to}) a certain prespecified \emph{coverage level} $1-\alpha \in [0, 1]$. We refer to prediction sets satisfying this coverage guarantee as \emph{$(1-\alpha)$-prediction sets}. Typical choices of $\alpha$ include $0.1$ and $0.05$, but other values may be appropriate depending on the domain (e.g., safety-critical applications such as autonomous driving may require a very small $\alpha$). 

\paragraph{A naive approach to building prediction sets} For many multiclass prediction tasks, there exist readily available high-performance pretrained classifiers such as deep neural networks which, as an intermediate layer of their prediction mechanism, estimate a probability vector $\Tilde{p}(x) \in \Delta \cY$ on any input $x \in \cX$. Suppose we could take these probabilities at face value --- i.e., trust that they closely approximate the true conditional label distribution $p(y|x)$. Then, one compelling way to build $(1-\alpha)$-prediction sets would be to: 
\begin{enumerate}
    \item Sort the probabilities $\Tilde{p}(x)$ in nonincreasing order, such that each label $i \in [k]$ gets assigned a position $\sigma_\text{sorted}(i) \in [k]$ in this sorted order, and
    \item Output the prediction set $S(x)$ consisting of the top $k_x$ labels in this sorted order, where $k_x$ is the smallest index such that the sum of predicted probabilities of all chosen labels is at least $1-\alpha$. That is, \[S_\text{sorted}(x) := \left\{i \in [k] \, \Big| \, \sigma_\text{sorted}(i) \in [1, k_x], \text{ where } k_x = \min_{k_\mathrm{thr} \in \mathbb{N}} \sum_{i \in [k]: 1 \leq \sigma_\text{sorted}(i) \leq k_\mathrm{thr}} (\Tilde{p}(x))_i \geq 1-\alpha\right\}.\]
\end{enumerate}

The sorting step in this procedure guarantees that if $\Tilde{p}(x)$ in fact truly represented conditional label probabilities, this procedure would not only give the desired $1-\alpha$ coverage, but would in fact, on average, yield the \emph{smallest} prediction sets with $1-\alpha$ coverage, as measured by the number of labels included in $S(x)$. 

Unfortunately, the caveat here is that such probability scores $\Tilde{p}(x)$, even those output by pretrained models with very low cross-entropy (or other) loss, are not guaranteed to even be close to true conditional probabilities. For this reason, the above procedure will generically fail to obtain its target coverage level. 

\paragraph{Prediction sets with provable coverage guarantees} The field of \emph{conformal prediction}, which has seen substantial development in recent years (see \cite{angelopoulos2021gentle} for a gentle introduction), provides one principled way to address the shortcomings of using predicted probabilities to form prediction sets. Namely, instead of using raw predicted probabilities $\Tilde{p}(x)$, conformal prediction sets are formed with the help of a so-called \emph{non-conformity score} function $s: \cX \times \cY \to \mathbb{R}$, which projects the otherwise high-dimensional problem of coming up with a prediction set onto a one-dimensional quantile estimation problem: For a new example $x$, a prediction set is chosen by selecting all labels $y$ such that $s(x,y) \leq \tau$, for some threshold $\tau \in \mathbb{R}$ which should ideally be a $1-\alpha$ quantile of the distribution on non-conformity scores over examples $(x, y)$ in the calibration dataset. Most conformal prediction methods in both the batch and sequential setting boil down to estimating quantiles of this distribution in various ways. 

Conformal prediction, in contrast to naive set prediction methods, offers \emph{guaranteed} valid coverage: in other words, it provably forms $(1-\alpha)$-prediction sets, and is oblivious to the choice of score $s$ or the underlying data distribution. A disadvantage of this approach is that the (average) prediction set size, and hence the informativeness and usefulness of the produced prediction sets in practice, \emph{is} very sensitive to the choice of non-conformity score. For this reason, significant effort goes into designing custom non-conformity scores $s(x, y)$ for various applications, and such scores are usually not guaranteed to have any form of tight relationship with the true conditional probabilities $p(y|x)$.

\paragraph{Our contribution: Prediction sets with \emph{anytime transparent} coverage} In contrast to conformal prediction, in what follows we present an online algorithm for obtaining prediction sets with valid coverage that directly calibrates (high dimensional) raw predicted probability vectors, rather than working in a one-dimensional projection of the coverage problem. The probability forecasts we produce have a number of desirable properties:

\begin{enumerate}
    \item \emph{Transparent coverage:} Our predicted probability vectors can immediately be used by any downstream prediction set algorithm as if they were the true conditional probabilities --- in particular, the naive sorting-based prediction set algorithm discussed above, when fed our probability vectors as input, will in hindsight, on actual data, obtain exactly (up to vanishing terms) the same (marginal or conditional) coverage that it would obtain if our probabilities were the true conditional class probabilities. 
    %
    
    This transparency property can be understood as a novel \emph{interpretability}, or \emph{trustworthiness}, guarantee for prediction-set-based uncertainty quantification: 
    While existing conformal prediction methods rely on calibrating thresholds for nonconformity scores --- which are not necessarily interpretable as, or guaranteed to look like, true conditional class probabilities --- our method of producing raw predicted class probabilities bestows transparency upon any number of downstream prediction set algorithms, allowing them to produce valid-coverage sets by simply taking our probabilities at face value.
    \item \emph{Anytime coverage:} It is typical for online adversarial learning approaches to specify a fixed time horizon $T$ upfront, such that both the learning rates and the final guarantees are a function of $T$ and may only hold \emph{upon reaching} the time horizon. In contrast, on any prespecified subsequence of data points of length $t$, we are able to attain statistically optimal $O\left(\frac{1}{\sqrt{t}}\right)$ coverage error rates, and we only need a very coarse upper bound on the maximum time horizon $T_\mathrm{max}$ as (1) our method's per-round running time at any round $t$ does not depend on $T_\mathrm{max}$ (instead it depends on $t$ itself), and (2) our coverage error bound only depend \emph{logarithmically} on $T_\mathrm{max}$, i.e.\ is very insensitive to $T_\mathrm{max}$.\footnote{See also Remark~\ref{rem:anytime} on anytime coverage in Section~\ref{sec:general_algorithm}. Also, note that, while all our transparent coverage results are (almost) anytime, our Best-in-Class result, Theorem~\ref{thm:omnipred}, does require a more time-horizon-sensitive setting of a discretization parameter.} Thus, for example, we could pause the algorithm at any point $t \leq T_\mathrm{max}$, and then restart it when more data becomes available, and the transparency and coverage guarantees would continue to smoothly evolve from where we left off. A similar property holds for the one-dimensional adaptive conformal inference methods of~\cite{gibbs2022conformal, zaffran2022adaptive}, but these methods do not accommodate conditional coverage constraints (or have the  other desirable downstream properties of our method).
    \item \emph{Accommodates multiple diverse downstream tasks at once:} Our probability predictions not only help avoid the need to choose or commit to a non-conformity score, but allow us to make predictions that are \emph{simultaneously} useful for a \emph{variety} of downstream tasks. For example, we can predict probabilities that are useful for simultaneously producing prediction sets with different coverage levels (see Corollary~\ref{cor:aspiring-multiple-targets}) and/or aimed at optimizing different objectives.
    \item \emph{Achieves best-in-class prediction quality:} In our probability vector predictions, we are able to incorporate side information in the form of logits of any existing multiclass predictor: in fact, we show the stronger property that, given any finite collection $\cQ$ of such external predictors (which could, for instance, be pretrained NNs), we can increase the quality of our probability vector predictions to nearly match the performance of the logits of the best model in $\cQ$, simultaneously as measured by Brier score, cross-entropy loss, and other well-behaved proper scoring rules. 
    \item \emph{Enforces set-size-conditional validity, multigroup fairness, and other conditional coverage guarantees:} By making our predictions unbiased conditional on appropriate (modestly-sized) families of events, our probability predictions can easily guarantee various forms of conditional validity on the outputs of downstream prediction set algorithms. This includes not only strengthened multigroup fairness guarantees as compared to~\cite{bastanipractical} (in terms of the dependence on the number of demographic groups $|\cG|$), but also the highly desirable property of valid coverage conditional on the prediction set size (dubbed \emph{size-stratified coverage validity} in~\cite{angelopoulos2020uncertainty}), which to our knowledge has not been provably obtained in the adversarial conformal prediction literature until now. More generally, we can efficiently obtain coverage conditional on \emph{any} polynomially sized collection of events that are determined by any combination of external context and our own predictions. 
\end{enumerate}

We now provide (the most general form of) our algorithm: \texttt{Class-Probabilities-for-Prediction-Sets}, which is displayed below as Algorithm~\ref{alg:conformal}. We also informally summarize the statements of its diverse set of donwstream guarantees, which we just listed above. Before doing that, we give an informal definition that captures the behavior of \emph{straightforward} prediction set algorithms whose aim is to achieve valid $(1-\alpha)$-coverage.

\begin{definition*}[Prediction set algorithm]
    A \emph{prediction set algorithm} $S$ is a mapping \[S: \Pi^* \times \cX \times \Delta \cY \to [0, 1]^{|\cY|}.\] 
    
    At each round $t$, it takes in the past history, $\pi_{t-1}$, the new context $x_t$, our predicted class probability vector in this round, $p_t$, and outputs this round's prediction set $S(\pi_{t-1}, x_t, p_t)$ in the format of a vector in $[0, 1]^\cY$, where the entry corresponding to each label $y \in \cY$, $(S(\pi_{t-1}, x_t, p_T)) \in [0, 1]$, designates the probability with which the algorithm includes label $y$ into its prediction set. (If the algorithm is deterministic, then each entry $(S(\pi_{t-1}, x_t, p_T)) \in \{0, 1\}$.
\end{definition*}
This definition is reiterated, somewhat more formally, in Definition~\ref{def:p-set} below.

\begin{definition*}[$(1-\alpha)$-aspiring prediction set algorithm]
    A downstream prediction set algorithm $S$ is called \emph{$(1-\alpha)$-aspiring} if, when fed any predicted probability vector $\tilde{p}(x) \in \Delta \cY$, it always outputs a prediction set that would have coverage $1-\alpha$ if $\tilde{p}(x)$ was the true conditional probability vector.
\end{definition*}

This definition parallels Definition~\ref{def:straightforward-dm} of a straightforward decision maker and essentially adapts it to our uncertainty estimation setting: it captures prediction set algorithms that are naive, or straightforward, in that they are willing to \emph{trust} the predicted class probabilities as if they were correct. For instance, the naive greedy sorting-based algorithm described above is (roughly speaking) $\geq (1-\alpha)$-aspiring. (The precise definition is given in Definition~\ref{def:aspiring}.)

\begin{theorem*}[Guarantees for \texttt{Class-Probabilities-for-Prediction-Sets}]
    With any collection of downstream prediction set algorithms $\cS$ as input, \texttt{Class-Probabilities-for-Prediction-Sets} guarantees, on any sequence of potentially adversarially generated data, for every $S \in \cS$ that:
    \begin{itemize}
        \item The realized \emph{per-label coverage} of $S$ for each label $y \in \cY$ (the online analog of the batch probability $\Pr_{X, Y}[Y = y \, \wedge \, y \in S(X)]$) will be equal, up to a diminishing error, to the \emph{anticipated} per-label coverage for $y$ (i.e., the coverage level on $y$ that $S$ would expect to obtain, based on treating our predicted probabilities at each round as correct). See Theorem~\ref{thm:per-label};
        \item Similarly, the realized marginal coverage of $S$ (the analog of the batch probability $\Pr_{X, Y}[Y \in S(X)]$) will be equal, up to a diminishing error, to the \emph{anticipated} marginal coverage $S$ would obtain if our class probabilities had been correct. See Corollary~\ref{cor:marginal_coverage};
        \begin{itemize}
            \item For any downstream prediction set algorithm that is $(1-\alpha)$-aspiring, this implies valid $(1-\alpha)$ realized \emph{marginal} coverage; see Corollary~\ref{cor:aspiring-marginal}.
        \end{itemize}
        \item ($\text{set-size-conditional?} = Y$) Given as input any \emph{set-size function} defined as\footnote{More generally, $\mathrm{sz}$ can be any bounded mapping from $[0, 1]^{|\cY|}$ to $[0, \infty)$; in that case, we can simply discretize it.} $\mathrm{sz}: [0, 1]^{|\cY|} \to \{0\} \cup [N_\mathrm{maxsz}]$ for some maximum set size $N_\mathrm{maxsz}$, the anticipated and realized coverage will coincide, up to error terms, even conditionally on the prediction set size. See Section~\ref{sec:conditional-coverage};
        \begin{itemize}
            \item For any downstream prediction set algorithm $S$ that is $(1-\alpha)$-aspiring, this implies valid realized \emph{set-size-conditional} $(1-\alpha)$ coverage.
        \end{itemize}
        \item ($\text{multigroup-fair?} = Y$) Given as input any family $\cG$ of \emph{demographic groups} $G: \cX \to [0, 1]$ --- which for any individual with features $x$ determine if they belong ($G(x) = 1$), do not belong ($G(x) = 0$), or belong partially ($G(x) \in (0, 1)$) to a certain (often normatively defined) population group --- the anticipated and realized coverage will coincide, up to error terms, conditionally on every group $G \in \cG$. See Section~\ref{sec:conditional-coverage};
        \begin{itemize}
            \item For any downstream prediction set algorithm $S$ that is $(1-\alpha)$-aspiring, this implies valid realized $(1-\alpha)$ \emph{multigroup} coverage.
        \end{itemize}
        \item ($\text{booster?} = Y$) Generalizing  set-size-conditional and the multigroup coverage, it can be guaranteed for every family $\cB$ of conditional constraints defined by subsequence indicators $B: \Pi^* \times \cX \times \Delta \cY \to [0, 1]$ (which we call `boosters') that decide whether or not to invoke the conditional constraint given any transcript $\pi_{t-1}$, context $x_t$, and prediction $p_t$.  See Section~\ref{sec:conditional-coverage};
        \begin{itemize}
            \item For any downstream prediction set algorithm $S$ that is $(1-\alpha)$-aspiring, this implies valid realized $(1-\alpha)$ \emph{conditional} coverage with respect to the family of conditioning events given by the booster mappings.
        \end{itemize}
        \item ($\text{best-in-class?} = Y$) Given as input any collection of competing class probability predictors $\cQ$ --- where each predictor $q \in \cQ$ outputs, at each round $t$, a predicted class probability vector $q_t \in \Delta \cY$ as a function of the context $x_t$ (and possibly also as a function of past history and even of our class probability predictions in the current round) --- our predictions will be guaranteed to \emph{beat} \emph{all} predictors $q \in \cQ$ simultaneously with respect to \emph{all} $L$-Lipschitz Bregman losses, up to diminishing gap that scales with $L$. (The result also extends to \emph{locally} Lipschitz Bregman losses.) See Theorem~\ref{thm:omnipred} and Corollary~\ref{cor:log-loss}.
    \end{itemize}
    Moreover, at each round $t = 1, 2, \ldots$, the per-round runtime of \texttt{Class-Probabilities-for-Prediction-Sets} is polynomial in $t, |\cY|$, as well as in the size $|\cE^t|$ of the aggregated event collection in that round.
\end{theorem*}

\begin{algorithm}[ht]
\begin{algorithmic}
\STATE \textbf{INPUT:} Collection of downstream prediction set algorithms $\cS$, \\ 
\quad \quad \quad \quad \: set-size-conditional? $\in \{\text{Y, N}\}$, 
\quad set size function $\mathrm{sz}: [0, 1]^{|\cY|} \to \{0\} \cup [N_\mathrm{maxsz}]$, \\
\quad \quad \quad \quad \: multigroup-fair? $\in \{\text{Y, N}\}$, 
\quad \quad \,\, collection of protected groups $\cG$, \\ 
\quad \quad \quad \quad \: booster? $\in \{\text{Y, N}\}$, 
\quad \quad \quad \quad \quad \quad collection of booster predictors $\cB$, \\ 
\quad \quad \quad \quad \: best-in-class? $\in \{\text{Y, N}\}$, 
\quad \quad \quad \quad collection of competing predictors $\cQ$
\FOR{rounds $t = 1, 2, \ldots$}
    \STATE \textbf{Receive} context $x_t \in \cX$
    \STATE \textbf{Assemble} event collection $\cE^t$: 
    \STATE For every $S \in \cS$, form: \text{$\cE^t_S := \{E_{S, y}\}_{y \in \cY}$, with $E_{S, y}(p) := (S(\pi_{t-1}, x_t, p))_y$ for $p \in \Delta \cY$ }
    \STATE $\cE_\cS^t \gets \cup_{S \in \cS} \cE^t_S$
    \STATE $\cE^t \gets \cE_\cS^t$ \tcp{transparent per-label coverage}
    \IF{$\text{set-size-conditional?} = Y$}
        \STATE For all $S \in \cS$, form: $\cE_\mathrm{sz, S}^t \gets \{E_{\mathrm{sz}, S, n}\}_{0 \leq n \leq N_\mathrm{maxsz}}$, with $E_{\mathrm{sz}, S, n}(p) := 1[\mathrm{sz}(S(\pi_{t-1}, x_t, p)) = n]$ for all $p$
        \STATE $\cE^t \gets \cE^t \cup \left( \bigcup_{S \in \cS}\cE_S^t \times \cE_\mathrm{sz, S}^t \right)$
    \ENDIF
    \IF{$\text{multigroup-fair?} = Y$}
        \STATE $\cE_\cG^t \gets \{E_G\}_{G \in \cG}$, where $E_{G} := G(x_t)$ \tcp*{each $G$ stands for a context-defined demographic group}
        \STATE $\cE^t \gets \cE^t \cup \cE_\cS^t \times \cE_\cG^t$
    \ENDIF
    \IF{$\text{booster?} = Y$}
        \STATE $\cE_\cB^t \gets \{E_B\}_{B \in \cB}$, where $E_{B}(p) := B(\pi_{t-1}, x_t, p)$ for $p \in \Delta \cY$ \tcp*{boosters find miscovered regions}
        \STATE $\cE^t \gets \cE^t \cup \cE_\cS^t \times \cE_\cB^t$
    \ENDIF
    \IF{$\text{best-in-class?} = Y$}
        \STATE $\cE_{\mathrm{ls}, p}^t \gets \{E_{p, n, i}\}_{n \in \{0\} \cup [\lceil 1/\delta \rceil], i \in [k]}$, where $E_{p, n, i}(p) := 1[p_{i} \approx n \delta]$ for $p \in \Delta \cY$  \tcp*{$\delta$-level sets of our predictor; see Definition~\ref{def: ls}, and see Theorem~\ref{thm:omnipred} for the appropriate setting of $\delta$} 
        \STATE $\cE_{\mathrm{ls}, \cQ}^t \gets \{E_{q, n, i}\}_{q \in \cQ, n \in \{0\} \cup [\lceil 1/\delta \rceil], i \in [k]}$, where $E_{q, n, i}(p) := 1[(q(\pi_{t-1}, p))_i \approx n \delta]$ for $p \in \Delta \cY$ \tcp*{$\delta$-level sets of all competing predictors $q \in \cQ$}
        \STATE $\cE^t \gets \cE^t \cup \cE_{\mathrm{ls}, p}^t \times \cE_{\mathrm{ls}, \cQ}^t$
    \ENDIF
    \STATE \textbf{Predict} $P_t \gets \texttt{UnbiasedPrediction}(\cE^t, t, \pi_{t-1}, x_t)$ \tcp*{$P_t \in \Delta \Delta \cY$ is a mixture over probability vectors}
    \STATE \textbf{Sample} realized predicted probability vector $p_t \sim P_t$
    \STATE \textbf{Receive} correct label $y_t \in \cY$ from adversary
    \STATE \textbf{Update} transcript: $\pi_t \gets \pi_{t-1} \cup (x_t, p_t, y_t)$
\ENDFOR

\end{algorithmic}
\caption{\texttt{Class-Probabilities-for-Prediction-Sets}}
\label{alg:conformal}
\end{algorithm}

\subsection{Online Adversarial Multiclass Set Prediction Preliminaries}

\paragraph{Setting} We let $\cX$ be the feature space and $\cY$ be the label space with $|\cY| = k$. When it is notationally convenient for indexing purposes, we without loss of generality assume $\cY = \{1, \ldots, k\}$. Examples arrive---and predictions about them are made---in a sequential fashion, in rounds $t = 1, 2, \ldots$. Specifically, in each round $t$:

\begin{enumerate}
    \item A new example arrives, and its features $x_t \in \cX$ are revealed to the learner.
    \item The learner computes a (realized) probability prediction $p_t \in \Delta \cY$ as a function of past rounds' history and the observed features $x_t$.
    \item The (realized) prediction set $S_t \subseteq \cY$ is generated as a function of $p_t$.
    \item The realized label $y_t \in \cY$ is revealed by the adversary.
\end{enumerate}

\paragraph{Prediction set algorithms} We now define the notion of a prediction set algorithm, which is responsible for forming prediction sets $S_t$ in all rounds $t$ as a function of past history, the current context, and our (the learner's) predicted multiclass probability vector. Intuitively, this algorithm may be viewed as representing a downstream decision maker who takes in our probability predictions and builds a prediction set so as to optimize whichever utility function (e.g., a notion of average prediction set size) it has.

\begin{definition}[Prediction set algorithm] \label{def:p-set}
    A \emph{prediction set algorithm} $S$ is a mapping \[S: \Pi^* \times \cX \times \Delta \cY \to [0, 1]^{|\cY|}.\] 
    The first input to $S$ is the transcript $\pi_{t-1}$ up until the current round $t$, thus allowing the algorithm's decisions to depend on past history. 
    The second input takes in the features of the example at round $t$. The third input captures the dependence of the algorithm's decisions in round $t$ on our predicted probability vector $p_t \in \Delta  \cY$. 
\end{definition}

For any $\pi_{t-1}, x$ and $p_t$, the algorithm's output $S(\pi_{t-1},x, p_t)$ is a $|\cY|$-dimensional vector; by convention here and below, we will index it by $y \in \cY$. The interpretation of the algorithm's output is that for any $y \in \cY$, the algorithm $S$ will include label $y$ in the realized prediction set $S_t$ with probability $(S(\pi_{t-1},x, p_t))_y \in [0, 1]$ over its internal randomness. Note that the entries of $S(\pi_{t-1}, x, p_t)$ will generally not sum to one: since $S$ outputs \emph{prediction sets} rather than singular labels, the output $S_t$ is allowed to contain multiple labels at once.

As an important special case of the above definition, a \emph{deterministic} prediction set algorithm is defined as a \emph{binary} vector valued mapping \[S: \Pi^* \times \cX \times \Delta \cY \to \{0, 1\}^{|\cY|}.\] A prediction set algorithm that is not deterministic is referred to as randomized.

When convenient, for any label $y$, instead of $(S(t,x, p_t))_y$ we may alternatively write $\Pr[y \in S_t]$ when $S$ is randomized and $\textbf{1}[y \in S_t]$ when $S$ is deterministic, suppressing the dependence on $\pi_{t-1}$, $x_t$ and $p_t$. 

\paragraph{Coverage metrics}
We now define the notions of (marginal and per-label) \emph{realized} and \emph{anticipated} coverage of a prediction set algorithm on any transcript. As discussed above, a natural desideratum in this setting is to ensure that the realized coverage of the prediction set algorithm meets some prespecified target level.

\begin{definition} \label{def:realized-marginal}
    The \emph{realized (marginal) coverage} of a prediction set algorithm $S$ on a $T$-round transcript $\pi_T$ is denoted:
    \[\overline{\Pr}_T[Y \in S] := \frac{1}{T} \sum_{t = 1}^T \Pr[y_t \in S_t],\]
    where $y_t$, for $t \in [T]$, are the realized labels.
\end{definition}
Note that this is simply the \emph{empirical probability} of the realized labels belonging to the prediction sets produced by the algorithm.

\begin{definition} \label{def:anticipated-marginal}
    The \emph{anticipated (marginal) coverage} of a prediction set algorithm $S$ on a $T$-round transcript $\pi_T$ is denoted: \[\widetilde{\Pr}_T[Y \in S] := \frac{1}{T} \sum_{t = 1}^T \sum_{y \in \cY} p_{t, y} \cdot \Pr[y \in S_t],\]
    where $p_{t, y}$ are the learner's predictions. Note that the anticipated coverage does not depend on the actual realized labels $y_t$.
\end{definition}
The interpretation is simple: If our predicted probability vector $p_{t,y}$ were to be taken at face value, i.e.\ assuming that the labels $y_t$ are sampled from the distribution $p_t$, then $\widetilde{\Pr}_T[Y \in S]$ would be exactly the expected coverage that the algorithm would get after round $T$, where the expectation is taken both over the randomness of $S$ and the randomness of the labels $y_t$ being sampled from the distributions $p_t$. Note that to evaluate the algorithm's anticipated coverage $\widetilde{\Pr}_T[Y \in S]$, one needs to know our predictions $p_t$, but \emph{not} the realized labels.

\begin{definition}[Per-label realized and anticipated coverage] \label{def:per-label}
    For any label $y \in \cY$, we define the \emph{realized} and \emph{anticipated} coverage that a prediction set algorithm $S$ achieves \emph{on label $y$ alone} over a $T$-round transcript $\pi_T$, analogously to the marginal coverage metrics defined above:
    \[\overline{\Pr}_T[y \in S] := \frac{1}{T} \sum_{t = 1}^T 1[y = y_t] \cdot \Pr[y \in S_t], \quad \text{and } 
    \widetilde{\Pr}_T[y \in S] := \frac{1}{T} \sum_{t = 1}^T p_{t, y} \cdot \Pr[y \in S_t].
    \]
\end{definition}
The just defined notion of per-label coverage will serve as a key stepping stone towards showing our transparent marginal coverage guarantee for any downstream algorithm $S$. Namely, we will (in Theorem~\ref{thm:per-label}) show the stronger guarantee of \emph{transparent per-label coverage} for any $S$. For intuition, it is instructive to compare \emph{per-label} coverage, which we just introduced, to \emph{label-conditional} coverage, a common object of study in the area of conditional conformal coverage guarantees (\cite{vovk2012conditional}). In the batch setting, the former would correspond to $\Pr_{(X, Y)}[Y = y \, \wedge \, y \in S(X)]$, whereas the latter would translate to $\Pr_{(X, Y)}[Y \in S(X) | Y = y]$. Thus, per-label coverage is a joint probability (of two events coinciding: label $y$ being selected as the correct label, and it being included in the prediction set by $S$), while label-conditional coverage is the corresponding conditional probability.

We observe that for any prediction set algorithm $S$, its time-averaged (anticipated or realized) marginal coverage is simply the sum over all labels $y \in \cY$ of its label-specific (anticipated or realized) coverage values:
\begin{equation} \label{eq:per-label-to-marginal}
    \overline{\Pr}_T[Y \in S] = \sum_{y \in \cY} \overline{\Pr}_T[y \in S], \quad \text{and } \widetilde{\Pr}_T[Y \in S] = \sum_{y \in \cY} \widetilde{\Pr}_T[y \in S].
\end{equation}
This observation will later be useful in translating our per-label transparent coverage guarantees (Theorem~\ref{thm:per-label}) into marginal transparent coverage guarantees (Corollary~\ref{cor:marginal_coverage}).





\subsection{Anytime Transparent Coverage: Marginal, Per-Label, and Conditional}

In this section, we show how to apply our \texttt{UnbiasedPrediction} algorithm in order to produce predictions of per-label probabilities $p_t \in \Delta \cY$ that provide transparent and flexible coverage guarantees for any downstream algorithm $S$ that produces prediction sets $S_t := S(\pi_{t-1} \cup \{x_t\}, p_t)$ as a function of past history and our current predictions. 


Intuitively, since our predictions are to be used by $S$ to compute a prediction set, we would like to make them unbiased conditional on the collection of events, for each label $y \in \cY$, that $y$ is included by $S$ in the prediction set once $S$ sees our predicted probabilities. (This may remind the reader of the ``regret'' constructions in the prior sections, where we condition on the collection of events defined by the downstream decision maker's best-response correspondence.)

Thus, we instantiate \texttt{UnbiasedPrediction} on the region $C_\mathrm{multiclass} = \Delta \cY \subset \R^k$ with a natural collection of $k$
events: $\cE_S := \{E_{S,y}\}_{y \in \cY}$, where each event $E_{S,y}$ is defined by $E_{S,y}(\pi_{t-1}, x_t, p_t) = (S(\pi_{t-1} \cup \{x_t\}, p_t))_y$, encoding the probabilities of inclusion of each fixed label $y$ into the prediction sets $S_t$ across rounds $t$.  

For instance, when the prediction set algorithm $S$ is deterministic, all events $E_{S,y}$ are binary, and simply represent the indicator of whether or not, given the predictions $p_t$ made on that round, the label $y$ is included in the prediction set $S_t$.

In fact, as we will show in Theorem~\ref{thm:per-label} below, we can achieve transparent coverage simultaneously for multiple downstream algorithms $S$ belonging to any finite collection $\cS$. It will suffice to instantiate \texttt{UnbiasedPrediction} with $\cE_\cS = \cup_{S \in \cS} \cE_S$ the \emph{union} of event families $\cE_S$ over all algorithms $S \in \cS$. As a result, the runtime of our method will only scale linearly in $|\cS|$, while the coverage error bound will only pick up a $\log |\cS|$ dependence.



\paragraph{On the dependence of \texttt{Class-Probabilities-for-Prediction-Sets} on input prediction set algorithms $S$:} Our algorithm only makes \emph{oracle queries} to the downstream prediction set algorithm $S$ at each round $t$. It completely disregards the internals of $S$, and only needs to know (for a deterministic algorithm) which labels the algorithm includes in its prediction set as a function of the underlying probability forecast $p_t$, or (for a randomized algorithm) the probability that each label is included in its prediction set.

In this manner, \texttt{Class-Probabilities-for-Prediction-Sets} functions as a wrapper around any prediction set algorithm $S$ (or even a collection of algorithms $\cS$). Given the computational efficiency guarantees for \texttt{UnbiasedPrediction} established in Theorem~\ref{thm:main-guarantee}, the ``wrapped'' prediction algorithm will thus be computationally efficient insofar as the downstream prediction algorithm $S$ is efficient and $\cY$ is enumerable --- i.e. the running time will be polynomial in $|\cY|$ and the running time of $S$. 

\paragraph{Transparent coverage guarantees:} We now show that \texttt{Class-Probabilities-for-Prediction-Sets} guarantees, no matter the downstream prediction set algorithm(s), that the realized and anticipated coverage for every label $y \in \cY$ will coincide, up to $O(1/\sqrt{T})$ --- or better --- error terms.

\begin{theorem}[Transparent Per-Label Coverage] \label{thm:per-label}
    Fix any collection $\cS$ of prediction set algorithms, where each algorithm $S \in \cS$ uses our method's probability vector predictions $p_t \in \Delta \cY$ in every round $t$. If we instantiate \texttt{Class-Probabilities-for-Prediction-Sets} with the event collection 
    \begin{equation} \label{eq:marginal-events}
        \cE_\cS := \{E_{S, y}\}_{S \in \cS, y \in \cY}, \text{ with } E_{S, y}(\pi_{t-1}, x_t, p) := (S(\pi_{t-1} \cup \{x_t\}, p))_y \text { for $p \in \Delta \cY$},
    \end{equation}
    then in expectation over $P_1, P_2, \ldots \in \Delta \Delta \cY$, the distributions \texttt{Class-Probabilities-for-Prediction-Sets} outputs, 
    the \emph{anticipated} and \emph{realized} coverage of every label $y \in \cY$ by every prediction set algorithm $S \in \cS$ will be approximately equal at every round $T \in [T_\mathrm{max}]$ (where $T_\mathrm{max}$ is the maximum time horizon):
    \[ \E_{p_t \sim P_t, \, \forall t} \left[ \left|\overline{\Pr}_T[y \in S] - \widetilde{\Pr}_T[y \in S] \right| \right]
    \leq O\left(\frac{\ln(k |\cS| T_\mathrm{max})+ \sqrt{\ln(k |\cS| T_\mathrm{max}) \cdot \sum_{t=1}^T \E_{p_t \sim P_t} \left[\Pr[y \in S_t] \right]}}{T}\right).\]
\end{theorem}

\begin{proof}
    Let $\cC_\mathrm{multiclass} := \Delta \cY \subset \R^k$. Suppose in each round, we predict and receive realizations from the space $\cC_\mathrm{multiclass}$, where at each round $t$, the adversary picks a vector in $\{\text{standard basis vectors in } \R^k\} \subset \cC_\mathrm{multiclass}$, and is interpreted as $(1[y = y_t])_{y \in \cY}$ --- the vector that indicates which label $y \in \cY$ was realized in round $t$.

    Accordingly, we abuse notation and index the coordinates of our \emph{predicted} vectors $p_t$ by $y \in \cY$ as well. In the current multiclass setting, we interpret our chosen predicted vector $p_t$ as a distribution over the labels.
    
    We now instantiate Theorem~\ref{thm:main-guarantee} on the space $\cC_\mathrm{multiclass}$ with the collection of events $\cE_\cS$ defined as in Equation~\ref{eq:marginal-events}.
    For each $S \in \cS$ and $y \in \cY$, the event $E_{S, y}$ is defined simply as the probability in round $t$ that $S$ would include label $y$ into the prediction set, given features $x_t$ and our prediction $p_t$. In other words, we have, over the internal randomness of the set selection algorithm $S$ given the current round's inputs, that
    \[E_{S, y}(\pi_{t-1}, x_t, p_t) = \Pr[y \in S_t].\]
    
    Abusing notation and indexing the coordinates of the predicted vector $p_t$ and the realized vector $y_t$ by $y \in \cY$ rather than $i \in [k]$, we then have, for all $E_{S, y} \in \cE_\cS$, any $T$, some constant $c' > 0$, and per-round distributions $\psi_t$ over the predictions output by our unbiased prediction algorithm, that
    \[
    \E_{p_t \sim \psi_t \, \forall t}  \left|\sum_{t=1}^T E_{S, y}(x_t, p_t) \cdot \left(p_{t,y} - 1[y = y_t] \right)\right|
    \leq c' \ln(2 k^2 |\cS| T_\mathrm{max})+ c' \sqrt{\ln(2 k^2 |\cS| T_\mathrm{max}) \cdot \sum_{t\in [T]} \E_{p_t \sim \psi_t} [(E_{S, y}(x_t, p_t))^2]}.
    \]
    Here, we used that $d = |\cE_\cS| = k |\cS|$. Now, rename the distributions over predictions from $\psi_t$ into $P_t$. Dropping the squares on the events inside the square root (since our events' values are in $[0, 1]$) and observing that \[\sum_{t=1}^T E_{S, y}(t, x_t, p_t) \cdot \left(p_{t,y} \!\!-\!\! 1[y \!=\! y_t] \right) = \sum_{t=1}^T \Pr[y \in S_t] \cdot p_{t, y} - \sum_{t=1}^T \Pr[y \in S_t] \cdot 1[y \!=\! y_t] = T ( \overline{\Pr}_T[y \in S] - \widetilde{\Pr}_T[y \in S]),\] we obtain, for every prediction set algorithm $S \in \cS$ and label $y \in \cY$, that
    \[
    \E_{p_t \sim P_t \, \forall t} \left[ T \cdot \left|\overline{\Pr}_T[y \in S] - \widetilde{\Pr}_T[y \in S]\right| \right]
    \leq c' \ln(2 k^2 |\cS| T_\mathrm{max})+ c' \sqrt{\ln(2 k^2 |\cS| T_\mathrm{max}) \cdot \sum_{t\in [T]} \E_{p_t \sim P_t} [\Pr[y \in S_t]]},
    \]
    and dividing through by $T$ yields the desired claim.
\end{proof}

Recalling that marginal (realized or anticipated) coverage is just the sum of per-label (realized or anticipated) empirical coverage probabilities over all labels, and applying the triangle inequality, we thus obtain, as a direct consequence of Theorem~\ref{thm:per-label}, the corresponding \emph{marginal} coverage guarantees for the instantiation of \texttt{Class-Probabilities-for-Prediction-Sets} with the same event family $\cE_\cS$: 

\begin{corollary}[Transparent Marginal Coverage] \label{cor:marginal_coverage}
 Fix any collection $\cS$ of prediction set algorithms, where each algorithm $S \in \cS$ uses our method's probability vector predictions $p_t \in \Delta \cY$ in every round $t$. If we instantiate \texttt{Class-Probabilities-for-Prediction-Sets} with the event collection as in Equation~\ref{eq:marginal-events}, i.e.:
    \[
        \cE_\cS := \{E_{S, y}\}_{S \in \cS, y \in \cY}, \text{ with } E_{S, y}(\pi_{t-1}, x_t, p) := (S(\pi_{t-1} \cup \{x_t\}, p))_y \text { for $p \in \Delta \cY$},
    \]
    then in expectation over $P_1, P_2, \ldots \in \Delta \Delta \cY$, the distributions \texttt{Class-Probabilities-for-Prediction-Sets} outputs, 
    the \emph{anticipated} and \emph{realized} marginal coverage of every prediction set algorithm $S \in \cS$ will be approximately equal at every round $T \in [T_\mathrm{max}]$ (where $T_\mathrm{max}$ is the maximum time horizon):
    \begin{align*}
        &\E_{p_t \sim P_t, \, \forall t} \left[ \left|\overline{\Pr}_T[Y \in S] - \widetilde{\Pr}_T[Y \in S] \right| \right] \\
        &\leq O\left(\frac{|\cY| \ln(|\cY| |\cS| T_\mathrm{max})}{T} + \frac{1}{T} \sum_{y \in \cY} \sqrt{\ln(|\cY| |\cS| T_\mathrm{max}) \cdot \sum_{t\in [T]} \E_{p_t \sim P_t} \left[\Pr[y \in S_t] \right]}\right) 
        \leq \tilde{O} \left(\frac{|\cY| \ln(|\cY| |\cS| T_\mathrm{max})}{\sqrt{T}}\right).
    \end{align*}
\end{corollary}

\paragraph{Application: Interpretable Conformal-Style Valid Coverage without Nonconformity Scores}

We now spell out a sample application of \texttt{Class-Probabilities-for-Prediction-Sets} instantiated with the collection of events $\cE = \cE_\cS$ for some family of downstream prediction set algorithms $\cS$, to illustrate how our method can easily produce $(1-\alpha)$-prediction sets, thereby achieving the main desideratum of e.g.\ conformal prediction. We begin with a definition:

\begin{definition}[$(1-\alpha)$-Aspiring Prediction Set Algorithm] \label{def:aspiring}
    Fix $\alpha \in [0, 1]$. A prediction set algorithm $S: \Pi^* \times \cX \times \Delta \cY \to [0, 1]^{|\cY|}$ is called \emph{$(1-\alpha)$-aspiring} if for all $\pi \in \Pi^*. x \in \cX, p = (p_y)_{y \in \cY} \in \Delta \cY$:
    \[
    \sum_{y \in \cY} p_y \cdot (S(\pi, x, p))_y = 1 - \alpha.
    \]
\end{definition}
Importantly, the property of an algorithm $S$ being $(1-\alpha)$-aspiring guarantees that, when it uses our predicted class probabilities in each round $t$, its \emph{anticipated marginal coverage} is equal to $1-\alpha$ --- and not just on average over all rounds $t \in [T]$, but also \emph{on any given round $t$}.

Perhaps the simplest example of an aspiring prediction set algorithm $S$ is a fractional variant of the naive algorithm discussed at the beginning of this section: 

\begin{enumerate}
    \item Given a predicted probability vector $p$, order labels $y \in \cY$ in decreasing order of $p_y$ (ignoring transcript $\pi$ and features $x$).
    \item Start with an empty prediction set, and keep adding labels to it in decreasing order of $p_y$, until the anticipated coverage (i.e., the running sum of predicted probabilities) has become $\geq 1-\alpha$.
    \item If the final running sum of $p_y$ has overshot $1-\alpha$, the last included label $y_\mathrm{last}$ should be included ``fractionally'', i.e., $(S(\pi, x, p))_{y_\mathrm{last}} \in (0, 1)$, to satisfy \emph{exactly} $1-\alpha$ anticipated coverage.
\end{enumerate}


Observe that at every round $T$, the anticipated coverage of any $(1-\alpha)$-aspiring algorithm $S$ satisfies: 
\[
    \widetilde{\Pr}_T[Y \in S] = \frac{1}{T} \sum_{t = 1}^T \sum_{y \in \cY} p_{t, y} \cdot \Pr[y \in S_t] = \frac{1}{T} \sum_{t=1}^T (1-\alpha) = 1-\alpha.
\]
Consequently, Corollary~\ref{cor:marginal_coverage} implies the following about $(1-\alpha)$-aspiring algorithms that use our predictions: 

\begin{corollary}[Valid realized marginal coverage for $(1-\alpha)$-aspiring prediction set algorithms] \label{cor:aspiring-marginal}
    For any family $\cS$ of $(1-\alpha)$-aspiring prediction set algorithms, instantiating \texttt{Class-Probabilities-for-Prediction-Sets} with the collection of events $\cE = \cE_\cS$ guarantees $1-\alpha$ marginal coverage \emph{to all algorithms $S \in \cS$ simultaneously}, at all times $T \leq T_\mathrm{max}$, with deviation from exact $1-\alpha$ coverage decaying as (or better than) $\tilde{O}\left(\frac{|\cY| \ln(|\cY| |\cS| T_\mathrm{max})}{\sqrt{T}}\right)$.
\end{corollary}
When might it make sense to provide appropriately unbiased class probabilities to ensure valid realized coverage for many downstream prediction set algorithms with the same coverage target $1-\alpha$? Note that there are generally many ways to select a prediction set over $\cY$ while achieving the same $(1-\alpha)$ anticipated coverage --- the difference between such different ways to form a $(1-\alpha)$-prediction set may come down, e.g., to the difference in the \emph{objective functions} (downstream from the prediction set algorithm) that these selections optimize. Thus, one way to interpret this result is that it provides valid $(1-\alpha)$-coverage in the face of uncertainty over which one of a finite class of downstream objectives the prediction sets will be used to optimize for.


In fact, our method does not need the algorithms in the collection $\cS$ to all aspire to the same coverage level $1-\alpha$; they can each have their individual target coverage levels, and our guarantees will still hold:

\begin{corollary}[Valid realized marginal coverage at different target levels] \label{cor:aspiring-multiple-targets}
    Consider any collection $\cS = \{S_1, S_2, \ldots, S_N\}$ of downstream prediction set algorithms where each $S_i$ is $(1-\alpha_i)$-aspiring --- i.e., different algorithms in the collection are seeking different coverage levels. Then, by instantiating our \texttt{Class-Probabilities-for-Prediction-Sets} method with the collection of events $\cE = \cE_\cS$, we guarantee valid realized target $1-\alpha_i$ coverage \emph{to each $S_i \in \cS$ simultaneously}, at all times $T \leq T_\mathrm{max}$, with deviation from exact $1-\alpha_i$ coverage decaying as (or better than) $\tilde{O}\left(\frac{|\cY| \ln(|\cY| |\cS| T_\mathrm{max})}{\sqrt{T}}\right)$.
\end{corollary}
The probability vector predictions issued by \texttt{Class-Probabilities-for-Prediction-Sets} under this instantiation will thus have the interesting property of being \emph{simultaneously valid at different coverage levels}.


\subsubsection{Transparent Conditional Coverage for Set-Size-Conditional and Multigroup-Fair Validity}
\label{sec:conditional-coverage}

Thus far we have shown how to make probabilistic predictions that have transparent coverage guarantees. But this is not the same thing as making high quality predictions. Can we guarantee that our predictions are somehow ``good'' while still offering the kinds of attractive coverage guarantees we have just given? 

In general, in an adversarial setting, there is no way to guarantee that our predictions are ``good'' in an absolute sense, but we can make sure that they satisfy various consistency conditions. For set-valued predictions, a natural goal for us will be to ask that such consistency conditions translate, for downstream predictors using our predictions, into \emph{conditional coverage guarantees}, a subject of intensive study in distribution-free uncertainty quantification at least since~\cite{vovk2012conditional}. At a high level, conditionally valid coverage guarantees help make sure that target $(1-\alpha)$-coverage was \emph{not} simply achieved by overcovering on some subsets of the data and undercovering on others. 

We will now formalize the observation that our framework allows us to layer in various kinds of probabilistic consistency constraints: we can add events to the collection $\cE$ that will ensure that our predictions $p_t$ entail transparent coverage guarantees over any --- possibly weighted --- subsequences of interest in the data. 

\paragraph{Boosting a model's \emph{conditional} coverage transparency and validity}

We begin by defining \emph{conditional} realized and anticipated coverage for a prediction set algorithm $S$ (and their per-label variants); this will directly parallel the marginal coverage Definitions~\ref{def:anticipated-marginal}, \ref{def:realized-marginal}, and~\ref{def:per-label}. We define them to be conditional on any (nonempty) weighted subsequence of rounds: recall that a \emph{weighted subsequence} $W$ is a mapping $W: \mathbb{N} \to [0, 1]$, where $W(t)$ is the weight with which round $t$ belongs to the subsequence. For unweighted subsequences, i.e., $W: \mathbb{N} \to \{0, 1\}$, $W(t) = 1$ means that round $t$ was included in the subsequence, and vice versa for $W(t) = 0$.

\begin{definition}[Conditional Coverage Metrics]
    Fix a prediction set algorithm $S$, a $T$-round transcript $\pi_T$, and a weighted subsequence $W: \mathbb{N} \to [0, 1]$. Let $(y_t)_{t \in [T]}$ be the realized labels and $(p_{t, y})_{t \in [T], y \in \cY}$ be our realized predictions over the transcript $\pi_T$. 
    
    The \emph{realized} and \emph{anticipated  coverage} of $S$ \emph{conditional} on $W$ are given by:
    \[
        \overline{\Pr}_T[Y \in S \, | \, W] := \frac{\sum_{t = 1}^T W(t) \cdot \Pr[y_t \in S_t]}{\sum_{t \in [T]} W(t)}, 
        \quad \text{and }
        \widetilde{\Pr}_T[Y \in S \, | \, W] := \frac{\sum_{t = 1}^T W(t) \sum_{y \in \cY} p_{t, y} \cdot \Pr[y \in S_t]}{\sum_{t \in [T]} W(t)}.
    \]

    The \emph{realized and anticipated per-label coverage} of $S$ \emph{conditional} on $W$ for any label $y \in \cY$ are given by:
    \[  
        \overline{\Pr}_T[y \in S \, | \, W] := \frac{\sum_{t = 1}^T W(t) \cdot 1[y = y_t] \cdot \Pr[y \in S_t]}{\sum_{t \in [T]} W(t)}, 
        \quad \text{and } 
        \widetilde{\Pr}_T[y \in S \, | \, W] := \frac{\sum_{t = 1}^T W(t) \cdot p_{t, y} \cdot \Pr[y \in S_t]}{\sum_{t \in [T]} W(t)}.
    \]
\end{definition}

Clearly, it would be ideal to have our transparent coverage guarantees hold conditional on \emph{all} subsequences. Since this is impossible, we must restrict attention to offering such guarantees conditional on all \emph{relevant} subsequences in the data. Deciding which subsequences are the relevant ones can be done based either on trying to correct for past failures of the model (i.e., over- or undercoverage), or based on normative considerations (e.g., if some regions in the data must be paid special attention due to them corresponding to protected group of points/individuals). 

We will address normative scenarios shortly; for now, we assume that our goal is to form subsequences in a data-driven way in an attempt to find and correct for persistent over- or undercoverage of the model on certain regions of the dataset, thus \emph{boosting} its coverage performance. This is often best relegated to side models which are trained to perform such coverage checks, detecting any over- or underconfidence regions based on past history. We will refer to such models as \emph{booster models}. For our purposes, we will abstract away any internal details of booster models, and instead define a booster through the mapping from transcripts, contexts, and predictions to a dynamically evolving subsequence of rounds.

\begin{definition}[Booster]
    A \emph{booster} is any mapping of the form $B: \Pi^* \times \cX \times \Delta \cY \to [0, 1]$, where for any round $t$, we interpret $B(\pi_{t-1}, x_t, p_t) \in [0, 1]$ as the booster model's decision to include, not include, or partially include the current round $t$ in its generated subsequence of rounds. 
    
    Note that the booster $B$ naturally corresponds to a \emph{weighted subsequence of rounds}, so we will denote anticipated and realized coverage conditional on its generated subsequence by $\overline{\Pr}_T[Y \in S \, | \, B], \widetilde{\Pr}_T[Y \in S \, | \, B]$, $\overline{\Pr}_T[y \in S \, | \, B], \widetilde{\Pr}_T[y \in S \, | \, B]$.
\end{definition}

We now provide our most general transparent conditional coverage guarantee, which generalizes the marginal guarantees of Theorem~\ref{thm:per-label} and Corollary~\ref{cor:marginal_coverage}. The proof proceeds in much the same way as before, with the main difference that we now use an appropriately expanded family of conditioning events to instantiate \texttt{UnbiasedPrediction}, as well as that our guarantees for each booster now depend on its cumulative incidence $\sum_{t \in [T]} B(\pi_{t-1}, x_t, p_t)$, rather than on $T$. (Note: below, we may simply write $B(x_t, p_t)$, omitting the transcript.)

\begin{theorem}[Conditional Transparent Coverage for Arbitrary Booster Collection]
    Fix a collection $\cS$ of prediction set algorithms and a collection $\cB$ of boosters. If \texttt{Class-Probabilities-for-Prediction-Sets} is instantiated with the event collection $\cE := \cE_\cS \times \cE_\cB$, where $\cE_\cS$ was defined in Equation~\ref{eq:marginal-events} and 
    \begin{equation*}
        \cE_\cB := \{E_B\}_{B \in \cB}, \quad \text{where $E_{B}(p) := B(\pi_{t-1}, x_t, p)$ for $p \in \Delta \cY$},
    \end{equation*}
    then in expectation over $P_1, P_2, \ldots \in \Delta \Delta \cY$, the distributions \texttt{Class-Probabilities-for-Prediction-Sets} outputs, 
    the \emph{anticipated} and \emph{realized} coverage of every label $y \in \cY$ by every prediction set algorithm $S \in \cS$ conditional on every $B \in \cB$ will be approximately equal at every round $T \in [T_\mathrm{max}]$ (where $T_\mathrm{max}$ is the maximum time horizon):
    \begin{align*}
        \E_{p_t \sim P_t, \, \forall t} & \left[ \left( \sum_{t \in [T]} B(x_t, p_t) \right) \cdot \left| \overline{\Pr}_T[y \in S | B] - \widetilde{\Pr}_T[y \in S | B] \right| \right] \\
        &\leq O\left(\ln(|\cY| |\cS| |\cB| T_\mathrm{max}) + \sqrt{\ln(|\cY| |\cS| |\cB| T_\mathrm{max}) \cdot \sum\limits_{t=1}^T \E\limits_{p_t \sim P_t} \left[B(x_t, p_t) \cdot \Pr[y \in S_t] \right]}\right).
    \end{align*}

    Moreover, the \emph{anticipated} and \emph{realized} conditional coverage of every prediction set algorithm $S \in \cS$ conditional on every $B \in \cB$ will be approximately equal for all $T \in [T_\mathrm{max}]$ (where $T_\mathrm{max}$ is the maximum time horizon):
    \begin{align*}
        \E_{p_t \sim P_t, \, \forall t} &\left[ \left( \sum_{t \in [T]} B(x_t, p_t) \right) \cdot \left|\overline{\Pr}_T[Y \in S | B] - \widetilde{\Pr}_T[Y \in S | B] \right| \right] \\
        &\leq O\left(|\cY| \ln(|\cY| |\cS| |\cB| T_\mathrm{max})
        + \sum_{y \in \cY} \sqrt{\ln(|\cY| |\cS| |\cB| T_\mathrm{max}) \cdot \sum_{t\in [T]} \E\limits_{p_t \sim P_t} \left[B(x_t, p_t) \cdot \Pr[y \in S_t] \right]}\right) \\
        &\leq \tilde{O} \left(|\cY| \ln(|\cY| |\cS| |\cB| T_\mathrm{max}) \sqrt{\sum_{t\in [T]} \E\limits_{p_t \sim P_t} \left[B(x_t, p_t) \right]} \right).
    \end{align*}
\end{theorem}

As a corollary of this result, we obtain the corresponding coverage guarantees for any family of $(1-\alpha)$-aspiring downstream prediction set algorithms. This is a strengthening of Corollary~\ref{cor:aspiring-marginal}, and gives our most general valid $(1-\alpha)$ conditional coverage guarantees.

\begin{corollary}[Valid realized $\cB$-conditional coverage for $(1-\alpha)$-aspiring prediction set algorithms] \label{cor:aspiring-conditional}
    For any family $\cS$ of $(1-\alpha)$-aspiring prediction set algorithms and any booster family $\cB$, instantiating our method \texttt{Class-Probabilities-for-Prediction-Sets} with the collection of events $\cE = \cE_\cS \times \cE_\cB$ guarantees $1-\alpha$ realized coverage conditional on every $B \in \cB$ to all algorithms $S \in \cS$ simultaneously, at all times $T \leq T_\mathrm{max}$, with deviation from exact $1-\alpha$ coverage satisfying:
    \[
        \E_{p_t \sim P_t, \, \forall t} \left[ \left( \sum_{t \in [T]} B(x_t, p_t) \right) \cdot \left|\overline{\Pr}_T[Y \in S | B] - (1-\alpha) \right| \right]
        \leq \tilde{O} \left(|\cY| \ln(|\cY| |\cS| |\cB| T_\mathrm{max}) \sqrt{\sum_{t\in [T]} \E\limits_{p_t \sim P_t} \left[B(x_t, p_t) \right]} \right).
    \]
\end{corollary}

We now briefly discuss two important applications of our transparent conditional coverage guarantees: set-size-conditional coverage (cf.\ \emph{size-stratified coverage validity} of~\cite{angelopoulos2020uncertainty}), and multigroup coverage (introduced in~\cite{jung2021moment,bastanipractical, jung2023batch}). We have singled them out in the pseudocode for our method (Algorithm~\ref{alg:conformal}) due to their importance, even though they follow from very simple instantiations of our boosters.

\paragraph{Set-size-conditional coverage}

The idea behind obtaining valid coverage conditional on the set size is quite intuitive~\cite{angelopoulos2021gentle, angelopoulos2020uncertainty}. We ideally want to ensure that the prediction set is always appropriately large, reflecting the true amount of uncertainty with respect to any sample --- and thus want the target $(1-\alpha)$ coverage to hold conditional on picking any possible set size. In particular, this helps avoid situations where the prediction set is very small \emph{and} undercovers (which clearly means the optimal prediction set should have been made larger), or where the prediction set is very large \emph{and} overcovers (which clearly means it would be better to remove some labels from the set, improving both its coverage and its size).

While set-size-conditional coverage has been explored in the batch conformal setting, to our knowledge such coverage guarantees have not been rigorously studied in the online adversarial case; we fill this gap in Theorem~\ref{thm:set-size} below. Before presenting our result, we first formally introduce and discuss a generalized measure of prediction set size.

\begin{definition}[Set size function]
    A \emph{set-size function} $\mathrm{sz}$ is a mapping\footnote{As mentioned earlier, $\mathrm{sz}$ can more generally be any bounded mapping from $[0, 1]^{|\cY|}$ to $[0, \infty)$; in that case, we can simply discretize it to reduce to the definition given here.} $\mathrm{sz}: [0, 1]^{|\cY|} \to \{0, 1, \ldots, N_\mathrm{maxsz}\}$, where $N_\mathrm{maxsz}$ denotes the maximum set size. If the prediction set algorithm is known to be deterministic, the set size function can simply be defined as a mapping $\mathrm{sz}: 2^\cY \to \{0, 1, \ldots, N_\mathrm{maxsz}\}$.

    Given any realized prediction set $S_\mathrm{pred}$, its \emph{size} is then given by $\mathrm{sz}(S_\mathrm{pred})$.
\end{definition}

What are some examples of useful set size functions? The simplest is to take, for any prediction set $S_\mathrm{pred} \in [0, 1]^{|\cY|}$, its size to be $\mathrm{sz}(S_\mathrm{pred}) := \sum_{y \in \cY} (S_\mathrm{pred})_y$. If the prediction set is deterministic, i.e.\ $S_\mathrm{pred} \in \{0, 1\}^{|\cY|}$, this reduces to $\mathrm{sz}(S_\mathrm{pred}) := \sum_{y \in \cY} 1[y \in S_\mathrm{pred}]$, which is the size of $S_\mathrm{pred}$ as a set of labels.
We can also define a \emph{weighted} variant of this set size mapping: given any fixed nonnegative label weights $(w_y)_{y \in \cY}$, we can let $\mathrm{sz}(S_\mathrm{pred}) := \sum_{y \in \cY} w_y \cdot (S_\mathrm{pred})_y$, which for deterministic prediction sets reduces to $\mathrm{sz}(S_\mathrm{pred}) := \sum_{y \in \cY} w_y \cdot 1[y \in S_\mathrm{pred}]$. This can be useful when there is a different ``cost to pay'' for including each label $y \in \cY$ into the prediction set.

Beyond these examples, we can define $\mathrm{sz}(\cdot)$ in arbitrary ways. For instance, one may intuitively desire that for deterministic sets $S_\mathrm{pred}$, the set size $\mathrm{sz}(S_\mathrm{pred})$ should grow with $|S_\mathrm{pred}|$; however, we place no such requirement on $\mathrm{sz}$. Even more generally, $\mathrm{sz}$ can be interpreted decision-theoretically as a (dis)utility function that maps any selected prediction set to a certain (dis)utility value. Our next guarantee covers all these cases:\

\begin{theorem}[Valid set-size-conditional coverage for $(1-\alpha)$-aspiring prediction set algorithms] \label{thm:set-size}
    For any family $\cS$ of $(1-\alpha)$-aspiring prediction set algorithms and any set size function $\mathrm{sz}: [0, 1]^{|\cY|} \to \{0, 1, \ldots, N_\mathrm{maxsz}\}$, 
    instantiating our method \texttt{Class-Probabilities-for-Prediction-Sets} with the collection of events $\cE = \bigcup_{S \in \cS} \cE_S \times \cE_\mathrm{sz, S}$, where for each $S \in \cS$, we define $\cE_S := \{E_{S, y}\}_{y \in \cY}$ with $E_{S, y}$ defined as in Equation~\ref{eq:marginal-events}, and
    \[
        \cE_\mathrm{sz, S} := \{E_{\mathrm{sz}, S, n}\}_{0 \leq n \leq N_\mathrm{maxsz}}, \text{ where $E_{\mathrm{sz}, S, n}(p) := 1[\mathrm{sz}(S(\pi_{t-1} \cup \{x_t\}, p)) = n]$ for all $p \in \Delta \cY$,}
    \]
    guarantees $1-\alpha$ realized coverage, conditional on every particular set size $0 \leq n \leq N_\mathrm{maxsz}$, to all algorithms $S \in \cS$ simultaneously, at all times $T \leq T_\mathrm{max}$, with deviation from target coverage satisfying:
    \[
        \E_{p_t \sim P_t, \,\, \forall t} \!\!\!\!\!\!\!\!\! \left[ \mathrm{num}_T(\mathrm{sz} = n) \cdot \left|\overline{\Pr}_T[Y \in S \,|\, \mathrm{sz} = n] \!-\! (1\!-\!\alpha) \right| \right]
        \leq \tilde{O} \left(|\cY| \ln(|\cY| |\cS| N_\mathrm{maxsz} T_\mathrm{max}) \sqrt{\E\limits_{p_t \sim P_t \, \forall\, t} \!\!\! \!\!\!\! \! \left[\mathrm{num}_T(\mathrm{sz} = n)\right]} \right),
    \]
    where $\mathrm{num}_T(\mathrm{sz} = n) := \sum_{t \in [T]} 1[\mathrm{sz}(S(\pi_{t-1}, x_t, p)) = n]$ denotes the incidence of set size $n$ through round $T$, and $\overline{\Pr}_T[Y \in S | \mathrm{sz} = n]$ is a shorthand for the realized coverage conditional on the corresponding subsequence.
\end{theorem}
This result is a direct corollary of Corollary~\ref{cor:aspiring-conditional}, since the events corresponding to each prediction set size are instances of boosters.

\paragraph{Multigroup coverage} 

As originally introduced, the idea behind multigroup coverage is to protect various contextually-defined subsets of the data from over- and undercoverage. This is best illustrated by settings where the feature space $\cX$ corresponds to feature vectors of individuals in a population. In that case, many notions of demographic groups considered in the fairness literature --- e.g., defined by sex, age, skin color, gender --- are captured by defining appropriate context-dependent group mappings $G: \cX \to [0, 1]$, where $G(x) = 1$ means the individual with features $x$ belongs to the group and vice-versa for $G(x) = 0$, and $G(x) \in (0, 1)$ denotes \emph{partial} group membership (e.g., belonging to the group with probability $G(x)$). The goal is then to ensure that prediction sets produced by an algorithm $S$ achieve $\approx 1-\alpha$ coverage not just on average over the entire population, but also conditionally on each group $G$ that we wish to protect. This desideratum is referred to as \emph{multi}group coverage to emphasize that the protected groups do not necessarily partition the feature space $\cX$ but can instead overlap in arbitrary ways.

\begin{definition}[Group family]
    A \emph{group family} is a (finite) collection of groups $G \in \cG$, where each $G$ is an arbitrary mapping from $\cX$ to $[0, 1]$. We assume that we have oracle access to every $G \in \cG$, i.e., that we can query it in $O(1)$ time on every individual $x \in \cX$.
\end{definition}

We now state our multigroup coverage result; we note that it improves on the previous multigroup coverage result of~\cite{bastanipractical}, as our coverage error bound now has the optimal $O(\ln |\cG|)$ dependence on the number of groups in the collection, while retaining the statistically optimal $1/\sqrt{\sum_{t \in [T]} G(x_t)}$ dependence on the frequency of occurrence of every group $G \in \cG$.

\begin{theorem}[Valid multigroup coverage for $(1-\alpha)$-aspiring prediction set algorithms]
    For any family $\cS$ of $(1-\alpha)$-aspiring prediction set algorithms and any group collection $\cG$, 
    instantiating our method \texttt{Class-Probabilities-for-Prediction-Sets} with the collection of events $\cE = \cE_\cS \times \cE_\cG$, where:
    \[
        \cE_\cG := \{E_G\}_{G \in \cG}, \text{ where $E_G(p) := G(x_t)$ for all $p \in \Delta \cY$,}
    \]
    guarantees $1-\alpha$ realized coverage, conditional on every demographic group $G \in \cG$, to all algorithms $S \in \cS$ simultaneously, at all times $T \leq T_\mathrm{max}$, with deviation from exact $1-\alpha$ coverage satisfying:
    \[
        \E_{p_t \sim P_t, \, \forall t} \left[ \left|\overline{\Pr}_T[Y \in S \, | \, G] - (1-\alpha) \right| \right]
        \leq \tilde{O} \left( \frac{|\cY| \ln(|\cY| |\cS| |\cG| T_\mathrm{max})}{\sqrt{\mathrm{num}_T(G)}}\right),
    \]
    where $\mathrm{num}_T(G) := \sum_{t \in [T]} G(x_t)$ denotes the incidence of group $G$ through round $T$, and $\overline{\Pr}_T[Y \in S \, | \, G]$ is a shorthand for the realized coverage conditional on the group's subsequence.
\end{theorem}

This result is a direct corollary of Corollary~\ref{cor:aspiring-conditional}, since the events corresponding to each demographic group $G$ are instances of boosters.






\subsection{Best-in-Class Probabilistic Predictions}

We have just established that our probability vector predictions are not only capable of providing marginal coverage guarantees, but can in fact (with the help of an appropriately augmented conditioning event collection $\cE$) provide conditional coverage guarantees of any desired granularity. This represents one way of ensuring that our predictions are `good' in an \emph{absolute} sense --- i.e. if we had the computational and statistical power to condition on all possible events, our predictor would approach the Bayes optimal one; however, in practice, conditioning on all events is unachievable. 

Instead, we now focus on showing that we can make our predictions ``good'' in a very useful \emph{relative} sense: given any collection of benchmark prediction models (which can be arbitrary), we will want our predictions to have lower loss than those of the best model in the benchmark class. Thus, whenever we have some other way of making good predictions (e.g. by using a large neural network), we can use that method as part of our own to endow predictions of the same or better quality with transparent coverage guarantees.

Indeed, in various domains, we might have high quality predictors trained on large quantities of data: these predictors might perform very well according to standard performance metrics, and yet not have the kinds of transparent coverage guarantees that we aim to give here. Our goal will be to take any (polynomially sized) collection of such predictors as inputs, and provide probabilistic predictions with anytime transparent coverage guarantees that perform as well as the best predictors in this benchmark class. This result will thus demonstrate that, in some sense, our anytime transparent coverage guarantees discussed above come ``for free'', rather than at the cost of predictive performance.

To formalize this, we introduce two definitions, of a class of predictors against which we will be able to compete, and of a class of loss functions measuring how well we can compete with these predictors.

\begin{definition}[Competing Predictor]
    A predictor $q$ is defined as a sequence of prediction mappings $q_1, \ldots, q_t, \ldots$, where each $q_t: \Delta \cY \to \Delta \cY$ is a mapping from the learner's predictions to the predictor's predictions such that if the learner makes prediction $p_t$ in round $t$ then the predictor $q$ will make prediction $q_t(p_t)$ in round $t$.
    
    In particular, predictor $q$ is called a \emph{competing predictor} if for each $t$, the prediction mapping $q_t$ is revealed to the learner at the beginning of round $t$ (i.e., before the learner has to make a prediction). 
\end{definition}
This definition is strong in that it allows (though does not require) the competing predictor to (1) depend in a functional way upon the predictions that the learner is about to issue, and (2) be updated dynamically over time without revealing any prediction mapping $q_t$ to the learner until round $t$. An important special case is when the predictor $q_t$ is an arbitrary pre-trained model $f$ of the context, and simply predicts $q_t = f(x_t)$ --- but our definition is more general and allows the predictor to change over time and even to depend on our own predictions.

\subsubsection{Scoring Rules, Losses, and Bregman Divergences}

Following~\cite{grunwald2004game, abernethy12, kull2015novel}, we now define a natural family of loss functions known as \emph{Bregman scores}. All such loss functions assess the quality of a predictor using a chosen Bregman divergence of its predictions from the ground truth.

\begin{definition}[Scoring Rule, Expected Score]
    A \emph{scoring rule} is a nonnegative extended-valued mapping $\phi: \Delta \cY \times \mathrm{oh}(\Delta \cY) \to [0, \infty]$, where $\mathrm{oh}(\Delta \cY)$ denotes the set of the $|\cY|$ standard basis vectors in $\Delta \cY$.

    An \emph{expected score} $\ell$ associated with scoring rule $\phi$ is a mapping $\ell: \Delta \cY \times \Delta \cY \to [0, \infty)$ such that for every $p_1, p_2 \in \Delta \cY$,
    \[\ell(p_1, p_2) = \sum_{i \in [k]} p_{2, i} \cdot \phi(p_1, e_i),\]
    where $e_i \in \Delta \cY$ denotes the $i$th standard basis vector.
\end{definition}

\begin{definition}[Bregman Divergence]
    Given a strictly convex function $f: [0, 1] \to \R$ that is differentiable on $(0, 1)$, its associated \emph{Bregman divergence over the simplex $\Delta \cY \subset \R^k$} is the mapping $d_f: \Delta \cY \times \mathrm{ri} (\Delta \cY) \to \R$ defined\footnote{Here, $\mathrm{ri}$ denotes the relative interior of a set. The definition of Bregman divergence that restricts the second argument to be in the relative interior of the domain is customary in the literature~\cite{banerjee2005clustering}, and helps avoid problems at boundary points (such as e.g.\ for KL divergence). However, we note that for sufficiently well-behaved functions $f$, their Bregman divergence $d_f$ can easily be well-defined on the boundary --- such is the case e.g.\ for the squared loss. When that is the case, additional restrictions that we have to place on the second argument of the Bregman divergence (i.e., in our context, assuming that all coordinates of our and competing predictions are nonzero, as in Theorem~\ref{thm:omnipred}) can naturally be ignored.} as
    \[d_f(p_1, p_2) = \sum_{i \in [k]} \{ f(p_{1,i}) - f(p_{2, i}) - f'(p_{2, i}) \cdot (p_{1, i} - p_{2, i}) \}.\]
    We simply refer to a function $d: \Delta \cY \times \mathrm{ri} (\Delta \cY) \to \R$ as a \emph{Bregman divergence} if $d$ is identically equal to $d_f$ for some strictly convex differentiable $f: [0, 1] \to \R$.

    Furthermore, for any $i \in [k]$ we define the notation $d^i_f(p_1, p_2) := f(p_{1,i}) - f(p_{2, i}) - f'(p_{2, i}) \cdot (p_{1, i} - p_{2, i})$, so that $d_f(p_1, p_2) = \sum_{i \in [k]} d^i_f(p_1, p_2)$ for any $p_1 \in \Delta \cY, p_2 \in \mathrm{ri} (\Delta \cY)$. Viewed as a mapping from $p_{1, i} \in [0, 1]$ and $p_{2, i} \in (0, 1)$ to $[0, \infty)$, $d^i$ is a \emph{Bregman divergence (between two scalars) over $[0, 1]$}.
\end{definition}
We note that our definition of a Bregman divergence between two vectors asks that it be \emph{separable} across the coordinates of these vectors (as opposed to the most general definition of Bregman divergence that drops this requirement); this separability property is convenient for us to use given our techniques and is satisfied for major Bregman divergences of interest (e.g., Euclidean distance, KL divergence, Itakura-Saito distance).

\begin{definition}[Bregman Score]
    An expected score $\ell: \Delta \cY \times \Delta \cY \to [0, \infty)$ is called a \emph{Bregman score} if the mapping $D: \Delta \cY \times \Delta \cY \to \R$ defined, for every $p_1, p_2 \in \Delta \cY$, as
    \[D(p_1, p_2) = \ell(p_1, p_2) - \ell(p_2, p_2)\] satisfies $D(p_1, p_2) = d(p_2, p_1)$ for some Bregman divergence $d$ and all $p_1 \in \Delta \cY, p_2 \in \mathrm{ri}(\Delta \cY)$.
\end{definition}

The following two examples showcase two key members of the class of Bregman scores.
\begin{example}[Brier score]
The \emph{squared Euclidean distance}, defined as \[d_{L^2}(p_1, p_2) = ||p_1-p_2||^2 = \sum_{i \in [k]} (p_{1, i} - p_{2, i})^2,\] is a Bregman divergence with respect to the strictly convex differentiable function $f(x) = x^2$. 

Consider the \emph{Brier scoring rule} $\phi_\mathrm{Brier}$, given by $\phi_\mathrm{Brier}(p, e) := \sum_{i \in [k]} (p_i - e_i)^2$ for any $p \in \Delta \cY$ and any standard basis vector $e$. The \emph{expected Brier score} $\ell_\mathrm{Brier}$ is then given by \[\ell_\mathrm{Brier}(p_1, p_2) = \sum_{i \in [k]} p_{2, i} \left(\sum_{j \neq i} p_{1, j}^2 + (1 - p_{1, i})^2 \right) = 1 + ||p_1||^2 - 2 \langle p_1, p_2 \rangle.\]

Now, it can be easily verified that the expected Brier score is in fact the Bregman score corresponding to the squared Euclidean distance, since $\ell_\mathrm{Brier}(p_1, p_2) - \ell_\mathrm{Brier}(p_2, p_2) = 1+ ||p_1||^2 - 2\langle p_1, p_2 \rangle - (1 + ||p_2||^2 - 2 \langle p_2, p_2 \rangle) = ||p_1||^2 - 2 \langle p_1, p_2 \rangle + ||p_2||^2 = ||p_1 - p_2||^2 = d_{L^2}(p_2, p_1)$.
\end{example}

\begin{example}[Logarithmic score]
The \emph{KL divergence}, defined as \[d_\mathrm{KL}(p_1, p_2) = \sum_{i \in [k]} p_{1, i} \log \frac{p_{1, i}}{p_{2, i}},\] is a Bregman divergence with respect to the strictly convex function $f(x) = x \log x$ differentiable on $(0, \infty)$. 

Consider the \emph{logarithmic scoring rule} $\phi_\mathrm{Log}$, given by $\phi_\mathrm{Log}(p, e) := - \log p_{i(e)}$ for any $p \in \Delta \cY$ and any standard basis vector $e$ (where by $i(e)$ we denote the index of the nonzero coordinate of $e$). The \emph{expected logarithmic score} $\ell_\mathrm{Log}$ is then given by \[\ell_\mathrm{Log}(p_1, p_2) = \sum_{i \in [k]} p_{2, i} \left(- \log p_{1, i}\right) = -\sum_{i \in [k]} p_{2, i} \log p_{1, i}.\]

As it turns out, the expected logarithmic score is in fact the Bregman score with respect to the KL divergence, since $\ell_\mathrm{Log}(p_1, p_2) - \ell_\mathrm{Log}(p_2, p_2) = -\sum_{i \in [k]} p_{2, i} \log p_{1, i} + \sum_{i \in [k]} p_{2, i} \log p_{2, i} = \sum_{i \in [k]} p_{2, i} \log \frac{p_{2, i}}{p_{1, i}} = d_\mathrm{KL}(p_2, p_1)$.
\end{example}

In the remainder of this section, we will lightly abuse notation. Namely, we henceforth let $y_t$, formerly the notation for the realized label in round $t$, stand for the standard basis vector in $\Delta \cY$ whose entry corresponding to the realized label is $1$ and all other entries are $0$. This makes quantities like $\ell(p_t, y_t)$ well-defined since now both $p_t, y_t \in \Delta \cY$.

We now define our notion of loss for a sequential predictor --- which will be its cumulative Bregman loss across all rounds.

\begin{definition}[Bregman Loss of a Predictor]
Given any Bregman score $\ell$, we define the Bregman loss of a predictor $q = (q_t)_{t \in [T]}$ over any $T$-round transcript $\pi_T$ as:
\[L_\ell(q, \pi_T) := \sum_{t \in [T]} \ell(q_t, y_t).\]
\end{definition}

\subsubsection{A Best-in-Class Result for Bregman Losses}

Now, we are ready to show the following result: for any family $\cQ$ of competing predictors $q$, if we condition on the product of some $\delta$-level sets of each competing predictor $q \in \cQ$ with the $\delta$-level sets of our predictor in every coordinate (altogether, this amounts to $O(k |\cQ| \lceil \delta^{-2} \rceil)$ events) with $\delta$ chosen appropriately, then for all Lipschitz Bregman losses simultaneously, the loss of our predictor will be no greater than the loss of any competing predictor, up to lower-order terms.

\begin{definition}[$\delta$-level sets] \label{def: ls}
    A partition $\Omega$ of the interval $[0, 1]$ into subintervals, in order, $\Omega_1, \ldots, \Omega_N$ is called a $\delta$-partition if the length of every $\Omega_i$ is no greater than $\delta$, for all $i \in [N]$. For a predictor with values in $[0, 1]$, its $\delta$ level sets are any level sets $\Omega$ generated by a $\delta$-partition of $[0, 1]$.
\end{definition}

\begin{theorem}[Beating Competing Predictors under All $L$-Lipschitz Bregman Losses]
    \label{thm:omnipred}
    Fix any $L > 0$ and let $\cL_L$ be the family of all Bregman losses $L_\ell$ whose associated Bregman divergence $d_\ell$ is $L$-Lipschitz in its second argument:
    \[|d_\ell(p_1, p_2) - d_\ell(p_1, p_2')| \leq L|p_2 - p_2'| \quad \text{for all $p_1 \in \Delta \cY$ and $p_2, p_2' \in \mathrm{ri} (\Delta \cY)$}.\] 
    
    Consider any family $\cQ$ of competing predictors $q = (q_t)_{t \in \mathbb{N}}$. Suppose we make our predictor $p = (p_t)_{t \in \mathbb{N}}$ unbiased conditional on a family of events $\cE_\cQ$ that includes, for every $q \in \cQ$, the product of some coordinate-wise (for $i \in [k]$) $\delta$-level sets $\Omega^{p, i}$ and $\Omega^{q, i}$ of $p$ and $q$, respectively: \[\cE_\cQ = \{\cE_\mathrm{ls(p, q)}\}_{q \in \cQ} \text{ with } \,  \cE_\mathrm{ls(p, q)} := \{E^i_{a, b}\}_{(i, a, b) \in [k] \times [N_1] \times [N_2]},\]
    where for each $i \in [k]$, $a \in [N_1], b \in [N_2]$, we define the product of the $a$th level set of $p$ and the $b$th level set of $q$ with respect to coordinate $i$ as 
    \[E^i_{a, b}(p_t) := \textbf{1}[p_{t, i} \in \Omega^{p, i}_a \wedge q_{t, i}(p_t) \in \Omega^{q, i}_b]. \]
    Further, suppose that $p_{t, i}, q_{t, i} \in (0, 1)$ for all $t \in \mathbb{N}$, $q \in \cQ$ and $i \in [d]$.
    
    Then, the expected Bregman loss $L_\ell$ of our predictor $p$ (over its internal randomness) is no greater than \emph{every} competing predictor's loss $L_\ell$ for \emph{all Bregman losses $L_\ell \in \cL_L$}, up to $o(T)$ terms. Namely, setting $\delta = \Theta(T^{-1/4})$, we have:
    \[\E_{\pi_T}[L_\ell(p, \pi_T)] \leq \min_{q \in \cQ} L_\ell(q, \pi_T) + O(L \, k \, \sqrt{\ln (k |\cQ| T)} \, T^{3/4}) \text{ for all } L_\ell \in \cL_L.\]
\end{theorem}
\begin{proof}
    Fix any $L$-Lipschitz Bregman loss $L_\ell \in \cL_L$, and any competing predictor $q \in \cQ$. First, note that for any transcript $\pi_T$ of the first $T$ rounds:
    \[L_\ell(p, \pi_T) - L_\ell(q, \pi_T) = \sum_{t \in [T]} \ell(p_t, y_t) - \ell(q_t, y_t) = \sum_{t \in [T]} d(y_t, p_t) - d(y_t, q_t) = \sum_{i \in [k]} \left( \sum_{t \in [T]} d^i(y_t, p_t) - d^i(y_t, q_t) \right),\]
    where $d$ is the Bregman divergence associated with $\ell$, and $d^i$ are the coordinate-wise terms that our Bregman divergence $d$ splits into.

    Thus, we now fix a single coordinate $i$ and carry out the proof for that coordinate alone; the final bound will just be multiplied by a factor of $k = |\cY|$ to give the overall Bregman loss bound.
    Note that:
    \[ \sum_{t \in [T]} d^i(y_t, p_t) = \sum_{a, b} \sum_{t: E^i_{a, b}(p_t) = 1} d^i(y_{t, i}, p_{t, i}), \quad \sum_{t \in [T]} d^i(y_t, q_t) = \sum_{a, b} \sum_{t: E^i_{a, b}(p_t) = 1} d^i(y_{t, i}, q_{t, i}).\]

    We now focus on any fixed $a \in [N_1], b \in [N_2]$. For convenience, let $n^i_{a, b} := \sum_t \textbf{1}[E^i_{a, b}(p_t) = 1]$ denote the total number of rounds on which the event $E^i_{a, b}$ transpired, and let the average values over this event's subsequence of $p, q, y$ be defined as $\mu^i_{a, b}(p) := \frac{1}{n^i_{a, b}} \sum_t E^i_{a, b}(p_t) \cdot p_{t, i}$, as $\mu^i_{a, b}(q) := \frac{1}{n^i_{a, b}} \sum_t E^i_{a, b}(p_t) \cdot q_{t, i}$, and as $\mu^i_{a, b}(y) := \frac{1}{n^i_{a, b}} \sum_t E^i_{a, b}(p_t) \cdot y_{t, i}$, respectively. 

    Regarding these averages, observe that for any $t$, letting $\alpha^i_{a, b}$ denote any bound on the absolute bias on this subsequence, we get
    \[
    |p_{t, i} - \mu^i_{a, b}(y)|
    \leq |p_{t, i} - \mu^i_{a, b}(p)| + |\mu^i_{a, b}(p) - \mu^i_{a, b}(y)| \leq \delta + \alpha^i_{a, b}.
    \]
    Here, the first term is bounded by $\delta$ since the subinterval $\Omega^{p, i}_{a, b}$ is convex (so that both $p_{t, i}$ and the average over time $\mu^i_{a, b}(y)$ belong to it) and has length at most $\delta$. The second term is bounded by $\alpha^i_{a, b}$ by definition of subsequence bias.

    Similarly, using that $\Omega^{q, i}$ is a $\delta$-level sets partition for the competing predictor $q$, we have for any $t$ that
    \[
    |q_{t, i} - \mu^i_{a, b}(q)|
    \leq \delta.
    \]
    
    By the $L$-Lipschitzness of $d^i$ in its second argument, we therefore have for our predictor $p$ that
    \[
    \sum_{t: E^i_{a, b}(p_t) = 1} d^i(y_{t, i}, p_{t, i}) 
    \leq 
    \sum_{t: E^i_{a, b}(p_t) = 1} d^i \left(y_{t, i}, \mu^i_{a, b}(y) \right) + L \cdot n^i_{a, b} \cdot (\delta + \alpha^i_{a, b}), \]
    and for the competing predictor $q$ that
    \[\sum_{t: E^i_{a, b}(p_t) = 1} d^i(y_{t, i}, q_{t, i})  \geq 
    \sum_{t: E^i_{a, b}(p_t) = 1} d^i \left(y_{t, i}, \mu^i_{a, b}(q) \right) - L \cdot n^i_{a, b} \cdot \delta.\]

    Therefore, we have
    \begin{align*}
        \sum_{t \in [T]} d^i(y_t, p_t) - d^i(y_t, q_t) 
        &=
        \sum_{a, b} \left( \sum_{t: E^i_{a, b}(p_t) = 1} d^i(y_t, p_t) - \sum_{t: E^i_{a, b}(p_t) = 1} d^i(y_t, q_t) \right) \\
        &\leq \sum_{a, b} \left( \sum_{t: E^i_{a, b}(p_t) = 1} d^i \left(y_{t, i}, \mu^i_{a, b}(y) \right) - \sum_{t: E^i_{a, b}(p_t) = 1} d^i \left(y_{t, i}, \mu^i_{a, b}(q) \right) \right) \\
        &+ \sum_{a, b} L \cdot n^i_{a, b} \cdot (2 \delta + \alpha^i_{a, b}) \\
        &\leq \sum_{a, b} L \cdot n^i_{a, b} \cdot (2 \delta + \alpha^i_{a, b}) \\
        &= 2 L \delta T + L \sum_{a, b} n^i_{a, b} \alpha^i_{a, b}.
    \end{align*}
    Here, the last equality uses that $\sum_i n^i_{a, b} = T$ for every $i$. Meanwhile, the second inequality is by the fundamental fact that \emph{Bregman divergences elicit means}, first observed by~\cite{banerjee2005clustering}:
    \begin{theorem}[Adapted from Proposition 1 of~\cite{banerjee2005clustering}]
        For any Bregman divergence $d: [0, 1] \times (0, 1) \to [0, \infty)$, and any random variable $Y$ with finite support over $[0, 1]$ such that $\E[Y] \in (0, 1)$, we have:
        \[\E[Y] = \argmin_{y \in (0, 1)} \E[d(Y, y)].\]
    \end{theorem}
    
    Applying this fact for $d^i$, which is a Bregman divergence over $[0, 1]$, to the random variable $Y$ representing the empirical distribution of the realized labels $y_{t, i}$ over the rounds $t$ on which $E^i_{a, b}(p_t) = 1$, we get:
    \[\mu^i_{a, b}(y) = \frac{1}{n^i_{a, b}} \sum_{t:E^i_{a, b}} y_{t, i} = \E[Y] = \argmin_{y \in (0, 1)} \E[d^i(Y, y)] = \argmin_{y \in (0, 1)} \frac{1}{n^i_{a, b}} \sum_{t:E^i_{a, b}} d^i(y_{t, i}, y).\]
    This then implies the second inequality in the derivation above, since we now have: \[\sum_{t: E^i_{a, b}(p_t) = 1} d^i \left(y_{t, i}, \mu^i_{a, b}(y) \right) \leq \sum_{t: E^i_{a, b}(p_t) = 1} d^i \left(y_{t, i}, \mu^i_{a, b}(q) \right).\]
    
    Summing the inequality $\sum_{t \in [T]} d^i(y_t, p_t) - d^i(y_t, q_t) \leq 2 L \delta T + L \sum_{a, b} n^i_{a, b} \alpha^i_{a, b}$ over $i \in [k]$, we obtain:
    \[
    L_\ell(p, \pi_T) - L_\ell(q, \pi_T) \leq \sum_i \left( 2 L \delta T + L \sum_{a, b} n^i_{a, b} \alpha^i_{a, b} \right) = 2 k L \delta T + L \sum_{i, a, b} n^i_{a, b} \alpha^i_{a, b}.
    \]
    By Theorem~\ref{thm:main-guarantee}, our prediction algorithm achieves for all $i, a, b$ (recalling that our prediction vector dimension is $k$ and the number of events we condition on is of order $|\cQ|k/\delta^2$):
    \[ \E_{\pi_T}[n^i_{a, b} \alpha^i_{a, b}] \leq c' \ln (2 k (|\cQ| k/\delta^2) T) + c' \sqrt{\ln (2 k (|\cQ| k/\delta^2) T)} \cdot \sqrt{\E_{\pi_T}[n^i_{a, b}]}\]
    for some constant $c' > 0$.
    We thus have:
    \[\E_{\pi_T} \left[L_\ell(p, \pi_T) \right] - L_\ell(q, \pi_T) \leq 2 k L \delta T + 2 L c' \ln (2 T k/\delta) k / \delta^2 + c' L \sqrt{\ln (2 k (|\cQ| k/\delta^2) T)} \sum_{i} \left( \sum_{a, b} \sqrt{\E_{\pi_T} [n^i_{a, b}]} \right).\]
    Now note that for any fixed $i$, the events corresponding to each pair $(a, b)$ are disjoint by definition, and thus we have that $\sum_{a, b} n^i_{a, b} = T$ in that case. Since there are at most $1/\delta^2$ terms in the sum $\sum_{a, b} \sqrt{\E_{\pi_T}[n^i_{a, b}]}$, we thus have that its worst-case value over all possible values of $n^i_{a, b}$ is when $\E_{\pi_T}[n^i_{a, b}] = T \delta^2$ for all $a, b$. Bounding for convenience $\sqrt{\ln (2 k (|\cQ| k/\delta^2) T)}$ by $2 \ln (2 k |\cQ| T / \delta)$ in the last term and rearranging, this gives us the worst-case bound of
    \[
        \E_{\pi_T} \left[L_\ell(p, \pi_T) \right] - L_\ell(q, \pi_T) \leq 2 k L \delta T + 5 \ln (2 k |\cQ| T / \delta) \cdot c' \cdot L \cdot  k \cdot \frac{1}{\delta^2} \cdot \sqrt{T \delta^2}.
    \]
    Finally, setting $\delta = \frac{\sqrt{\ln (k |\cQ| T)}}{T^{1/4}}$ now asymptotically equalizes both terms in the upper bound, giving us: 
    \[\E_{\pi_T} \left[L_\ell(p, \pi_T) \right] - L_\ell(q, \pi_T) \leq c'' k L \sqrt{\ln (k|\cQ|T)} T^{3/4},\]
    for some constant $c'' > 0$, thus implying our desired bound for all $L_\ell \in \cL_L, q \in \cQ$.
\end{proof}

Now, observe that our theorem critically uses the Lipschitzness of the coordinate-wise Bregman scores in the second argument throughout its domain $(0, 1)$. This property will indeed be satisfied by a great many Bregman divergences, including, for example, all those that, like the Euclidean distance, are based on a \emph{Lipschitz} convex function $f$. However, there also exist many useful Bregman divergences, including notably KL divergence, that are not Lipschitz but are \emph{locally} Lipschitz on $(0, 1)$. 

\begin{example}[Local Lipschitzness of KL divergence]
    Consider the $i$th component of the KL divergence, for any $i \in [k]$: $d^i_\mathrm{KL}(x, y) = x \log \frac{x}{y}$. Observe that $\frac{\partial}{\partial y} \left( x \log \frac{x}{y} \right) = -\frac{x}{y}$. Therefore, since in our setting $x \in [0, 1]$, we obtain that the KL divergence is $\frac{1}{\epsilon}$-Lipschitz in its second argument $y$ whenever $y \in (\epsilon, 1)$.
\end{example}
Based on this, we can easily formulate an updated version of Theorem~\ref{thm:omnipred} for KL-divergence:\footnote{We remark that it is straightforward to define a more general analog of Corollary~\ref{cor:log-loss}, promising to beat competing predictors according not just to log loss, but to any other locally Lipschitz loss $\ell$.}
\begin{corollary}[Beating the log loss of competing predictors] 
    \label{cor:log-loss}
    Given any competing predictor $q = (q_t)_{t \in [T]}$, our predictor $p = (p_t)_{t \in [T]}$ can achieve (expected) log loss no worse than that of $q$, up to a sublinear term:
    \[
    \E_{\pi_T}[L_\mathrm{Log}(p, \pi_T)] \leq L_\mathrm{Log}(q, \pi_T) + \Tilde{O} \left(k \cdot \frac{1}{\min_{t \in [T], i \in [k]} \min \{p_{t, i}, q_{t, i}\}} \cdot T^{3/4} \right) \quad \text{for any } T.
    \]
\end{corollary}
Since we have control over our predictor $p$, we can easily enforce that it doesn't have any very small components $p_{t,i}$. However, it would appear that we do not have control over the competing predictor $q$; thus, in theory, this bound could quickly deteriorate if the competing predictor continually predicts ever smaller values. However, this should not be a problem in practice: indeed, we can threshold any competing predictor to never predict less than a certain small probability value in each coordinate without significantly affecting downstream prediction set making, given that any ``reasonable'' prediction set algorithm would be extremely unlikely to include an exceedingly small-probability label in its prediction set.

\begin{remark}[Theorem~\ref{thm:omnipred} holds over any convex compact region $\cC \subset \R^k$, not just the simplex $\Delta \cY$]
    We want to highlight that our Best-in-Class Theorem~\ref{thm:omnipred} (along with natural locally-Lipschitz extensions in the spirit of~\ref{cor:log-loss}) holds for \emph{any} convex compact prediction space $\cC$, not just the simplex $\Delta \cY$ --- we only presented it as a result over the simplex since this is the context in which we are applying this result in this section.
\end{remark}

\begin{remark}[An Omniprediction view of Theorem~\ref{thm:omnipred}; The necessity of using Bregman scores as loss measures]
Our theorem provides an omniprediction-type statement, in the sense of~\cite{GopalanKRSW22}. That is, it is a statement of the form: ``Given a collection of competing predictors, and a family of well-behaved loss functions, it is possible to train a predictor that does (up to a small error term) at least as well as every competing predictor, simultaneously with respect to all loss functions in the family.'' 

Our results are stronger than the original omniprediction results in that (1) we generalize from binary valued outcomes to real valued outcomes, and (2) we generalize from one-dimensional outputs to multi-dimensional outputs, as primarily considered in \cite{GopalanKRSW22} and subsequent work. Our results also hold in the online adversarial setting, as in \cite{GJRR23}. Necessarily, the class of loss functions that we produce ``omnipredictors'' with respect to is more limited than in the one-dimensional binary case. One dimensional binary distributions are fully characterized by their mean, and so sufficiently (multi)-calibrated mean prediction is enough information to optimize all loss functions defined over one-dimensional binary values. Real valued distributions are not characterized by their means --- and so using tools based on multicalibration and unbiased prediction, we are necessarily limited to giving guarantees for loss functions such that the mean of the distribution is a sufficient statistic for optimization, which is what we do here.
Indeed, note that a direct converse to the result of~\cite{banerjee2005clustering} that we used in our proof above, proved under various technical assumptions in~\cite{banerjee2005optimality} and later in~\cite{abernethy12}, shows that \emph{any expected score} which elicits means must be a Bregman score; from this, it follows that our omniprediction result cannot hold more generally than over Bregman scores.
\end{remark}

\section{Conclusion and Discussion}
We have introduced a new and fundamental computational problem: making high-dimensional sequential predictions against an adversary in a way that is \emph{unbiased} conditional on an arbitrary collection of events $\cE$. Moreover we have shown that this problem is solvable in time polynomial in the dimension of the prediction space and the cardinality $|\cE|$ of the events, so long as the events can themselves be evaluated in polynomial time. This sheds light on the difficulty of (full) calibration in high dimensions: it is difficult not because the prediction space is exponentially large, but rather because it asks for consistency with respect to an exponentially large set of conditioning events.

Full calibration is useful because it makes it a dominant strategy, amongst all policies mapping predictions to actions, for a downstream decision maker to behave \emph{straightforwardly} --- i.e., to best respond to the predictions --- no matter what their utility function is. But when we have particular downstream decision problems in mind, strong guarantees for straightforward decision makers can be obtained by instead asking for predictions that are unbiased subject to only polynomially many events, tailored to the decision problem at hand --- which we can achieve algorithmically. We have given several very different applications of this technique: to obtaining new subsequence regret bounds in large action space settings such as online combinatorial optimization and extensive form games, and to obtaining forecasts that can be used to produce score-free prediction sets with transparent coverage guarantees. Both of these applications can be simultaneously tailored to multiple downstream decision makers and objective functions, allowing us to produce a single set of predictions that is simultaneously useful for many parties with differing goals. Along the way we have introduced ``best-in-class'' style performance guarantees to conformal prediction and uncertainty estimation, and have generalized omniprediction from one dimensional binary outcomes to high dimensional real valued outcomes. 

We expect that the technique we develop in this paper (both the concrete algorithms and the way of thinking about prediction for action) will find a broad set of further uses in sequential decision making and prediction problems and in game-theoretic settings involving many parties with differing objectives. This has already started to be realized: \cite{CRS23} have used the techniques developed in this paper to solve online prior-free principal-agent problems with computational and regret bounds that have an exponentially improved dependence on the size of the state space, compared to prior work that relied on full calibration \citep{camara2020mechanisms}. We hope and expect to see more applications along these and other lines.

\subsection*{Acknowledgements}
We are grateful to Edgar Dobriban, Amy Greenwald, Jason Hartline, Michael Jordan, Shuo Li, Jon Schneider, and Rakesh Vohra for insightful conversations at various stages of this work.

\bibliographystyle{plainnat}
\bibliography{refs}

\appendix

\section{Additional Details from Section \ref{sec:disjoint-minimax}}
\label{app:LP}
We solve for an $\epsilon$-approximate solution of linear program \ref{eq:minmax-reduced-lp} using a \emph{weak separation oracle}, using an approximate version of the Ellipsoid algorithm.
\begin{definition}
    For any $\epsilon>0$ and any convex set $S$, let 
    $$S^{+\epsilon} = \{p: ||p-q||_{2}\le \epsilon \quad \text{for some } q\in S\} \quad S^{-\epsilon} = \{x:B_{2}(x,\epsilon) \subseteq S\}$$
    be the positive and negative $\epsilon$-approximate sets of $S$, where $B_2(x,r)$ is a ball of radius $r$ under the $\ell_2$ norm.
\end{definition}
\begin{definition}
    A \emph{weak separation oracle} for a convex set $S$ is an algorithm that, when given input $\psi \in \mathbb Q^d$ and positive $\epsilon \in \mathbb Q$, confirms that $\psi \in S^{+\epsilon}$ if true, and otherwise returns a hyperplane $a\in \mathbb Q^d$ such that $||a||_\infty = 1$ and $\langle a, \psi \rangle \le \langle a,\psi' \rangle + \epsilon$ for all $\psi'\in S^{-\epsilon}$.
\end{definition}
We express a separation oracle for linear program \ref{eq:minmax-reduced-lp} as the convex program that solves for the most violated constraint given a candidate solution $\psi$, which is simply the best response problem for the maximization player in minimax problem \ref{eq:minmax-reduced}. This is the problem of maximizing a $d$-variable linear function over the convex set $\cC$. To make sure that we can control the bit complexity of the constraint returned by the separation oracle we round the coordinates of the constraint $a \in \R^d$ output by the separation oracle to a rational-valued vector within $\pm \frac{\epsilon}{2}$ of the exact solution by truncating each coordinate of $a$ to  $\log (\frac{1}{\epsilon})$ bits.
\begin{definition}
    A solution $\psi\in S^{+\epsilon}$ is \emph{$\epsilon$-weakly optimal} if, given $\epsilon > 0$, $\E_{p \sim \psi}[u(p,y)] \le \E_{p \sim \psi'}[u(p,y)] + \epsilon$ for all $\psi' \in S^{-\epsilon}$ and for all $y$.
\end{definition}

For an $\epsilon$-approximate solution to minimax problem \ref{eq:minmax-reduced}, it suffices to find an $\epsilon$-weakly optimal solution to linear program \ref{eq:minmax-reduced-lp}, which we can do using the Ellipsoid method. However, the solution to the weak optimization may not even be a valid probability distribution (since it only approximately satisfies the constraints) -- in this case, we can project our infeasible solution back to feasibility. We use the simplex Euclidean projection algorithm given by \cite{condat2016fast} to project the candidate solution back to a feasible region and show that this projected feasible solution is still $\epsilon$-approximately optimal.

\ellipsoid*

\begin{proof}
    Linear program $\ref{eq:minmax-reduced-lp}$ encodes minimax problem \ref{eq:minmax-reduced}. To solve LP \ref{eq:minmax-reduced-lp}, we use the Ellipsoid algorithm, which gives an approximate solution in polynomial time under the following conditions:
    \begin{theorem}[\cite{grotschel1988geometric}, Theorem 4.4.7]
       Given a weak separation oracle over convex constraint set $S$ and $\epsilon>0$, the Ellipsoid algorithm finds a $\epsilon$-weakly optimal solution over $S$ in time polynomial in the bit complexity of the constraints returned by the separation oracle, the bit complexity of the objective function, and the bit complexity of $\epsilon$.
    \end{theorem}
    Fix some $\epsilon >0$. Let $S$ be the constraint set, which are a set of linear constraints over a convex compact set (i.e. $y \in \cC$) and constraints enforcing a probability simplex (i.e. $\psi \in \Delta(\cP)$), implying that $S$ is a convex set.
    Let $\epsilon' = \frac{\epsilon}{2C\sqrt{|\cE|}}$. Given an exact separation oracle over $\cC$, preserving $\log(\frac{1}{\epsilon'})$ bits of the most violated constraint given by the separation oracle and rounding to a rational number yields an rational $\epsilon'$-approximate most violated constraint, which satisfies the conditions for a weak separation oracle. Thus, we can find an $\epsilon'$-weakly optimal solution $(\gamma',\psi')$ to minimax problem \ref{eq:minmax-reduced}, where $\psi' \in S^{+\epsilon}$. In the case that $\psi'$ is a valid probability distribution, we have found an $\epsilon$-approximate optimal solution $\psi_t^*=\psi'$. 
    
    Otherwise, $\psi'$ may violate conditions for a valid probability distribution if the linear constraints do not constrain the feasible set (i.e. $S=\Delta(P)$). Since $\psi' \in S^{+\epsilon}$, there exists some $\psi^\epsilon \in S$ such that $||\psi^\epsilon - \psi'|| \le \epsilon'$. We find this point $\psi^\epsilon$ via the simplex projection algorithm in \cite{condat2016fast}.
    
    We show that this projection back to a feasible probability distribution still leaves us with an $\epsilon$-approximately optimal solution.
    Let $u_t(p_t^*,y)$ be the $|\cE|$-dimensional vector such that each coordinate $E$ has entry $u_t(p_t^{*,E},y)$. First, we show that $|u_t(p_t^{*, E},y)|$ is bounded by $C=2\max_{y\in \cC} ||y||_\infty$ for $E \in \cE$:
    \begin{align*}
        |u(p_t^{*,E},y)| &\leq \sum_{i=1}^d\sum_{\sigma \in \{-1,1\}}\sum_{E \in \cE} q_{t,(i,\sigma,E)}\cdot |\sigma| \cdot E(x_t,p_t^{*,E})\cdot |p^{*,E}_{t,i} - y_{i}|\\
        & \le \sum_{i=1}^d\sum_{\sigma \in \{-1,1\}}\sum_{E \in \cE} q_{t,(i,\sigma,E)} \cdot |p^{*,E}_{t,i} - y_{i}|\\
        & \le \sum_{i=1}^d\sum_{\sigma \in \{-1,1\}}\sum_{E \in \cE} q_{t,(i,\sigma,E)} \cdot C \\
        & = C
    \end{align*}
    where we used that $E(x_t,p^{*,E}_{t})\le 1$ and $|\sigma| = 1$, and that $q\in \Delta(2d|\cE|)$, implying it must sum to 1. From this, we find that $||u_t(p_t^{*},y)||_2 \le \sqrt{C^2_1 + \ldots + C^2_{|\cE|}}\le C\sqrt{|\cE|}$.
    
    Next, by continuity of inner product, given $\epsilon>0, y \in \cC$, there exists $\delta>0$ such that $||\psi^\epsilon - \psi'|| \le \delta$ implies that $||\E_{p\sim \psi^\epsilon}[u_t(p,y)]-\E_{p\sim \psi'}[u_t(p,y)]|| \le \epsilon$. By Cauchy-Schwarz, we can bound the difference between the expectations as follows:
    \begin{align*}
    ||\E_{p\sim \psi^\epsilon}[u_t(p,y)]-\E_{p\sim \psi'}[u_t(p,y)]||_2 &= \langle \psi^\epsilon - \psi^*, u_t(p_t^*,y) \rangle\\
    &\le ||\psi_t^\epsilon - \psi'||_2 \cdot ||u_t(p^*_t,y)||_2 \\
    &\le \delta \cdot ||u_t(p_t^*,y)||_2 \\
    &\le \delta \cdot C\sqrt{|\cE|}.
    \end{align*}
    Thus, using $\psi_t^*=\psi^\epsilon$ as the solution and setting $\delta=\epsilon'$ gives us an solution that is $\epsilon'\cdot C\sqrt{|\cE|} + \epsilon' = \frac{\epsilon}{2} + \frac{\epsilon}{2C\sqrt{|\cE|}} \le \epsilon$ approximate.
    By Lemma \ref{lem:2implies1}, any optimal solution to minimax problem \ref{eq:minmax-reduced} is an optimal solution to minimax problem \ref{eq:minmax}, so we must have that $\psi_t^*$ is an $\epsilon$-approximate solution to minimax problem $\ref{eq:minmax}$.
    
    Now we consider the runtime of the algorithm. In order for LP \ref{eq:minmax-reduced-lp} to be well-formulated, we first solve $|\cE|$ convex programs (one for each $p^{*,E})$, which takes time polynomial in $d$. Now, consider the bit complexity of the constraints. For the inequality constraints, the bit complexity of each constraint bounding the objective function is given by the bit complexity of $\E_{p\sim \psi}[u(p, y)]$.  Each coefficient of $\psi(p_t^{*,E})$ is $u_t(p_t^{*,E},y)$, which is bounded by $C$ from above. 
    Since there are $|\cE|$ variables in this constraint, the maximum bit complexity of any constraint is bounded by $O(\log (C|\cE|))$. Similarly, the objective function has polynomial bit complexity on the scale of $O(\log(C|\cE|))$. Finally, $\epsilon$ has a bit complexity of $\log(\frac{1}{\epsilon})$. The simplex projection algorithm has quadratic runtime in the dimension of the vector, which takes $O(|\cE|^2)$ time.
    
    Thus, the runtime of the algorithm is polynomial in $d, |\cE|,\log(C|\cE|),$ and $\log(\frac{1}{\epsilon})$.
\end{proof}

\section{Proofs from Section \ref{sec:disjoint-minimax}}
\label{app:proofs-disjoint}

\lemeventoptimal*
\begin{proof}
     The constraint that $E(\pi_{t-1},x_t,p) = 1$ together with the fact that the set of events $\cE$ is disjoint and binary implies that for all other events $E' \in \cE$, $E'(\pi_{t-1},x_t,p)=0$. For any $p$ such that  $E(\pi_{t-1},x_t,p) = 1$, we therefore have that $u_t(p,y)$ reduces to:
    $$u_t(p,y)=\left(\sum_{i=1}^d\sum_{\sigma \in \{-1,1\}} q_{t,(i,\sigma,E)}\cdot\sigma \cdot p_{i} -   q_{t,(i,\sigma,E)}\cdot\sigma \cdot y_{i}\right) $$
    But in this expression, the $p$ terms have no interaction with the $y$ terms, and hence we have that for any $y$:
    \begin{eqnarray*}
    \argmin_{p :E(\pi_{t-1},x_t,p) = 1} u_t(p,y)  &=& \argmin_{p :E(\pi_{t-1},x_t,p) = 1} \left(\sum_{i=1}^d\sum_{\sigma \in \{-1,1\}} q_{t,(i,\sigma,E)}\cdot\sigma \cdot p_{i} -   q_{t,(i,\sigma,E)}\cdot\sigma \cdot y_{i}\right) \\
    &=&\argmin_{p :E(\pi_{t-1},x_t,p) = 1} \left(\sum_{i=1}^d\sum_{\sigma \in \{-1,1\}} q_{t,(i,\sigma,E)}\cdot\sigma \cdot p_{i} \right) \\
    &=& p^{*,E}_{t}
    \end{eqnarray*}
\end{proof}

\lemtwoimpliesone*
\begin{proof}
    We first observe that minimax problem \ref{eq:minmax-reduced} is only a more constrained problem for the minimization player than minimax problem \ref{eq:minmax}, as $\cP_t \subset \cC$. Thus it suffices to show that given a solution $\hat \psi_t$ for minimax problem \ref{eq:minmax}, we can transform it into a new solution $\psi_t$ such that:
    \begin{enumerate}
        \item $\psi_t$ has support only over points in $\cP_t$, and
        \item For all $y \in \cC$, $\E_{p_t \sim \psi_t}[u(p_t,y)] \geq \E_{p_t \sim \hat \psi_t}[u(p_t,y)]$.
    \end{enumerate}
    Given $\hat \psi_t$, we construct $\psi_t$ as follows: for each event $E$, we take all of the weight that $\hat \psi_t$ places on points $p$ such that $E(\pi_{t-1}, x_t, p) = 1$, and place that weight on $p_t^{*,E} \in \cP_t$:
    $$\psi_t(p_t^{*,E}) = \hat \psi_t(\{ p : E(\pi_{t-1}, x_t, p) = 1\})$$
    By construction $\psi_t$ has support over points in $\cP_t$. It remains to show that $\psi_t$ has objective value that is at least as high as $\hat \psi_t$ for every $y \in \cC$:
    \begin{eqnarray*}
        \E_{p_t \sim \hat \psi_t}[u(p_t,y)] &=& \sum_{E \in \cE} \Pr_{p_t \sim \hat \psi_t}[E(\pi_{t-1},x_t,p_t) = 1]  \E_{p_t \sim \hat \psi_t}[u(p_t,y) | E(\pi_{t-1},x_t,p_t) = 1] \\
        &\leq& \sum_{E \in \cE} \Pr_{p_t \sim \hat \psi_t}[E(\pi_{t-1},x_t,p_t) = 1] u(p_t^{*,E},y) \\
        &=&  \E_{p_t \sim  \psi_t}[u(p_t,y)] 
    \end{eqnarray*}
    The inequality follows from Lemma \ref{lem:eventoptimal}.
\end{proof}

\section{Proofs From Section \ref{sec:general-minimax}}
\label{app:general-minimax}

\lemFTPLConstants*

\begin{proof}
Recall from Theorem \ref{thm:kalai-vempala} that regret at round $T'$ for any decision $y \in \cC$ is bounded by:
\begin{equation*}
    R_{T', y} \leq 2\sqrt{\Delta K A T'}
\end{equation*}
when we choose $\delta = \sqrt{\frac{\Delta}{KAT'}}$, where $K = \sup_{y, p \in \cC} |\langle y, s_t(p)\rangle|$, $A = \sup_{p \in \cC} \lVert s_t(p) \rVert_1$ and $\Delta = \sup_{y_1, y_2 \in \cC}\lVert y_1 - y_2 \rVert_1$. \\ \\
From the definition of state vectors in Equation \ref{eq:FTLP-state}, for any $p \in \cC$, we have that 
\begin{align*}
    \lVert s_{t}(p)\rVert_1 = \sum_{i = 1}^{d} |s_{t, i} (p)|
    &= \sum_{i = 1}^{d} \left|\sum_{\sigma \in \{-1,1\}}\sum_{E \in \cE} q_{t,(i,\sigma,E)}\cdot\sigma \cdot {E(x_t,p)}\right| \\
    &\leq \sum_{i = 1}^{d} \sum_{\sigma \in \{-1,1\}}\sum_{E \in \cE} \left|q_{t,(i,\sigma,E)}\cdot\sigma \cdot {E(x_t,p)}\right| \\
    &= \sum_{i = 1}^{d} \sum_{\sigma \in \{-1,1\}}\sum_{E \in \cE} q_{t,(i,\sigma,E)} \cdot\left|\sigma \cdot {E(x_t,p)}\right| \\
    &\leq \sum_{i = 1}^{d} \sum_{\sigma \in \{-1,1\}}\sum_{E \in \cE} q_{t,(i,\sigma,E)} = 1
\end{align*}
since $q_t \in \Delta [2d|\cE|]$ and by definition, $\left|\sigma \cdot {E(x_t,p)}\right| \leq 1$. Thus, we can set $A = 1$.  Since $\lVert y \rVert_{\infty} \leq C/2$, we similarly have:
\begin{equation*}
    |\langle y, s_t(p)\rangle| \leq \sum_{i = 1}^d \left|y_i \cdot s_{t,i}(p)\right|  \leq \frac{C}{2} \cdot \sum_{i=1}^d  \left|\sum_{\sigma \in \{-1,1\}}\sum_{E \in \cE} q_{t,(i,\sigma,E)}\cdot\sigma \cdot {E(x_t,p)}\right| \leq \frac{C}{2}
\end{equation*}
for all $y, p \in \cC$, so $K = C/2$. Finally, since $|y_{i, 1} - y_{i, 2}| \leq C$ for all $y_1, y_2 \in \cC$, it is clear that $\Delta \leq dC$.

Note that for this setting, we have indeed set $\delta$ to the desired value, since
\begin{equation*}
    \sqrt{\frac{\Delta}{KAT'}} = \sqrt{\frac{2dC}{CT'}} = \sqrt{\frac{2d}{T'}} 
\end{equation*}
and now substituting the relevant quantities into the regret bound,
\begin{align*}
    R_{T', y} &\leq 2\sqrt{dC \cdot C /2 \cdot 1 \cdot T'} = \sqrt{2dC^2 T'}
\end{align*}
\end{proof}

Algorithm \ref{alg:FTPL-No-Sampling} for achieving unbiased predictions makes the implicit assumption that the expected value of FTPL's output is efficiently computable at each round. In conditions where this is not true, we can approximate the expectation via a sampling procedure over the distribution with respect to which we are taking the expectation, and use this approximation to achieve similar results. We make use of the following concentration inequalities:
\begin{theorem}[Hoeffding's Inequality]
\label{thm:hoeffding}
Let $\{X_j\}_{j=1}^n$ be independent random variables such that $X_j \in [a, b]$ for each $j \in [n]$. Let $S_n = \sum_{j=1}^n X_j$ be their sum. Then, for any $\epsilon > 0$,
\begin{equation*}
    \Pr_{}\left[{|S_n - \E[S_n]| \geq \epsilon}\right] \leq 2\exp\left(-\frac{2\epsilon^2}{n(b-a)^2}\right)
\end{equation*}
\end{theorem}
\begin{corollary}
    \label{cor:hoeffding-corollary}
Let $\{X_j\}_{j=1}^n$ be independent random variables such that $X_j \in [a, b]$ for each $j \in [n]$. Let $\overline{X} = \frac{1}{n} \sum_{j=1}^n X_j$ be the average of their sum. Then, for any $\epsilon > 0$,
\begin{equation*}
    \Pr_{}\left[{|\overline{X} - \E[\overline{X}]| \geq \epsilon}\right] \leq 2\exp\left(-\frac{2n\epsilon^2}{(b-a)^2}\right)
\end{equation*}
\end{corollary}
\begin{lemma}
\label{lem:sampling-bound}
    Given a sample $\{y_\tau^j\}_{j = 1}^n \sim D_\tau^n$, define $p_{\tau} = \sum_{j=1}^n y_\tau^j$ as the average of the sampled values. For $n = \frac{C^2}{2\epsilon_0^2} \ln\left(\frac{2d}{\delta_0}\right)$, we have that with probability $1 - \delta_0$
    \begin{equation*}
        \left| \\E_{y \sim D_\tau}[y_i] - p_{\tau, i} \right| < \epsilon_0 
    \end{equation*}
    simultaneously for all $i \in [d]$.
\end{lemma}
\begin{proof}
    For each $i \in [d]$, the set of values $\{y^j_{\tau, i}\}_{j=1}^n$ are independent draws from the same distribution such that $y^j_{\tau, i} \in [-C/2, C/2]$ for all $i \in [d], j \in [n]$. By direct application of Corollary \ref{cor:hoeffding-corollary},
    \begin{equation*}
        Pr_{}\left[|\E[y_{i}] - p_{\tau, i}| \geq \epsilon_0 \right] \leq  2\exp\left(-\frac{2n\epsilon_0^2}{C^2}\right)
    \end{equation*}
Taking a union bound,
\begin{align*}
    &\Pr_{}\left[|\E[y_{1}] - p_{\tau, 1}| \geq \epsilon_0 ~ \bigcup ~ |\E[y_{2}] - p_{\tau, 2}| \geq \epsilon_0 ~ \bigcup \cdots \bigcup ~ |\E[y_{d}] - p_{\tau, d}| \geq \epsilon_0 \right] \leq 2d\exp\left(-\frac{2n\epsilon_0^2}{C^2}\right) \\
    \implies&\Pr_{}\left[|\E[y_{1}] - p_{\tau, 1}| < \epsilon_0 ~ \bigcap ~ |\E[y_{2}] - p_{\tau, 2}| < \epsilon_0 ~ \bigcap \cdots \bigcap ~ |\E[y_{d}] - p_{\tau, d}| < \epsilon_0 \right] \geq 1 - 2d\exp\left(-\frac{2n\epsilon_0^2}{C^2}\right)
\end{align*}
With our choice of $n$, $ 2d\exp\left(-\frac{2n\epsilon_0^2}{C^2}\right)$ simplifies to $\delta_0$, completing the proof. 
\end{proof}
Lemma \ref{lem:sampling-bound} details a sampling procedure to select $p_\tau$ that is close to $ \E_{y_\tau \sim D_\tau}[y_\tau]$ with high probability. We can use this to slightly modify Algorithm \ref{alg:FTPL-No-Sampling} to return a strategy for the learner with similar guarantees, but which doesn't make any assumptions on our ability to compute $\E_{y \sim D_\tau}[y]$. This is captured by the following Algorithm~\ref{alg:FTPL-Sampling}:

\begin{algorithm}
\begin{algorithmic}
\STATE Initialize $s_t(p_0)$ to 0. \\
\STATE Set $T' = \frac{8dC^2}{\epsilon'^2}$, $\delta = \sqrt{\frac{2d}{T'}}, n = \frac{2C^2}{\epsilon'^2}\ln\left(\frac{2dT'}{\delta'}\right)$
\STATE Fix distribution $\cZ = \text{Unif}[0,\frac{1}{\delta}]^d$.
\FOR{${\tau}=1, \dots, T'$}
\FOR{${j}=1, \dots, n$}
\STATE Sample $z_\tau^j \sim \cZ$ \\
\STATE Compute $y_\tau^j = M\left(\sum_{k = 0}^{{\tau}-1} s_{t}(p_k) + z_\tau^j\right)$ \\
\ENDFOR
\STATE Set $p_{\tau} = \frac{1}{n} \sum_{j=1}^n y_\tau^j$ \\
\STATE Compute $s_{t}(p_{\tau})$ as defined in Equation \ref{eq:FTLP-state}. \\
\ENDFOR
\STATE Define $\overline{p}$ as the uniform distribution over the sequence $(p_1, p_2, \cdots , p_{T'})$.
\RETURN $\overline{p}$
\end{algorithmic}
\caption{\texttt{Get-Approx-Equilibrium-Sampling($t, \epsilon' ,\delta'$)}}
\label{alg:FTPL-Sampling}
\end{algorithm}

\begin{restatable}{theorem}{FTPLSampling}
\label{thm:FTPL-Sampling}
For any $t \in [T]$ and $\epsilon', \delta' > 0$, with probability $1-\delta'$, Algorithm \ref{alg:FTPL-Sampling} returns a distribution over actions $\overline{p}$ which is an $\epsilon'$-approximate minimax equilibrium strategy for the zero-sum game with objective $u_t$. 
\end{restatable}

The proof of Theorem \ref{thm:FTPL-Sampling} is similar to the proof of Theorem \ref{thm:FTPL-No-Sampling}. The primary difference is that we must now control and keep track of sampling error that comes from estimating $\E_{z \in \cZ}\left[ M\left(\sum_{k = 0}^{{\tau}-1} s_{t}(p_k) + z \right)\right]$ from samples, which we do with standard concentration inequalities.


\begin{proof}
We bound the adversary's maximum sum of scores in hindsight with respect to $u_t$:
\begin{equation}
\label{eq:no-sampling-bound}
    \max_{y \in \cC} \sum_{\tau = 1}^{T'} u_t(p_{\tau}, y) \leq \left(R_{T', y^*}  + \sum_{\tau=1}^{T'} \E_{y \sim D_{\tau}}\left[u_t'(p_\tau, y)\right] \right) + \sum_{\tau = 1}^{T'}\langle p_{\tau}, s_t(p_{\tau}) \rangle 
\end{equation}
where 
\begin{equation*}
    y^* =  \argmax_{y \in \cC} \sum_{\tau = 1}^{T'} u_t'(p_{\tau}, y^*) =  \argmax_{y \in \cC} \sum_{\tau = 1}^{T'} u_t(p_{\tau}, y).
\end{equation*}
Now, we choose $p_{\tau}$ as the average of $n = \frac{2C^2}{\epsilon'^2}\ln\left(\frac{2dT'}{\delta'}\right)$ independently drawn samples from $D_\tau$, which ensures we satisfy the guarantee detailed in Lemma \ref{lem:sampling-bound}. By our choice of $n$, we have $\epsilon_0 = \epsilon' / 2$ and $\delta_0 = \delta' / T'$. Since now  $p_\tau$ is an estimate of $\E_{y \sim D_\tau}[y]$ rather than the exact expectation, for any specific $\tau \in [T']$ we can only bound the expected utility with high probability:
\begin{align*}
    \E_{y \sim D_{\tau}}\left[u_t(p_{\tau}, y)\right] = \left\langle \left(\E_{y \sim D_{\tau}}[y] - p_{\tau}\right), s_t(p_{\tau}) \right\rangle  &\leq \left| \left\langle \left(\E_{y \sim D_{\tau}}[y] - p_{\tau}\right), s_t(p_{\tau}) \right\rangle \right| \\
    &\leq \sum_{i=1}^d \left| (\E_{y \sim D_\tau}[y_i] - p_{\tau,i}) \cdot s_{t,i}(p_\tau) \right| \\
    &\leq \epsilon_0 \cdot \sum_{i=1}^d  \left|\sum_{\sigma \in \{-1,1\}}\sum_{E \in \cE} q_{t,(i,\sigma,E)}\cdot\sigma \cdot {E(x_t,p_\tau)}\right| \\
    &\leq \epsilon_0
\end{align*}
where the second last inequality is true with probability $1 - \delta_0$, from Lemma \ref{lem:sampling-bound}. We can quantify the probability that this inequality holds for all $\tau \in [T']$ simultaneously using a union bound:
\begin{align*}
    &\Pr_{}\left[\bigcup_{\tau=1}^{T'}  \E_{y \sim D_{\tau}}\left[u_t(p_{\tau}, y)\right] > \epsilon_0 \right] \leq T' \delta_0 \\
    \implies&\Pr_{}\left[\bigcap_{\tau=1}^{T'}  \E_{y \sim D_{\tau}}\left[u_t(p_{\tau}, y)\right] \leq \epsilon_0 \right] \geq 1 - T'\delta_0 = 1 - \delta'
\end{align*}
So with probability $1 - \delta'$, the sum of expected utilities can be bound, once again following Theorem \ref{thm:FTPL-No-Sampling},
\begin{align*}
     &\sum_{\tau=1}^{T'} \E_{y \sim D_{\tau}} \left[u_t'(p_{\tau}, y)\right] + \sum_{\tau=1}^{T'} \langle p_{\tau}, s_t(p_{\tau}) \rangle = \sum_{\tau=1}^{T'} \E_{y \sim D_{\tau}} \left[u_t(p_{\tau}, y)\right] \leq \epsilon_0 T' = \epsilon' T' / 2 \\
     \implies &\sum_{\tau=1}^{T'} \E_{y \sim D_{\tau}} \left[u_t'(p_{\tau}, y)\right] \leq \epsilon' T' / 2 - \sum_{\tau=1}^{T'} \langle p_{\tau}, s_t(p_{\tau}) \rangle
\end{align*}
Substituting this into equation \ref{eq:no-sampling-bound},
\begin{align*}
    \max_{y \in \cC} \sum_{\tau = 1}^{T'} u_t(p_{\tau}, y) &\leq \left(R_{T', y^*} + \epsilon' T' /2 -\sum_{\tau = 1}^{T'}\langle p_{\tau}, s_t(p_{\tau}) \rangle \right) + \sum_{\tau = 1}^{T'}\langle p_{\tau}, s_t(p_{\tau}) \rangle  \\
    &= R_{T', y^*} + \epsilon' T' / 2\\
    &\leq  \sqrt{2dC^2 T'} + \epsilon' T' /2
\end{align*}
Since $\frac{1}{T'} \sum_{i=1}^{T'}  u({p_{\tau}, y}) = \E_{p \sim \overline{p}} [u_t(p, y)]$, it follows that
\begin{equation*}
     \max_{y \in \cC} \cdot \E_{p \sim \overline{p}} [u_t(p, y)] \leq \frac{ \sqrt{2dC^2 T'} + \epsilon' T' /2}{T'}
\end{equation*}
and so with probability $1 - \delta'$, $\overline{p}$ is a $\left(\sqrt{\frac{2dC^2}{T'}} + \frac{\epsilon'}{2}\right)$-approximate equilibrium strategy for the learner. By our choice of $T' = \frac{8dC^2}{\epsilon'^2}$, this simplifies to $\epsilon'$, as desired. 
\end{proof}

\section{Proofs from Section \ref{sec:generic-regret}}
\label{app:genericregret}

\typeregret*
\begin{proof}[Proof of Theorem \ref{thm:type}]
    Fix any pair $u, u' \in \cU_L$. We need to upper bound $r(\pi_T,u,u')$. Using linearity of both $u$ in its second argument, we can derive:
    \begin{eqnarray*}
        \frac{1}{T}\sum_{t=1}^T u(\delta_u(p_t),y_t) &=& \sum_{a \in \cA} \frac{1}{T}\sum_{t: \delta_u(p_t) =a}u(a,y_t) \\
        &=& \sum_{a \in \cA} \frac{1}{T} \sum_{t=1}^T E_{u,a}(p_t) u(a,y_t) \\
        &=& \sum_{a \in \cA} u\left(a,\frac{1}{T} \sum_{t=1}^T E_{u,a}(p_t) y_t\right) \\
        &\geq& \sum_{a \in \cA}\left( u\left(a,\frac{1}{T} \sum_{t=1}^T E_{u,a}(p_t) p_t\right) - \frac{L\alpha(n_T(E_{u,a},\pi_T)) }{T} \right)\\
        &=&  \frac{1}{T}\sum_{t=1}^T u(\delta_u(p_t),p_t)- \frac{L\sum_{a \in \cA}\alpha(n_T(E_{u,a},\pi_T))}{T} 
    \end{eqnarray*}
where the inequality follows from the $\alpha$-unbiasedness condition with respect to event $E_{u,a}$ and the $L$-Lipschitzness of $u$. Similarly, because of the $\alpha$-unbiasedness condition with respect to event $E_{u',a}$ we have that:
$$\frac{1}{T}\sum_{t=1}^T u(\delta_{u'}(p_t),y_t) \leq \frac{1}{T}\sum_{t=1}^T u(\delta_{u'}(p_t),p_t) + \frac{L \sum_{a \in \cA} \alpha(n_T(E_{u',a},\pi_T))}{T} $$

We can therefore bound the regret $r(\pi_T,u,u')$ as:
\begin{eqnarray*}
    r(\pi_T,u,u') &=& \frac{1}{T}\sum_{t=1}^T u(\delta_{u'}(p_t),y_t) - \frac{1}{T}\sum_{t=1}^T u(\delta_u(p_t),y_t) \\
    &\leq& \frac{1}{T} \sum_{t=1}^T\left(u(\delta_{u'}(p_t),p_t) -  u(\delta_u(p_t),p_t)\right) +  \frac{L \sum_{a \in \cA}\left( \alpha(n_T(E_{u',a},\pi_T))+\alpha(n_T(E_{u,a},\pi_T))\right)}{T} \\
    &\leq&  \frac{L \sum_{a \in \cA}\left( \alpha(n_T(E_{u',a},\pi_T))+\alpha(n_T(E_{u,a},\pi_T))\right)}{T} 
\end{eqnarray*}
Here the last inequality follows from the fact that by definition of $\delta_u(p_t)$, $u(\delta_u(p_t),p_t) \geq u(a', p_t)$ for all $a' \in \cA$ (including $a' = \delta_{u'}(p_t)$). 

The simplified bound follows from the concavity of $\alpha$ and the fact that for every utility function $u$, the events $\{E_{u,a}\}_{a \in \cA}$ are disjoint, which implies that $\sum_{a \in \cA}n_T(E_{u,a},\pi_T) \leq T$.
\end{proof}

\section{Proofs From Section \ref{sec:warmup}}
\groupswap*
\begin{proof}[Proof of Theorem \ref{thm:group-swap}]
    Fix any $\phi:\cA\rightarrow \cA$, $u \in \cU_L$, and $G \in \cG$. We need to upper bound $r(\pi_T,u,\phi,G)$. Using the linearity of $u(a,\cdot)$ in its second argument for all $a \in \cA$, we can write:
    \begin{eqnarray*}
r(\pi_T,u,\phi,G) &=& \frac{1}{T_G(\pi_T)}\sum_{t : x_t \in G} u(\phi(\delta_u(p_t)),y_t)-u(\delta_u(p_t),y_t) \\
&=& \frac{1}{T_G(\pi_T)} \sum_{a \in \cA} \sum_{t : \delta_u(p_t) = a,x_t \in G} u(\phi(a),y_t) - u(a,y_t) \\
&=& \sum_{a \in \cA}\frac{1}{T_G(\pi_T)}\sum_{t=1}^T E_{u,a,G}(x_t,p_t)\left(u(\phi(a),y_t) - u(a,y_t)\right) \\
&=&\sum_{a \in \cA} \left(u\left(\phi(a),\frac{1}{T_G(\pi_T)}\sum_{t=1}^T E_{u,a,G}(x_t,p_t) y_t\right) - u\left(a,\frac{1}{T_G(\pi_T)}\sum_{t=1}^T E_{u,a,G}(x_t,p_t) y_t\right)\right) \\
&\leq& \sum_{a \in \cA} ( u\left(\phi(a),\frac{1}{T_G(\pi_T)}\sum_{t=1}^T E_{u,a,G}(x_t,p_t) p_t\right) - u\left(a,\frac{1}{T_G(\pi_T)}\sum_{t=1}^T E_{u,a,G}(x_t,p_t) p_t\right) \\ 
&+& \frac{2L\alpha(n_T(E_{u,a,G},\pi_T)) }{T_G(\pi_T)} ) \\
&\leq& \sum_{a \in \cA} \left( \frac{2L\alpha(n_T(E_{u,a,G},\pi_T))}{T_G(\pi_T)}\right) \\
&=& \frac{2L \sum_{a \in \cA}\alpha(n_T(E_{u,a},\pi_T))  }{T_G(\pi_T)}
    \end{eqnarray*}
Here the first inequality follows from the $\alpha$-unbiasedness condition and the $L$-Lipschitzness of $u$: indeed, for every $a'$ (and in particular for $a' \in \{a, \phi(a)\}$) we have 
\[\left|u\left(a',\frac{1}{T_G(\pi_T)}\sum_{t=1}^T E_{u,a,G}(x_t,p_t) p_t\right) - u\left(a',\frac{1}{T_G(\pi_T)}\sum_{t=1}^T E_{u,a,G}(x_t,p_t) y_t\right)\right| 
\]
\[
\leq L|\frac{1}{T_G(\pi_T)} \sum_t (p_t - y_t) E_{u, a,G}(p_t)|_\infty = \frac{L}{T_G(\pi_T)} \max_{i \in [d]} |\sum_t (p_{t, i} - y_{t, i}) E_{u, a,G}(p_t)| \leq \frac{\alpha(n_T(E_{u,a,G},\pi_T)) L}{T_G(\pi_T)}. 
\]
The 2nd inequality follows from the fact that by definition, whenever $\delta_u(p_t) = a$ (and hence whenever $E_{u,a}(p_t) = 1$), $u(a,p_t) \geq u(a',p_t)$ for all $a' \in \cA$, and the fact that by Lemma \ref{lem:convexdelta}, the levelsets of $\delta_u$ are convex, and hence $u(a,\frac{1}{T_G(\pi_T)}\sum_{t=1}^T E_{u,a,G}(x_t,p_t)p_t) \geq u(a',\frac{1}{T_G(\pi_T)}\sum_{t=1}^T E_{u,a,G}(x_t,p_t)p_t) $ for all $a'$.

For any utility function $u$ and group $G$, the events $\{E_{u,a,G}\}_{a \in \cA}$ are disjoint, and so for any $u$, $G$, and any $\pi_T$, $\sum_{a \in \cA}n_T(E_{u,a,G},\pi_T) \leq T_G(\pi_T)$. Therefore, whenever $\alpha$ is a concave function (as it is for the algorithm we give in this paper, and as it is for essentially any reasonable bound), the term $\sum_{a \in \cA}\alpha(n_T(E_{u,a,G},\pi_T))$ evaluates to at most $K\alpha(T_G(\pi_T)/K)$.
\end{proof}

\section{Additional Details from Section \ref{sec: subsequence}}
\label{app: efg}
\subsection{A Best-Response Oracle for Leaf-Form Strategies}
\label{subsec:best-response}
In Section \ref{sec: subsequence}, we reduced the problem of obtaining subsequence regret in an extensive-form game to the problem of optimizing over the set of all leaf-form strategies for that game, to find the learner's best response to any collection of opponent strategies. Here, we present a best-response oracle that efficiently returns a learner's best response for the subset of games for which the learner has \textit{perfect recall} and \textit{path recall}. While perfect recall is a well-studied and standard property in the literature, we introduce path recall as a new but related property specific to the requirements of our best-response oracle.
\begin{definition} 
    A player $i$ has \emph{perfect recall} in an extensive-form game $G$ if, for any information set $I \in \cI_i$ and any two nodes $x, y \in I$, for any node $x' \in I' \in \cI_i$ that is a predecessor of $x$, there is also a node $y' \in I'$ that precedes $y$, such that the action played at $x'$ on the path to $x$ and the action played at $y'$ on the path to $y$ are the same. If all players have perfect recall, then the game $G$ itself is said to have perfect recall. 
\end{definition}

Intuitively, a player has perfect recall if the partition of their nodes into information sets does not lose them any information about their past actions, i.e. the full history of a player's own actions is deducible from their current information set.

\begin{definition} 
    A player $i$ has $\emph{path recall}$ in an extensive-form game $G$ if, for any information set $I \in \cI_i$ and any two nodes $x, y \in I$,
    \begin{itemize}
        \item For any node $x' \in I' \in \cI_i$ that is a predecessor of $x$, there is also a node $y' \in I'$ that precedes $y$.
        \item For any node $h$ such that $\rho(h) \neq i$, if $h$ is a predecessor of $x$, then $h$ is a predecessor of $y$. 
    \end{itemize}
\end{definition}

Intuitively, a player has path recall if the partition of their nodes into information sets does not lose them any information about the past sequence of their own information sets and opponents' nodes along the path to any information set, i.e. the full history of the player's information sets and opponents' nodes reached so far is deducible from their current information set. Thus, perfect recall allows us to reason about the past in terms of one's own actions, while path recall allows us to reason about the past in terms of the sequence of learner's information sets and opponents' nodes passed so far. A game can satisfy perfect recall and not path recall, and vice versa (see Figure \ref{fig:efg-1}). In conjunction, however, these properties imply a useful structure to the game which we can exploit.




\tikzset{ 
solid node/.style={circle,draw,inner sep=1.5,fill=black}, 
hollow node/.style={circle,draw,inner sep=1.5}, 
no node/.style={circle,draw,inner sep=0} 
} 
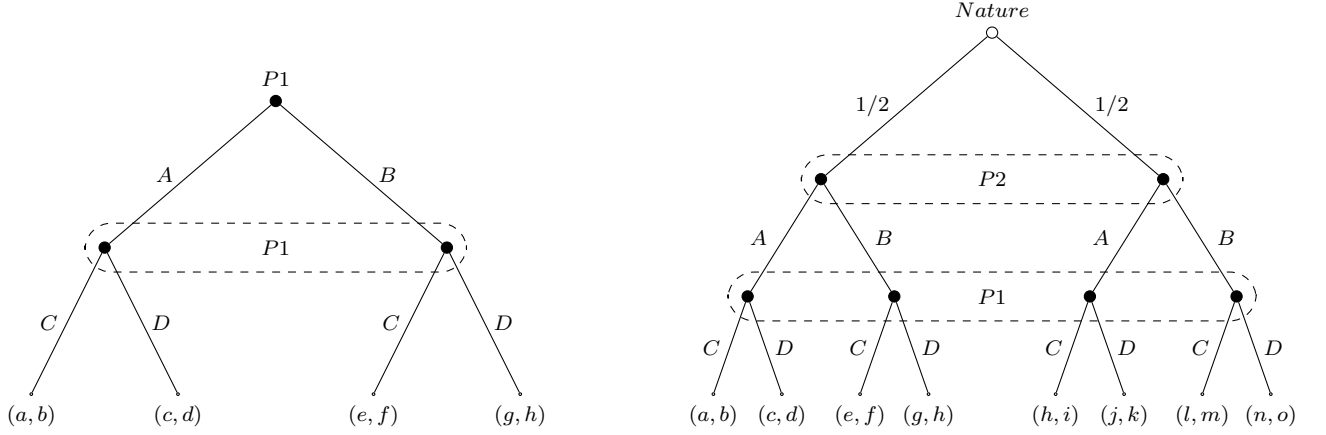
\begin{figure}[t!]
    \centering
    \begin{minipage}[t]{0.45\textwidth}
        \centering
    \begin{tikzpicture}[scale=1.3,font=\footnotesize] 
    \tikzstyle{level 1}=[level distance=15mm,sibling distance=35mm] 
    \tikzstyle{level 2}=[level distance=15mm,sibling distance=15mm]
    \node(0)[solid node,label=above:{$P1$}]{} 
        child{node(1)[solid node]{}
            child{node[no node,label=below:{$(a,b)$}]{}edge from parent node[left]{$C$}} child{node[no node,label=below:{$(c,d)$}]{}edge from parent node[right]{$D$}} 
        edge from parent node[left,xshift=-3]{$A$} 
    } 
        child{node(2)[solid node]{} 
            child{node[no node,label=below:{$(e,f)$}]{}edge from parent node[left]{$C$}} child{node[no node,label=below:{$(g,h)$}]{}edge from parent node[right]{$D$}}
        edge from parent node[right,xshift=3]{$B$} 
    };
    \draw[dashed,rounded corners=10]($(1)+(-.2,.25)$)rectangle($(2)+(.2,-.25)$);
    \node at ($(1)!.5!(2)$) {$P1$}; 
    \end{tikzpicture}
    \end{minipage}\hfill
    \begin{minipage}[t]{0.45\textwidth}
        \centering
\begin{tikzpicture}[scale=1.3,font=\footnotesize] 
    \tikzstyle{level 1}=[level distance=15mm,sibling distance=35mm] 
    \tikzstyle{level 2}=[level distance=12mm,sibling distance=15mm]
    \tikzstyle{level 3}=[level distance=10mm,sibling distance=7mm]
    \node(0)[hollow node,label=above:{$Nature$}]{} 
        child{node(1)[solid node]{}
            child{node(3)[solid node]{}
                child{node[no node, label=below:{$(a,b)$}]{}edge from parent node[left]{$C$}} 
                child{node[no node, label=below:{$(c,d)$}]{}edge from parent node[right]{$D$}} 
            edge from parent node[left, xshift=-3]{$A$}
            }
            child{node(4)[solid node]{}
                child{node[no node,label=below:{$(e,f)$}]{}edge from parent node[left]{$C$}} child{node[no node,label=below:{$(g,h)$}]{}edge from parent node[right]{$D$}} 
            edge from parent node[right,xshift=3]{$B$}
            }
        edge from parent node[left,xshift=-3]{$1/2$} 
    } 
        child{node(2)[solid node]{} 
            child{node(5)[solid node]{}
                child{node[no node,label=below:{$(h,i)$}]{}edge from parent node[left]{$C$}} child{node[no node,label=below:{$(j,k)$}]{}edge from parent node[right]{$D$}} 
            edge from parent node[left, xshift=-3]{$A$}
            }
            child{node(6)[solid node]{}
                child{node[no node,label=below:{$(l,m)$}]{}edge from parent node[left]{$C$}} child{node[no node,label=below:{$(n,o)$}]{}edge from parent node[right]{$D$}} 
            edge from parent node[right,xshift=3]{$B$}
            }
        edge from parent node[right,xshift=3]{$1/2$} 
    };
    \draw[dashed,rounded corners=10]($(1)+(-.2,.25)$)rectangle($(2)+(.2,-.25)$);
    \draw[dashed,rounded corners=10]($(3)+(-.2,.25)$)rectangle($(6)+(.2,-.25)$);
        \node at ($(1)!.5!(2)$) {$P2$};
    \node at ($(3)!.5!(6)$) {$P1$};
    \end{tikzpicture}
    \end{minipage}
    \caption{Examples of extensive-form games. In the game on the left, Player 1 has path recall, but not perfect recall. In the game on the right, Player 1 has perfect recall, but not path recall.}
\label{fig:efg-1}
\end{figure}

\begin{lemma}
\label{lem:adversary-children}
    Define an extensive-form game $G$ for which player $i$ has perfect recall and path recall. For any node $h$ such that $\rho(h) \neq i$, define $C_h$ as the set of all children of $h$, i.e. $$C_h = \{h' \in \cH \mid h' = c(h, a) \text{ for some $a \in \cA(\Pi(h))$}\}.$$
    For any of player $i$'s information sets $I' \in \cI_i$, either all nodes of $I'$ are contained in $C_h$, or none of them are.  
\end{lemma}
\begin{proof}
    Fix a node $h$ played by someone other than player $i$ and an information set $I' \in \cI_i$. Assume by way of contradiction that there exist nodes $h_1$ and $h_2$ in $I'$ such that $h_1 \in C_h$ and $h_2 \notin C_h$. By definition of path recall, $h$ must be an ancestor of $h_2$. Since $h_2 \notin C_h$, there must be some node $h_2'$ on the path between $h$ and $h_2$. If $\rho(h_2') \neq i$, this contradicts the assumption of path recall (since $h_2'$ is an ancestor of $h_2$ but not of $h_1$). If $\rho(h_2') = i$, then in order for path recall to be satisfied, some ancestor $h_1'$ of $h_1$ must be in the same information set as $h_2'$. Since any ancestor of $h_1$ is also an ancestor of $h_2'$, $h_1'$ and $h_2'$ being in the same information set violates the assumption of perfect recall. Thus, for any node $h$ such that $\rho(h) \neq i$, if any node in $I'$ is in $C_h$, then all nodes of $I'$ must be in $C_h$. 
\end{proof}

\begin{lemma}
\label{lem:learner-information-sets}
    Define an extensive-form game $G$ for which player $i$ has perfect recall and path recall. For any information set $I$ such that $\rho(I) = i$ and action $a \in \cA(I)$, define $C_{I, a}$ as the set of all nodes resulting from playing action $a$ at nodes in $I$, i.e. $$ C_{I, a} = \{h' \in \cH \mid h' = c(h, a) \text{ for some $h \in I$} \}.$$
    For any information set $I' \in \cI_i$, either all nodes of $I'$ are contained in $C_{I,a}$, or none of them are.
\end{lemma}
\begin{proof}
    The proof parallels that of \ref{lem:adversary-children}. Fix any two information sets $I, I' \in \cI_i$, and any action $a \in \cA(I)$. Assume by way of contradiction that there exist nodes $h_1$ and $h_2$ in $I'$ such that $h_1 \in C_{I, a}$ and $h_2 \notin C_{I,a}$. By definition of perfect recall, $h_2$ must be a descendant of some node in $I$ along the path taken after playing action $a$. That is, $h_2$ is the descendant of some node $h_2' \in C_{I,a}$. if $\rho(h_2') \neq i$, then path recall is not satisfied, since $h_2'$ is an ancestor of $h_2$ but not of $h_1$. If $\rho(h_2') = i$, then in order for path recall to be satisfied, some ancestor $h_1'$ of $h_1$ must be in the same information set as $h_2'$. However, this means that $h_1'$, an ancestor of a node in $I$ (call it $h_1''$), and $h_2'$, a descendant of a node in $I$, are in the same information set. In order for perfect recall to be satisfied, there must be some ancestor of $h_1''$ (a node in $I$) that is also in $I$, which itself violates the property of perfect recall. Thus, if any node in $I'$ is in $C_{I,a}$, then all nodes of $I'$ must be in $C_{I,a}$. 
\end{proof}

For any game that satisfies perfect recall and path recall for the learner, the children of any of the learner's information sets (or any opponent's nodes) is a union of learner's information sets and opponents' nodes. Thus, we can reason inductively about the leaf-form strategies restricted by any subtree of the game in terms of the leaf-form strategies restricted by the children of the root node of the subtree (or the information set rooting the subtree). This gives us a natural recursive algorithm able to return the optimal leaf-form strategy at any subtree rooted either at a learner's information set, or at an opponent's node, given a collection of opponents' strategies (in terms of a payoff-weighted reachability vector). 


\begin{algorithm}[h]
    \begin{algorithmic}
    \STATE // Base case - terminal nodes
    \IF{$I = \{h\} \cap h \in \cZ$}
        \RETURN ($\{h\}, \vec{v}(h)$) \tcp{return payoff-weighted probability of reaching terminal node $h$ and the corresponding leaf-form strategy}
    \ELSE 
        \IF{$I \in \cI_i$} 
        \STATE // Player $i$'s information set
            \FOR{$a \in \cA(I)$} \vskip 3pt
                \STATE Initialize $s_a = \emptyset, v_a = 0$.
                \FOR{$h' \in C_{I, a}$ such that $\rho(h') \neq i$}
                    \STATE Let $(s, v)$ = \texttt{Best-Response-Leaf-Form($\{h'\}, \vec{v}, i, G$)}
                    \STATE $s_a = s_a \cup s, v_a = v_a + v$
                \ENDFOR
                \FOR{$I' \in C_{I, a}$ such that $\rho(I') = i$}
                    \STATE Let $(s, v)$ = \texttt{Best-Response-Leaf-Form($I', \vec{v}, i, G$)}
                    \STATE $s_a = s_a \cup s, v_a = v_a + v$
                \ENDFOR \vskip 3pt
            \ENDFOR
            \STATE Let $a_{\text{max}} = \max_{a \in \cA(I)} v_a$ 
            \RETURN $(s_{a_{\text{max}}}, v_{a_{\text{max}}})$ \vskip 5pt
        \ELSIF{$I = \{h\}$ and $\rho(h) \neq i$}
            \STATE // Opponent node
            \STATE Initialize $s_I = \emptyset, v_I = 0$.
            \FOR{$h' \in C_{h}$ such that $\rho(h') \neq i$}
                \STATE Let $(s, v)$ = \texttt{Best-Response-Leaf-Form($\{h'\}, \vec{v}, i, G$)}
                \STATE $s_I = s_I \cup s, v_I = v_I + v$
            \ENDFOR
            \FOR{$I' \in C_{h}$ such that $\rho(I') = i$}
                \STATE Let $(s, v)$ = \texttt{Best-Response-Leaf-Form($I', \vec{v}, i, G$)}
                \STATE $s_I = s_I \cup s, v_I = v_I + v$
            \ENDFOR \vskip 3pt
            \RETURN $(s_I, v_I)$
        \ENDIF
    \ENDIF
    \end{algorithmic}
    \caption{\texttt{Best-Response-Leaf-Form($I, \vec{v}, i, G)$}}
    \label{alg:backward-ind}
\end{algorithm}
Firstly, note that Algorithm \ref{alg:backward-ind} returns a set of leaf nodes rather than a leaf-node strategy. We use this representation for ease of proof, with the observation that it is trivial to convert a set of leaf nodes to a leaf-node strategy, i.e. some $\vec{s} \in \{0, 1\}^{|\cZ|}$. For any subset of leaf nodes $s \in 2^\cZ$, we use $\vec{s}$ to represent the corresponding vector representation. Secondly, Algorithm $\ref{alg:backward-ind}$ takes a set of nodes $I$ as input. In practice, this will either be a single node played by an opponent (or a terminal node), or an information set played by the learner (player $i$). 

For notational purposes in following proofs, for any deterministic behavioral strategy $\pi \in \Pi_i$, let $s_{\pi}$ be the corresponding leaf-form strategy. For any set of nodes $I \in 2^\cH$, let $\pi|_{I}$ be the restriction of $\pi$ to nodes in the subtree rooted at $I$, and let $s_{\pi}|_{I}$ be the restriction of $s_{\pi}$ to the collection of leaf nodes reachable from $I$, that is, $z \in s_{\pi}|_{I}$ iff $z \in s_{\pi}$ and $z$ is a descendant of some node in $I$.

\begin{theorem}
    Fix an extensive-form game $G$, a player $i$, a set of nodes $I$ representing either one of player $i$'s information sets or a single node of another player, and a payoff-weighted reachability vector $v \in \R^{|\cZ|}$. If player $i$ has both perfect recall and path recall in $G$, then the backwards induction algorithm (Algorithm \ref{alg:backward-ind}) returns an optimal leaf-form strategy $s_I$ that maximizes the player's expected payoff over the set of all restrictions of leaf-node strategies to the subtree rooted at $I$, and the corresponding expected payoff $v_I$. That is, 
    \begin{align*}
        s_I &= \argmax_{s_{\pi}|_{I}, \pi \in \Pi_i} \langle \vec{v}, \vec{s}_{\pi}|_{I} \rangle \\
        v_I &= \langle \vec{v}, \vec{s_I}\rangle.
    \end{align*}
\end{theorem}
\begin{proof}
    We proceed by induction. Our base case is when $I$ consists of a single terminal node (some $h \in \cZ$). In this case, no more actions are to be played and since there must exist some strategy $\pi \in \Pi_i$ such that $h$ is reachable, there is exactly one leaf-form strategy restricted to the current sub-tree: $s = \{h\}$. So Algorithm \ref{alg:backward-ind} returns the only viable (and thus optimal) strategy, and the corresponding expected payoff for player $i$ is $\vec{v}(h)$.

    If $I$ is not just a terminal node, there are two cases to consider; when $I$ is an information set playable by player $i$, and when $I$ consists of a single node playable by some opponent. When $I \notin \cI_i$, $I = \{h\}$ such that player $i$ does not control the action taken at $h$. Define $S_I = \{s_{\pi}|_{I} \mid \pi \in \Pi_i\}$ as the collection of all leaf-node strategies restricted to the current sub-tree (rooted at $h$). Consider $C_h$, the set of children of node $h$. By Lemma \ref{lem:adversary-children}, $C_h$ can be written as a disjoint union $\cup_{k=1}^m I_k$ such that each $I_k$ is either an information set played by player $i$ or a single node played by some opponent. For any strategy $\pi \in \Pi_i$, the restricted strategy $\pi|_I$ is described entirely in terms of strategies restricted by each of the sets $I_k$. That is, for any information set $I' \in \cI_i$ within the current subtree, $\pi|_I$ is defined as follows:
    \begin{equation*}
        \pi|_I(I') = \pi|_C(I')
    \end{equation*}
    where $C$ is the unique $I_k$ which roots the subtree containing $I'$. Further, \begin{equation*}
        \langle \vec{v}, \vec{s}_{\pi}|_{I} \rangle = \sum_{k = 1}^m \langle \vec{v}, \vec{s}_{\pi}|_{I_k} \rangle
    \end{equation*} 
    since each terminal node in the subtree rooted at $I$ will be in the subtree of exactly one of the subtrees rooted at each of the subsets $I_k$. Therefore, we can optimize for $s_I$ by separately optimizing for each of the strategies $s_{I_k}$. By the inductive hypothesis, Algorithm \ref{alg:backward-ind} applied to each of these sets of nodes will return the optimal strategies, and the resulting leaf-node strategy (i.e. the set of reachable leaf nodes via this strategy) will be the union of all the resulting $s_{I_k}$. The corresponding expected payoff will accordingly be the sum of the childrens' payoffs. 

    Finally, we must consider the case when $I \in \cI_i$. For any action $a \in \cA(I)$, the set $C_{I, a}$ (using Lemma \ref{lem:learner-information-sets}) of children of nodes in $I$ reached upon taking action $a$ can be written as a disjoint union $\cup_{k=1}^m I_k$ of information sets played by player $i$ and single nodes played by opponents. Therefore, the maximum expected payoff \textit{conditioned on taking action $a$} at $I$ is the sum of the maximum expected payoffs with respect to each of the $I_k$ making up $C_{I, a}$ - and by the inductive hypothesis, the payoff returned by Algorithm \ref{alg:backward-ind} for eack $I_k$ is the payoff corresponding to the optimal leaf-form strategy restricted to $I_k$. Having computed the maximum expected payoff received for taking each possible action at $I$, selecting the optimal distribution over actions at $I$ is straightforward - we deterministically choose the single action which maximizes this value ex post. 

    As Algorithm \ref{alg:backward-ind} implements the logic described above for each separate case, the proof is complete.     
\end{proof}
Running Algorithm \ref{alg:backward-ind} on the root node of any game $G$ satisfying the necessary properties will return a best-response leaf-form strategy for the learner along with the corresponding expected payoff.

\subsection{Improving Algorithm Efficiency}
\label{subsec:efg-efficiency}
The algorithm for obtaining subsequence regret has computational efficiency depending largely on two factors - the length of the vector predicted at each round (upon which the algorithm making unbiased predictions has a linear dependence) and the efficiency of the best-response oracle. The best-response oracle described in the previous section (Algorithm \ref{alg:backward-ind}) uses the payoff-weighted reachability vector $v_t$ with length $|\cZ|$, the number of terminal nodes in the game being played. This representation was chosen in order to be able to describe the expected payoff of a learner's strategy as a linear function of the strategy. For extensive-form games where there is a more compact representation that satisfies this linearity property and is still sufficiently informative in order for the algorithm to compute a best-response, the general principle described in Section \ref{sec: subsequence} still works - now with the algorithms introduced in Section \ref{sec:general_algorithm} running in time linear with respect to this smaller vector dimension. Below, we describe a class of extensive-form games for which the length of the vector prediction scales with the number of the learner's information sets, as opposed to the number of terminal nodes, and for which the best-response oracle described in the previous section (with very minor alterations) still outputs best-responses. 

At a high level, Algorithm \ref{alg:backward-ind} reasons about the maximum possible payoff at any node of a tree in terms of the payoffs of its children, and, starting from the terminal nodes, backtracks through the tree to choose, at each of the learner's information sets, the action which maximizes expected payoff. Thus, any linear representation of the opponents' (predicted) strategy which allows us to compute the expected payoff for each action at each of the learner's information sets will suffice.

\begin{definition}
An information set $I \in \cI_i$ of player $i$ is a \emph{terminal information set} if for each node $h \in I$, none of the descendants of $h$ are playable by player $i$. 
\end{definition}

\begin{definition}
The \emph{predecessor function} $P_i: \cZ \to \cI_i \cup \emptyset$ for player $i$ maps each terminal node $z$ to the last information set of player $i$ that is played along the path from the root to $z$. If no such information set exists, then $P_i(z) = \emptyset$. 
\end{definition}

For any game $G$ where $P_i(z)$ is a terminal information set for each $z \in \cZ$\footnote{The games shown in Figure \ref{fig:efg-1} are examples of games satisfying this property. Figure \ref{fig:efg-2} is an example of game that does not satisfy this property.} (assuming the learner is player $i$), rather than predict the probability of reaching each terminal node, we can predict $v$ as the expected payoff for taking each viable action at each of the learner's terminal information sets. Correspondingly, instead of picking a leaf-form strategy, the learner can pick a ``terminal information set'' strategy $\pi$ indexed by pairs of terminal information sets $I$ and actions $a \in \cA(I)$, such that $\pi_{I, a} = 1$ iff action $a$ is taken at information set $I$. Then, the net payoff using strategy $\pi$ is $\langle \pi, v \rangle$, where $v_{I,a}$ is the expected payoff from taking action $a$ at information set $I$, and we once again have a linear optimization problem, now in dimension scaling linearly with the number of information sets and available actions at each of these information sets. Since the payoff for each non-terminal information set is deducible from the payoffs at each terminal information set (since the path to any leaf node must cross some terminal information set), the same Algorithm \ref{alg:backward-ind}, modified only to consider terminal information sets as base cases instead of terminal nodes, will be able to compute a best-response. 
\begin{figure}[t!]
    \begin{center}
    \begin{tikzpicture}[scale=1.5,font=\footnotesize] 
    \tikzstyle{level 1}=[level distance=15mm,sibling distance=35mm] 
    \tikzstyle{level 2}=[level distance=12mm,sibling distance=20mm]
    \tikzstyle{level 3}=[level distance=12mm,sibling distance=10mm]
     \tikzstyle{level 4}=[level distance=10mm,sibling distance=5mm]
    \node(0)[solid node,label=above:{$P1$}]{} 
        child{node(1)[solid node]{}
            child{node(3)[solid node, label=above left:{$P1$}]{}
                child{node(5)[solid node]{}
                    child{node[no node,label=below:{$z_1$}]{}edge from parent node[left]{$L$}} child{node[no node,label=below:{$z_2$}]{}edge from parent node[right]{$R$}}
                edge from parent node[left, xshift=-3]{$A$}
                }
                child{node(6)[solid node]{}
                    child{node[no node,label=below:{$z_3$}]{}edge from parent node[left]{$L$}} child{node[no node,label=below:{$z_4$}]{}edge from parent node[right]{$R$}} 
                    edge from parent node[right, xshift=3]{$B$}
                }
            edge from parent node[left, xshift=-3]{$C$}
            }
            child{node(4)[solid node, label=above right:{$P1$}]{}
                child{node(7)[solid node]{}
                    child{node[no node,label=below:{$z_5$}]{}edge from parent node[left]{$L$}} child{node[no node,label=below:{$z_6$}]{}edge from parent node[right]{$R$}}
                edge from parent node[left, xshift=-3]{$A$}
                }
                child{node(8)[solid node]{}
                    child{node[no node,label=below:{$z_7$}]{}edge from parent node[left]{$L$}} child{node[no node,label=below:{$z_8$}]{}edge from parent node[right]{$R$}} 
                    edge from parent node[right, xshift=3]{$B$}
                }
            edge from parent node[right, xshift=3]{$D$}
            }
        edge from parent node[left,xshift=-3]{$A$} 
    } 
    child{node(2)[solid node]{} 
        child{node[no node,label=below:{$z_9$}]{}edge from parent node[left]{$C$}} child{node[no node,label=below:{$z_{10}$}]{}edge from parent node[right]{$D$}} 
    edge from parent node[right, xshift=3]{$B$}
    };
    \draw[dashed,rounded corners=10]($(1)+(-.2,.25)$)rectangle($(2)+(.2,-.25)$);
    \draw[dashed,rounded corners=10]($(5)+(-.2,.25)$)rectangle($(6)+(.2,-.25)$);
    \draw[dashed,rounded corners=10]($(7)+(-.2,.25)$)rectangle($(8)+(.2,-.25)$);
    \node at ($(1)!.5!(2)$) {$P2 ~~ (I_1)$}; 
    \node at ($(5)!.5!(6)$) {$P2 ~ (I_2)$}; 
    \node at ($(7)!.5!(8)$) {$P2 ~ (I_3)$}; 
    \end{tikzpicture}
    \end{center}
    \caption{A two-player extensive-form game. Of Player 2's three information sets, only $I_2$ and $I_3$ are terminal. Note that $P_2(z_9) = P_2(z_{10}) = I_1$, which is not a terminal information set.}
    \label{fig:efg-2}
\end{figure}
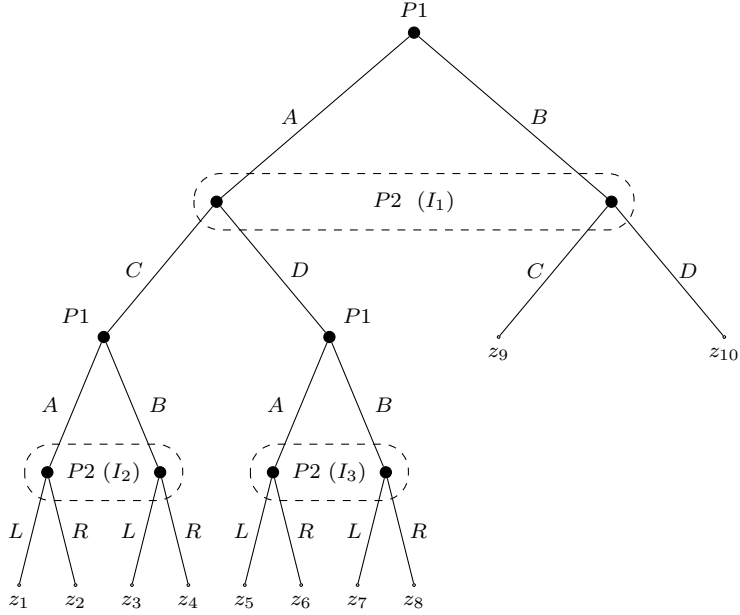


\subsection{Connecting Subsequence Regret to Other Forms of Regret in Extensive-Form Games}
To define subsequence regret in a similar framework as causal regret, or more generally, $\Phi$-regret, we must define deviations in terms of behavioral strategies, rather than in terms of leaf-form strategies. In the rest of this section, we work only with behavioral strategies, and we assume efficient conversion of leaf-form strategies to behavioral strategies via an oracle.
\begin{definition}
    Fix a collection of binary events $\cE$ and behavioral strategy space $S_i$. Fix some $E\in \cE$. A subsequence deviation is a function $\phi_E: S_i \to S_i$ such that
    \[ \phi_E(s_i) = \begin{cases}
        s'_i, & E = 1\\
        s_i, & \text{otherwise}\\
        \end{cases} \]
    Let $\Phi_{sub}$ be the set of all possible subsequence deviations.
\end{definition}

Translating the events defined by Theorem \ref{thm: efg-subsequence} to a set of subsequence deviations, we obtain:
\begin{corollary}
\label{cor: subseq-deviation}
    Fix an information set $I'$, subsequent action $a' \in \cA(I)$, strategy $s'_i$, events $\cE=\{E_{I,a}:I \in \cI_i, a \in \cA(I)\}$, and transcript $\pi_T$. Let a \emph{subsequence deviation $\phi_{I',a'}$} of strategy $s_i \in S_i$ return a strategy such that at each information set $I \in \cI$,
    \[
    \phi_{I',a'}(s_i)(I) = \begin{cases}
        s'_i(I), & E_{I',a'} = 1\\
        s_i(I), & \text{otherwise}\\
        \end{cases}
  \]
    Let $\Phi_{\textrm{sub}}$ be the collection of all $\phi_{I',a'}$. If the learner has $(\cE,\alpha)$-subsequence regret, then the learner has $\Phi$-regret with respect to the class of subsequence deviations bounded by:
    \[ \max_{\phi_{I',a'} \in \Phi_{\textrm{sub}}} r(\pi_T, u, \phi_{I',a'}) \le \alpha \]
\end{corollary}
The proof of this follows straight from definition -- for each event $E_{I',a'}$, the learner deviates to some fixed strategy $s'_i$ in the transcript whenever $E_{I',a'}=1$. To show that no subsequence regret over events $\cE$ implies no causal regret, we will define a related, but slightly different notion of a causal deviation that we will use to directly compare to a subsequence deviation known as a \emph{reachable informed causal deviation}.

\begin{definition}
  Fix an information set $I'$, subsequent action $a' \in \cA(I)$, and strategy $s'_i$. A \emph{reachable informed causal deviation} $\phi$ of strategy $s_i \in S_i$ returns a strategy such that at each information set $I \in \cI$,
    \[
    \phi(s_i) = \begin{cases}
        s'_i(I), & I \succeq I', s_i(I')=a', I' \text{ reachable under } s_i \\
        s_i(I), & \text{otherwise}\\
        \end{cases}
  \]
    Let $\Phi_\textrm{r-causal} = \{\phi_{I,a,s} : I \in \cI_i, a\in \cA(I), s \in S_i \}$ be the set of all reachable informed causal deviations.
\end{definition}

\begin{lemma}
\label{lem: reachable-causal}
    Fix an extensive-form game $G$ and any transcript $\pi_T$. Let the action space $\cA = S_i$ be the set of all deterministic behavioral strategies of player $i$. Then, for all $\phi_{I,a,s} \in \Phi_\textrm{r-causal}$ and $\phi'_{I,a,s} \in \Phi_\textrm{causal}$,  
    \[ r(\pi_T, u, \phi_{I,a,s}) = r(\pi_T, u, \phi'_{I,a,s})\]
\end{lemma}
\begin{proof}
    Suppose the transcript of behavioral strategies in $\pi_T$ is $(s_1,s_2,\ldots s_T)$. Observe that for all strategies $s_k$ where $k \in [T]$ such that the trigger information set $I$ is reachable under $s_k$, $\phi_{I,a,s}(s_k) = \phi'_{I,a,s}(s_k)$, by definition. 
    
    Thus, the modified strategies differ only between strategies $s_j$ such that $I$ is not reachable under $s_j$. When $I$ is not reachable under $s_j$, $\phi_{I,a,s}(s_j)(I') = s_j(I')$ and $\phi'_{I,a,s}(s_j)(I') = s_j(I')$ for all $I' \in \cI_i$ except for information sets $I' \succeq I$. However, since $I$ is not reachable under $s_j$, play will never reach $I$, and thus the regret must also be equal between the two modified strategies.
\end{proof}

\begin{lemma}
\label{lem: causal-subset}
    The set of reachable informed causal deviations $\Phi_\textrm{r-causal}$ is a subset of the subsequence deviations $\Phi_{sub}$ defined by events $\cE= \{E_{I,a}: I \in \cI_i, a\in \cA(I)\}$ in Theorem \ref{thm: efg-subsequence}.
\end{lemma}
\begin{proof}
    Fix any $\phi_{I',a',s'} \in \Phi_\textrm{r-causal}$ with trigger sequence $(I',a')$ and trigger deviation $s'$. By definition, 
    \[
    \phi_{I',a',s'}(s) = \begin{cases}
        s'(I), & I \succeq I', s(I')=a', I' \text{ is reachable under } s \\
        s(I), & \text{otherwise}\\
        \end{cases}
    \]
   Clearly, when $I'$ is reachable under $s$ and $s(I')=a$, and thus $E_{I,a}=1$. Thus, $\phi_{I',a',s'}(s)$ can be expressed as a subsequence deviation:
   \[
    \phi_{I',a',s'}(s) = \phi^{sub}_{I',a',s^!}(s) = \begin{cases}
        s^!(I), & E_{I,a} = 1 \\
        s(I), & \text{otherwise}\\
        \end{cases}
    \]
    where $s^!(I)$ is the behavioral strategy equal to $s'(I)$ at all information sets $I \succeq I'$, and equal to $s(I)$ otherwise.
\end{proof}

\causal*
\begin{proof}
    We can apply Corollary \ref{cor: subseq-deviation}, showing that if we have $\cE$-subsequence regret at most $\alpha$, we must have $\Phi$-regret with respect to $\Phi_\textrm{sub}$ bounded by:
    \[ \max_{\phi \in \Phi_{\textrm{sub}}} r(\pi_T, u, \phi) \le \alpha \]
    
    Next, by Lemma \ref{lem: reachable-causal}, we can safely consider only the set of reachable causal deviations $\Phi_{\textrm{r-causal}}$ as relevant to causal regret. By Lemma \ref{lem: causal-subset}, we have that $\Phi_{\textrm{r-causal}}\subseteq \Phi_\textrm{sub}$, and thus, we must have 
    \[ \max_{\phi \in \Phi_{\textrm{causal}}} r(\pi_T, u, \phi) = \max_{\phi \in \Phi_{\textrm{r-causal}}} r(\pi_T, u, \phi) \le \max_{\phi^{sub} \in \Phi_{\textrm{sub}}} r(\pi_T, u, \phi) \le \alpha\]

\end{proof}

\begin{remark}
The set of linear-swap deviations \cite{farina2023polynomial} is not a superset of the deviations that can be captured by subsequence regret over a fixed collection of events $\cE$. In their paper, the authors give an example of a nonlinear swap (see their Example E.3), where strategy $A_1B_2$ is swapped to $A_1B_1$, and strategy $A_2B_1$ is swapped to $A_1B_1$. A simplified game tree is given below this remark.
This can be represented with subsequence deviations with event set $\cE=\{E_{1,2},E_{2,1}\}$, where $E_{i,j} = 1$ if $A_i$ is played at information set $I_1$ and $B_j$ is played at information set $I_2$, deviating to the fixed strategy $A_1B_1$ whenever an event occurs.
\end{remark}
\begin{figure}[h!]
\includegraphics[width=12cm]{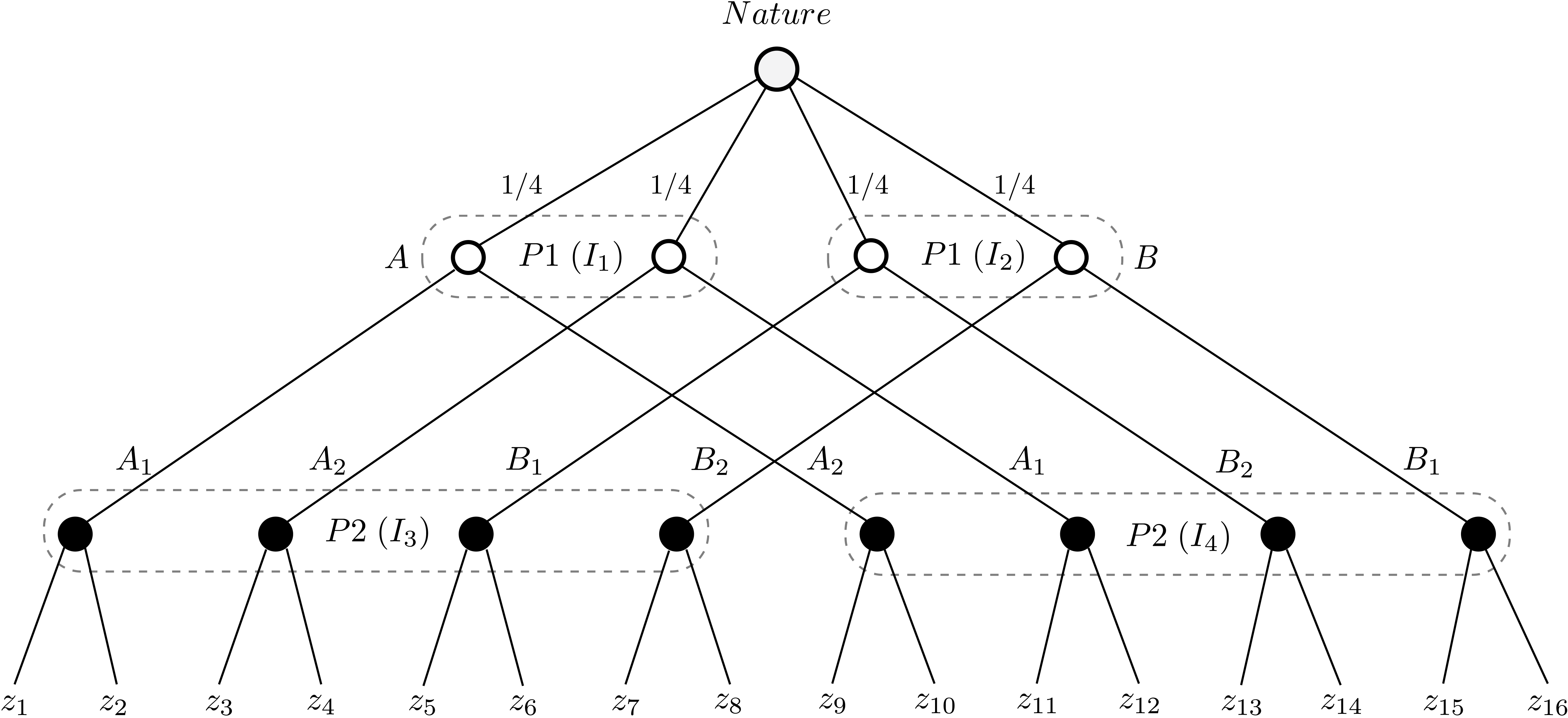}
\centering
\caption{A simplified game tree given in \cite{farina2023polynomial}, Example E.3. }
\end{figure}

\begin{remark}
The deviations of the behavioral deviation landscape \cite{morrill2021efficient} are not a superset of the deviations that can be captured by subsequence regret over a fixed collection of events $\cE$. For example, instantiating an event set $\cE = \{E_{I_1,I_2,I_3,a_3} : I_1,I_2,I_3 \in \cI_i, a_3 \in \cA(I_3)\}$ such that $E_{I_1,I_2,I_3,a_3}=1$ if a strategy makes $I_1,I_2$ reachable under the learner's strategy but also plays the action $a_3$ at information set $I_3$ does not correspond to any deviation type enumerated within the behavioral deviation landscape.
\end{remark}

\section{Proofs From Section \ref{sec: subsequence}}
\label{app: subsequence}
\subsequence*
\begin{proof}[Proof of Theorem \ref{thm: subsequence}]
    Fix any  $\phi \in \Phi_{\textrm{Ext}}$ and $E \in \cE$. We need to upper bound $r(\pi_T, u, \phi,E)$. Using the linearity of $u(S_t, p_t) = \sum_{i \in S_t} p_{t,i}$ in its second argument for all $S_t \in \cA$, and observing that $\phi(\cdot)$ is a constant strategy modification function (i.e. $\phi(D) = D'$ for all $D \in \cD$), we can write:
    \begin{eqnarray*}
        r(\pi_T, u, \phi,E) &=& \frac{1}{T} \sum_{t = 1}^T E(x_t,p_t,\pi_{t-1}) \cdot  \left(u(\phi(\delta_u(p_t)),y_t) - u(\delta_u(p_t),y_t)\right) \\
        &=& \frac{1}{T} \sum_{t = 1}^T E(x_t,p_t,\pi_{t-1}) \cdot  \left(u(D',y_t) - u(\delta_u(p_t),y_t)\right) \\
        &=& \frac{1}{T} \sum_{t = 1} ^ T E(x_t,p_t,\pi_{t-1}) \cdot \sum_{b \in D'} y_{t,b} - \frac{1}{T}\sum_{t=1}^T E(x_t, p_t, \pi_{t-1}) \sum_{b \in \delta_u(p_t)} y_{t,b}  \\
        &=& \frac{1}{T}  \left(\sum_{b \in D'}\sum_{t=1}^T E(x_t,p_t,\pi_{t-1}) y_{t,b} -\sum_{b \in B}\sum_{t=1}^T I_{b,E}(p_t) \cdot y_{t,b}\right) \\
        &\le& \frac{1}{T} \left(\sum_{b \in B}\sum_{t=1}^T E(x_t,p_t,\pi_{t-1}) \cdot p_{t,b}- I_{b, E}(p_t) \cdot p_{t,b}\right) \\
        &&  + \sum_{b \in B}\frac{\alpha(n_T(E, \pi_T)) + \alpha(n_T(I_{b,E}, \pi_T))}{T}\\
        &\le& \sum_{b \in B}\frac{\alpha(n_T(E, \pi_T)) + \alpha(n_T(I_{b,E}, \pi_T))}{T}\\
    \end{eqnarray*}
    Here the first inequality follows the $\alpha$-unbiasedness condition over $\cI$: for every $b \in B$, we have 
    \[
    \left|\sum_{t = 1}^T (p_{t,b}-y_{t,b}) I_{b, E}(p_t)\right| \leq \alpha(n_T(I_{b, E}, \pi_T)).
    \]
    Similarly, $\alpha$-unbiasedness over $\cE$ gives us, for every $b \in B$:
    $$\left|\sum_{t = 1}^T (p_{t,b} - y_{t,b}) E(x_t, p_t, \pi_{t-1})\right| \leq \alpha(n_T(E, \pi_T)).$$    
    The 2nd inequality follows from the fact that by definition of the straightforward decision maker, $u(\delta_u(p_t),p_t) \geq u(D',p_t)$ for any $D' \in \cD$, and the second term in the previous expression can be simplified:
    $$\sum_{b \in B} \sum_{t=1}^T I_{b, E} (p_t) \cdot p_{t,b} = \sum_{t=1}^T \sum_{b \in \delta_u(p_t)} E(x_t, p_t, \pi_{t-1}) \cdot p_{t,b} = \sum_{t=1}^T E(x_t, p_t, \pi_{t-1}) \cdot u(\delta_u(p_t), p_t).$$
\end{proof}

\subsequencecorr*
\begin{proof}
    From Theorem \ref{thm:main-guarantee}, at round $t$ the predictions of the canonical algorithm with event set $\cE'$ have expected bias bounded by

    \begin{align*}
        \E_{\pi^t}\left[\alpha(T, n_t(E', \pi_t))\right] &= O\left(\ln(n|\cE'|T) + \sqrt{\ln(n|\cE'|T) \cdot n_t(E',\pi_t)}\right) \\
        &= O\left(\ln(n|\cE|T) + \sqrt{\ln(n|\cE|T) \cdot n_t(E',\pi_t)}\right) \\
    \end{align*}
    for each event $E' \in \cE$. Since $n_t(E', \pi_t) \le t$, we can plug this into  Theorem \ref{thm: subsequence} to get that the canonical algorithm has $\cE$-subsequence regret bounded by 
    \[ \max_{E \in \cE, \phi \in \Phi_{\textrm{Ext}}}r(\pi_T,u,\phi,E) \leq O \left(\frac{n\sqrt{\ln(n|\cE|T)}}{\sqrt{t}} \right). \]
\end{proof}

\end{document}